\documentclass{article}

\usepackage{microtype}
\usepackage{graphicx}
\usepackage{subcaption}
\usepackage{booktabs} % for professional tables

\usepackage{hyperref}

\usepackage[preprint]{icml2026}

\usepackage{amsmath}
\usepackage{amssymb}
\usepackage{mathtools}
\usepackage{amsthm}
\usepackage{hyperref}
\usepackage{url}
\usepackage{graphicx}
\usepackage{booktabs} 
\usepackage{adjustbox}
\usepackage{array}
\usepackage{pifont}
\usepackage{subcaption}
\usepackage{wrapfig}
\usepackage{amsmath}
\usepackage{amsthm, amsthm, amssymb}
\usepackage[table]{xcolor}
\usepackage{multirow}
\usepackage{hyperref}

\makeatletter
\AtBeginDocument{
\renewcommand{\sectionautorefname}{\S\@gobble}

\renewcommand{\subsectionautorefname}{\S\@gobble} 
\renewcommand{\subsubsectionautorefname}{\S\@gobble}
\renewcommand{\appendixautorefname}{Appendix \S\@gobble}

}
\makeatother

\usepackage[capitalize,noabbrev]{cleveref}

\theoremstyle{plain}
\newtheorem{theorem}{Theorem}[section]

\theoremstyle{definition}

\theoremstyle{remark}

\usepackage[textsize=tiny]{todonotes}

\icmltitlerunning{Semantic Bridging Domains: Pseudo-Source as Test-Time Connector}

\begin{document}

\twocolumn[
  \icmltitle{Semantic Bridging Domains: Pseudo-Source as Test-Time Connector}



  \icmlsetsymbol{equal}{*}
  \icmlsetsymbol{boss}{\dag}

  \begin{icmlauthorlist}
    \icmlauthor{Xizhong Yang}{1}
    \icmlauthor{Huiming Wang}{2}
    \icmlauthor{Ning Xu}{1}
    \icmlauthor{Mofei Song}{1,boss}
  \end{icmlauthorlist}

  \icmlaffiliation{1}{Southeast University}
  \icmlaffiliation{2}{Kuaishou Technology}

  \icmlcorrespondingauthor{Mofei Song}{songmf@seu.edu.cn}

  \icmlkeywords{Machine Learning, ICML}

  \vskip 0.3in
]



\printAffiliationsAndNotice{}  % no special notice (required even if empty)

\begin{abstract}

    Distribution shifts between training and testing data are a critical bottleneck limiting the practical utility of models, especially in real-world test-time scenarios. To adapt models when the source domain is unknown and the target domain is unlabeled, previous works constructed pseudo-source domains via data generation and translation, then aligned the target domain with them. However, significant discrepancies exist between the pseudo-source and the original source domain, leading to potential divergence when correcting the target directly. From this perspective, we propose a \textit{Stepwise Semantic Alignment} (SSA) method, viewing the pseudo-source as a semantic bridge connecting the source and target, rather than a direct substitute for the source. Specifically, we leverage easily accessible universal semantics to rectify the semantic features of the pseudo-source, and then align the target domain using the corrected pseudo-source semantics. Additionally, we introduce a \textit{Hierarchical Feature Aggregation} (HFA) module and a \textit{Confidence-Aware Complementary Learning} (CACL) strategy to enhance the semantic quality of the SSA process in the absence of source and ground truth of target domains. We evaluated our approach on tasks like semantic segmentation and image classification, achieving a $5.2\%$ performance boost on GTA$\to$Cityscapes over the state-of-the-art. Code available at \href{https://github.com/yxizhong/SHLSA}{https://github.com/yxizhong/SSA}.
\end{abstract}

\section{Introduction}
    Most machine learning methods assume that training and testing data are independently and identically distributed (i.i.d.)~\cite{liang2025comprehensive}. However, this assumption often fails in real-world scenarios due to limited data coverage and distribution shifts between source and target domains. To address such shifts, \textit{Domain Generalization} (DG) learns domain-invariant features from labeled sources~\cite{zhou2022domain}, while \textit{Domain Adaptation} (DA) transfers knowledge from labeled source data to unlabeled target data~\cite{farahani2021brief}. A more challenging setting, \textit{Test-Time Adaptation} (TTA), adapts a source-trained model using only unlabeled target data, without access to source data.
    
    To tackle this source-free, label-free scenario, TTA methods employ strategies such as pseudo labeling and Source Distribution Estimation (SDE)~\cite{liang2025comprehensive}. Among them, SDE stands out for its ability to preserve source discriminability and ensure stable adaptation under distribution shifts. It reconstructs an estimate of the source distribution through data generation, domain translation, or memory selection, and aligns it with the target domain via semi-supervised or adversarial learning. This process mitigates confirmation bias and promotes more reliable adaptation by explicitly retaining the structural priors of the source model.

    Although existing SDE methods align low-level visual features (e.g., textures) and statistical information (e.g., BatchNorm mean and variance), or high-level semantic features, often focusing on constructing a high-quality pseudo-source domain~\cite{SHOT, ding2022source, yang2021generalized, wang2022continual}. This pseudo-source domain acts as a proxy for the original source domain to align with the target domain using methods like semi-supervised learning~\cite{ding2023proxymix}. However, significant differences exist between the pseudo-source and source domains, leading to inferior performance compared to using source data with the same alignment methods. Thus, the pseudo-source domain serves as a bridge between the inaccessible source domain and the target domain, facilitating stepwise alignment in the semantic space.

    To address the potential bias between the pseudo-source domain and the original source domain, we propose a \textit{Stepwise Semantic Adaptation} (SSA) method. Specifically, for a target domain with distribution shifts, we measure the deviation of each sample from the source domain based on the output probability distribution of the source model. By data selection, we partition the target domain into a pseudo-source domain and a remaining target domain. Then, unlike previous works that directly align the pseudo-source domain with the remaining target domain, we first perform similarity correction on the pseudo-source domain using coarse-grained abstract semantic features obtained from the pre-trained source model. Subsequently, we align the remaining target domain using the corrected semantic features. This two-step semantic alignment approach effectively enhances the quality of the pseudo-source features, which are traditionally used as direct alignment targets. In contrast to curriculum learning's data scheduling, SSA achieves an easy-to-difficult progression at the alignment process-level. Furthermore, due to the lack of labels in both the source and target domains under test-time conditions, which results in sparse supervision during alignment, we employ \textit{Hierarchical Feature Aggregation} (HFA) module and a \textit{Confidence-Aware Complementary Learning} (CACL)  to enhance the granularity of semantic supervision and the utilization of pseudo-labels, respectively.

    %In our experiments, we selected semantic segmentation and single/multi-label classification as downstream tasks, comparing different backbones such as CNN-Based and ViT-Based models. The results demonstrate that SSA consistently improves performance across various scenarios and tasks. Moreover, this performance enhancement exhibits a scaling effect as the semantic density of the tasks increases. Specifically, on the GTA$\to$Cityscapes and SYNTHIA$\to$Cityscapes benchmarks, SSA achieves improvements of $5.2\%$ and $5.0\%$ over the current state-of-the-art, respectively, showing performance comparable to methods utilizing source domain information.
    In our experiments, we selected semantic segmentation and single/multi-label classification as downstream tasks, comparing CNN-Based and ViT-Based models. The results show that SSA consistently improves performance across various scenarios and tasks, with a scaling effect as the semantic density increases. Specifically, on the GTA$\to$Cityscapes and SYNTHIA$\to$Cityscapes benchmarks, SSA achieves improvements of $5.2\%$ and $5.0\%$ over the current state-of-the-art, respectively, showing performance comparable to methods using source domain information.

\section{Related Work}
\label{relatedwork}
    %\paragraph{Source Distribution Estimation for TTA.} In the absence of source data, traditional domain adaptation (DA) and generalization (DG) methods become inapplicable. Recent TTA techniques address this by relying on the pretrained source model during inference. A key approach, SDE, treats TTA as a DA problem by constructing a pseudo-source domain from target data~\cite{liang2025comprehensive}. Techniques like adversarial training~\cite{nayak2021mining}, style transfer~\cite{hu2022prosfda}, and uncertainty-based sampling~\cite{liang2021source} aim to approximate the source distribution. After pseudo-source domain constructing, virtual domain alignment methods utilizing semi-supervised learning, adversarial training, or contrastive objectives, further align the pseudo-source and target domains~\cite{liang2021source, ding2023proxymix, kurmi2021domain}. These methods promote semantic consistency across domains. However, directly aligning the target domain using the pseudo-source domain is challenging, especially when there is significant divergence between the pseudo-source domain and the original source domain. To address this issue, we leverage the pre-trained general semantics from the source model to achieve stepwise semantic alignment from easy to difficult regions.

    \paragraph{Source Distribution Estimation for TTA.} In the absence of source data, traditional domain adaptation (DA) and generalization (DG) methods are inapplicable. Recent TTA techniques address this by using the pretrained source model during inference. SDE, a key approach, treats TTA as a DA problem by constructing a pseudo-source domain from target data~\cite{liang2025comprehensive}. Techniques like adversarial training~\cite{nayak2021mining}, style transfer~\cite{hu2022prosfda}, and uncertainty-based sampling~\cite{liang2021source} aim to approximate the source distribution. After constructing the pseudo-source domain, virtual domain alignment methods using semi-supervised learning, adversarial training, or contrastive objectives further align the pseudo-source and target domains~\cite{liang2021source, ding2023proxymix, kurmi2021domain}, promoting semantic consistency. However, directly aligning the target domain using the pseudo-source domain is challenging, especially with significant divergence from the original source domain. To address this, we leverage the pre-trained general semantics from the source model for stepwise semantic alignment from easy to difficult regions.
    \vspace{-5pt}
    % \paragraph{Pseudo Labeling for TTA.} 
    % A mainstream direction in TTA focuses on improving pseudo-label quality through denoising or weighting, with the goal of providing reliable supervision under source–target domain shifts \cite{SHOT, ATP, qu2022bmd}. In SDE, pseudo-labels are primarily used as auxiliary tools to approximate the source distribution \citet{ilse2020diva}. Common strategies include confidence-based filtering \cite{ATP}, iterative self-training \cite{huselective}, generative refinement \cite{ding2023proxymix}, and clustering-based assignment \cite{SHOT, qu2022bmd}. Most methods treat pseudo-labels as one-hot targets, which limits their semantic richness. Confidence-aware approaches from semi-supervised learning help mitigate this by better handling uncertain predictions and improving label quality~\cite{feng2024bacon, wang2022semi}. Meanwhile, recent work has demonstrated the effectiveness of hierarchical feature aggregation in capturing both local details and global context~\cite{HRDA}. In SSA, this design facilitates structured semantic fusion, leading to more complete object representations and more reliable pseudo-label generation.
    \paragraph{Pseudo Labeling for TTA.} A mainstream direction in TTA focuses on improving pseudo-label quality through denoising or weighting to provide reliable supervision under source–target domain shifts \cite{SHOT, ATP, qu2022bmd}. In SDE, pseudo-labels are used as auxiliary tools to approximate the source distribution \cite{ilse2020diva}. Strategies include confidence-based filtering \cite{ATP}, iterative self-training \cite{huselective}, generative refinement \cite{ding2023proxymix}, and clustering-based assignment \cite{SHOT, qu2022bmd}. Most methods treat pseudo-labels as one-hot targets, limiting semantic richness. Confidence-aware approaches from semi-supervised learning help by handling uncertain predictions and improving label quality \cite{feng2024bacon, wang2022semi}. Recent work shows hierarchical feature aggregation captures local details and global context \cite{HRDA}. In SSA, this facilitates structured semantic fusion, leading to complete object representations and reliable pseudo-label generation.

\section{Methodology}
\label{method}

\begin{figure*}
    \centering
    \includegraphics[width=0.9\textwidth]{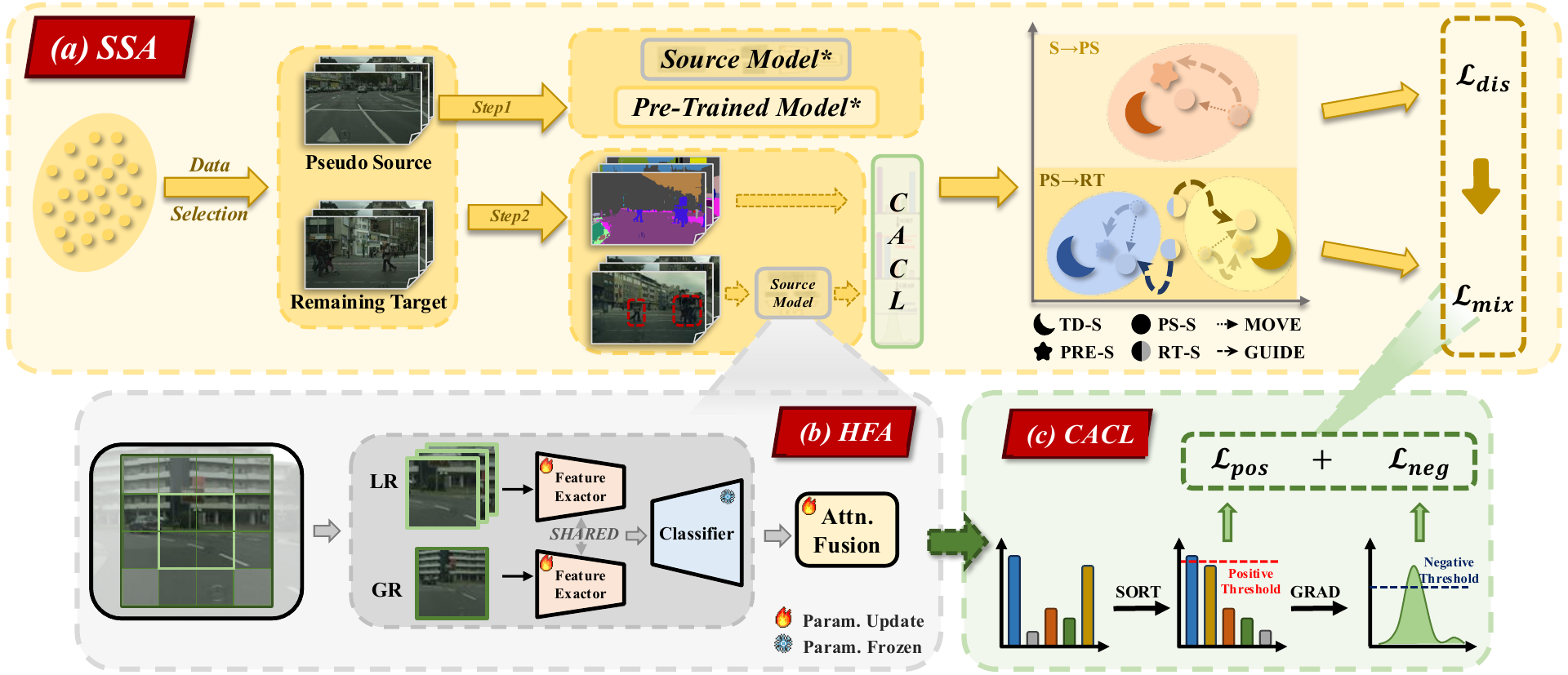}{}
    \caption{Framework of \textit{Stepwise Semantic Alignment} (SSA). \textbf{\textit{Source Model}}* and \textbf{\textit{Pre-Trained Model}}* are their respective feature extractors. TD-S and PRE-S refer to the \textbf{S}emantic features derived from the \textbf{T}arget-\textbf{D}omain (idealized semantics) and the \textbf{PRE}-Trained model, respectively. PS-S and RT-S correspond to the \textbf{S}emantic features extracted from $D_{ps}$ and $D_{rt}$. See \autoref{sec:appendix_pseudo_code} for pseudo-code. }
    \label{fig:framework}
    \vspace{-5pt}
\end{figure*}

\subsection{Problem definition}
    Given a model $\mathcal{M}_s$ trained on the source domain $D_s=\{\mathcal{X}_s,\mathcal{Y}_s\}=\{\boldsymbol{x}_s^i,y_s^i\}_{i=1}^{n_s}$ using pre-trained model $\mathcal{M}_{pre}$, and unlabeled target domain data $D_t=\{\mathcal{X}_t\}=\{\boldsymbol{x}_t^i\}_{i=1}^{n_t}$, where $n_s$ and $n_t$ denote the number of samples in the source and target domains respectively, TTA aims to fine-tune the source model $\mathcal{M}_s$ using $\mathcal{D}_t$ and obtain an adapted model $\mathcal{M}_t$ that can effectively handle domain shifts from the source to the target domain. To prevent ambiguity in the context of pseudo-source domain construction, the target domain is further partitioned in this work as follows: 
    \begin{equation}
        \begin{aligned}
    	D_t=D_{ps}+D_{rt}&=\{\{\mathcal{X}_{ps}\},\{\mathcal{X}_{rt}\}\} \\
        & = \{\{\boldsymbol{x}_{ps}^i\}_{i=1}^{n_{ps}},\{\boldsymbol{x}_{rt}^j\}_{j=1}^{n_{rt}}\},
        \end{aligned}
    \end{equation}
    where $D_{ps}$ and $D_{rt}$ denote the \textit{pseudo-source domain} samples and the \textit{remaining target domain} data, which are constructed through data selection. These subsets satisfy the conditions  $D_{ps}\cap D_{rt}=\emptyset$ and $D_{ps}\cup D_{rt}=D_t$. 
    %, and $n_{ps}+n_{rt}=n_t$. 
    %The pseudo-source domain $D_{ps}$ is constructed by partitioning $D_t$ based on sample self-entropy, ensuring both $D_{ps}$ and $D_{rt}$ are disjoint subsets of the target domain.

\subsection{Overview}

    As shown in \autoref{fig:framework}, SSA begins with data selection based on the source model's output probability distribution on target domain samples, forming the pseudo-source domain and the remaining target domain. Next, similarity alignment is performed using the source model and its pre-trained model to correct pseudo-source (S$\to$PS, where star represents pre-trained semantic PRE-S guiding pseudo-source semantic PS-S towards the moon representing target-domain semantic TD-S). Finally, the corrected pseudo-source semantics, now closer to target semantics, guide the remaining-target semantics RT-S that are farther away (PS$\to$RT), achieving overall semantic alignment in the target domain.

    In this process, \textit{Hierarchical Feature Aggregation} (HFA) and \textit{Confidence-Aware Complementary Learning} (CACL) serve as sub-modules for the \textit{feature extractor} and \textit{classifier}, respectively. HFA captures global and local semantic features using hierarchical windows and aggregates them through attention. CACL post-processes the model's output probability distribution, selecting negative labels by fully utilizing the relative changes in class probabilities within the distribution. We will first introduce the SSA method, followed by explanations of the HFA and CACL modules. 

    %During training, we compute the self-entropy of each sample and maintain an Entropy Bank via momentum-based updates. These entropy estimates are then used by \textbf{SDA} to partition the target domain into a pseudo-source subset $D_{ps}$ and a residual subset $D_{rt}$. To align the source and target domains, SDA first aligns pseudo-source features ($\text{HFA-PS-S}$) extracted by the task-specific $\text{HFA}^*$ with general semantics ($\text{PRE-S}$) from a frozen ImageNet-pretrained $\text{PRE-NET}^*$. This step refines the high-confidence pseudo-source features to better reflect transferable semantics. The corrected $\text{HFA-PS-S}$ then serves as a semantic bridge to guide low-confidence target features ($\text{HFA-RT-S}$) via a MixMatch-based semi-supervised strategy. This stepwise alignment enables phased adaptation from reliable to ambiguous regions, enhancing high-level semantic consistency under domain shift. \textbf{Pseudo-code} is provided in \autoref{sec:appendix_pseudo_code}.

\subsection{Stepwise Semantic Alignment}
    % During HFA module updates, the uncertainty of each target sample is measured by \textbf{self-entropy}. As the model is updated using Exponential Moving Average (EMA), self-entropy captures the distributional divergence between the target and source domains, with low-entropy samples being more similar to the source domain and high-entropy samples reflecting larger domain shifts.
    
    % To enable this partitioning, we maintain an \textit{Entropy Bank} that tracks the self-entropy of each sample throughout training. Similar to EMA, it updates the entropy value $\mathcal{H}_t(\boldsymbol{x})$ at each iteration $t$ by:
    
    \paragraph{Data Selection.} Similar to most SDE works, we first construct a pseudo-source domain using the source model and target domain samples. Inspired by previous works that use entropy to measure sample difficulty, we extend this approach to dynamic assessment. Specifically, we maintain an \textit{entropy memory} for sample during training, and update it across epochs $t$ via Exponential Moving Average (EMA):
    \begin{equation}
        \mathcal{H}_t(\boldsymbol{x}) = \alpha \cdot \mathcal{H}_{t-1}(\boldsymbol{x}) + (1 - \alpha) \cdot H(p(\boldsymbol{x})),
    \end{equation}
    
    where $H(p(\boldsymbol{x}))$ represents the current entropy computed from the predicted distribution $p(\boldsymbol{x})$, and $\alpha \in [0,1)$ is the momentum coefficient. This EMA-based update smooths entropy fluctuations, providing a stable estimate of sample uncertainty. After computing self-entropy, the former $\tau_{\text{par}}$ is assigned as the \textit{pseudo-source domain}, with the remaining samples as the \textit{reamining target domain}.
    
    %After constructing the pseudo-source domain, we employ a stepwise semantic alignment strategy. The \textit{pseudo-source domain} $D_{ps}$ is aligned with general semantics from source model's pretrained model $\mathcal{M}_{pre}$, refining high-confidence pseudo-source features. These corrected features then guide the adaptation of low-confidence \textit{remaining target domain} $D_{ps}$ samples, facilitating a smooth transition from reliable to ambiguous regions and enhancing semantic consistency under domain shift.
    
    \textbf{Pseudo-Source Semantic Correction.} To ensure semantic consistency and structural alignment between the $D_{s}$ and the $D_{ps}$, we introduce a feature alignment regularization that leverages universal semantic features extracted from a pretrained visual backbone. This approach guides the pseudo-source domain toward the semantic space of the source domain, promoting more coherent class alignment even in the presence of domain shifts.
    
    Given an input sample $\boldsymbol{x}$, let $\mathbf{f}(\boldsymbol{x}) \in \mathbb{R}^{C \times H \times W}$ represent the feature map generated by the current model, and $\mathbf{f}^{\text{pre}}(\boldsymbol{x}) \in \mathbb{R}^{C \times H \times W}$ the corresponding feature map extracted from a frozen pretrained model, where $\mathbf{f}(\cdot)$ represents the \textit{feature extractor} in \autoref{sec:hfa}. To enforce alignment between the features of the model and the pretrained model, we define a general feature alignment loss as follows:
    \begin{equation}
    	\label{eq_ldis}
    	\mathcal{L}_{\text{dis}} = \lambda_{\text{align}} \cdot \frac{1}{|\Omega|} \sum_{i \in \Omega} \text{Dist}\left(f(\boldsymbol{x})_i, f^{\text{pre}}(\boldsymbol{x})_i\right),
    \end{equation}
    
    where $\text{Dist}(\cdot,\cdot)$ denotes the cosine similarity between normalized feature vectors, used to measure feature-level affinity. The set $\Omega \subseteq \{1, \dots, H \times W\}$ denotes the indices of spatial locations considered for feature alignment, selected based on class-specific or task-specific criteria:
    \begin{equation}
    	\Omega = \left\{ i \,\middle|\, y_i \in \mathcal{C}_{\text{align}} \right\},
    \end{equation}
    where $y_i$ represents the available label or pseudo-label for location $i$, and $\mathcal{C}_{\text{align}}$ specifies the set of classes or regions targeted for feature alignment.

    This loss encourages the model to align its feature representations with those from the pretrained model, guiding the pseudo-source domain toward the source semantic structure and enhancing cross-domain feature consistency. For instance, if the model initially misclassifies a \textit{bus} as a \textit{truck}, alignment with general semantic features (e.g., \textit{vehicle}) from the pretrained model helps correct the representation. This guidance improves the reliability of the pseudo-source and provides a stable foundation for subsequent adaptation.
    
    \paragraph{Remaining-Target Semantic Alignment.} To facilitate semantic transfer from the $D_{ps}$ to the $D_{rt}$, we employ a class-aware feature mixing strategy within a semi-supervised learning framework. Specifically, we utilize pseudo-labels $y_{ps}$ and $y_{rt}$ generated for high- and low-confidence samples, respectively, to construct a class mask $\textbf{M}$ that defines region-wise semantic dominance. This mask enables interpolation between pseudo-source and target features at both the input and label levels. The resulting mixed sample $(\tilde{\boldsymbol{x}}_{\text{mix}}, \tilde{y}_{\text{mix}})$ is defined as:
    \begin{equation}
    	\label{eq_mixup}
        \begin{aligned}
    	\tilde{\boldsymbol{x}}_{\text{mix}} &= \textbf{M} \cdot \boldsymbol{x}_{ps} + (\textbf{1} - \textbf{M}) \cdot \boldsymbol{x}_{rt}, \\
        \tilde{y}_{\text{mix}} &= \textbf{M} \cdot y_{ps} + (\textbf{1} - \textbf{M}) \cdot y_{rt},
        \end{aligned}
    \end{equation}
    where $\boldsymbol{x}_{\text{ps}}$ and $\boldsymbol{x}_{\text{rt}}$ denote feature inputs from pseudo-source and target samples, and $y_{\text{ps}}$, $y_{\text{rt}}$ are their respective pseudo-labels. By enforcing prediction consistency on these mixed samples, the model is guided to propagate semantic structure from more reliable pseudo-source regions into uncertain target regions, improving decision boundary refinement and domain generalization.
    
    By leveraging the pseudo-source domain, which encapsulates high-confidence, semantically reliable regions, the mixed samples are infused with informative supervision that guides the learning of uncertain target representations. These mixed inputs are then used to optimize the model via a cross-entropy loss using CACL (\autoref{sec:cacl}):
    \begin{equation}
    	\label{eq_lce}
    	\mathcal{L}_{\text{mix}} = \text{CACL}(- \sum_{c} \tilde{y}_{\text{mix}}^{(c)} \log p(c \mid \tilde{\boldsymbol{x}}_{\text{mix}})),
    \end{equation}
    where $c$ indexes the semantic classes and $p(c \mid \tilde{\boldsymbol{x}}_{\text{mix}})$ denotes the predicted class probability from the model. This objective encourages consistency between the model's predictions and the interpolated supervision, refining decision boundaries under domain shift.
    
    Through this semi-supervised feature interpolation, the model progressively aligns uncertain target features with the more structured semantics of the pseudo-source. This process narrows the distributional gap and enhances generalization by leveraging pseudo-source guidance for the adaptation of less confident regions in the target domain.

\subsection{Hierarchical Feature Aggregation}
\label{sec:hfa}
    High-level semantic alignment is crucial for TTA, especially under large domain shifts with diverse object appearance and layout. Non-hierarchical representations, whether global or local, often miss important cues, hindering dense prediction. We introduce a \textbf{hierarchical feature aggregation} module that fuses local and global information across abstraction levels. A shared extractor encodes complementary spatial and semantic cues, producing both fine-grained and holistic representations.

    Formally, let $\boldsymbol{x} \in \mathcal{X}_t$ denote a target-domain input sample. To capture global semantic context, we downsample $\boldsymbol{x}$ and compute a coarse-grained prediction $P_{\text{global}}=\mathbf{g}(\mathbf{f}(\boldsymbol{x}))$, where $\mathbf{f}(\cdot)$ and $\mathbf{g}(\cdot)$ represents \textit{feature exactor} and \textit{classifier}. 
    
    To extract fine-grained local semantics, $\boldsymbol{x}$ is divided into overlapping local regions indexed by $\text{Grid}(\boldsymbol{x})$ in the spatial feature space, with each region processed by a shared $\mathbf{f}(\cdot)$ to extract localized features:
    \begin{equation}
    	\label{eq_flocal}
    	f_{\text{local}_i} = \left\{ \mathbf{f} \left(\boldsymbol{x}_{r_i} \right) \right\}, \quad r_i \in \text{Grid}(\boldsymbol{x}),
    \end{equation}
    
    Each patch-level prediction $P_i = \mathbf{g}(f_{\text{local}_i})$ is independently decoded using a forzen $\mathbf{g}(\cdot)$. These predictions are then aggregated to reconstruct a detailed, locally-aware output $P_{\text{local}}$. To properly handle overlapping regions, we maintain a count matrix that records the number of times each pixel is covered by a patch. The aggregated prediction is normalized by this count matrix $\boldsymbol{M}$ to ensure consistent pixel-wise contributions, where $\text{Pad}(\cdot)$ represents dimension padding alignment between $P_{i}$:
    \begin{equation}
    	\label{eq_local}
    	P_{\text{local}} = \frac{\sum_i \text{Pad}(P_i)}{\sum_i \boldsymbol{M}_i}.
    \end{equation}
    
    To effectively integrate both global and local semantic cues, we introduce a semantic-level attention mechanism to adaptively fuse $P_{\text{global}}$ and $P_{\text{local}}$. The fusion is defined as:
    \begin{equation}
    	\label{eq_fused}
    	P_{\text{fused}} = \textbf{A} \cdot P_{\text{local}} + (\mathbf{1} - \textbf{A}) \cdot \text{Align}(P_{\text{global}}).
    \end{equation}
    
    Here, $\textbf{A}_{ij} \in [0,1]$ denotes the attention weight that adaptively balances the contributions between global and local predictions. The operator $\text{Align}(\cdot)$ harmonizes their feature representations, where $P_{\text{global}}$ and $P_{\text{local}}$ jointly form the final prediction through this learned attention mechanism. % For more detailed implementation details, please refer to \autoref{sec:exp_details}.

    This framework fuses global context with fine-grained details, reinforcing semantic consistency and enabling robust cross-domain feature alignment for diverse visual tasks. 

\subsection{Confidence-Aware Complementary Learning}
\label{sec:cacl}
    To enhance high-level semantic alignment between pseudo-source and target domains, we propose a \textbf{confidence-aware complementary learning} strategy. Based on hierarchical pseudo-labels, we identify \textit{positive} classes with high confidence and \textit{negative} classes confidently rejected, providing complementary supervision that captures richer semantics and suppresses noisy predictions.
    
    In this learning strategy, let $C$ be the number of classes, and let $\boldsymbol{p} = [p_1, p_2, \dots, p_C] \in \mathbb{R}^C$ denote the predicted class probability distribution for a given input sample, satisfying:
    \begin{equation}
    	\sum_{i=1}^{C} p_i = 1, \quad p_i \geq 0.
    \end{equation}
    A confidence threshold $\tau_{\text{pos}}$ is introduced to identify positive predictions, where class $i$ is considered positive if $p_i \ge \tau_{\text{pos}}$. To justify the separation of predictions into positive and negative subsets, we introduce an entropy-based analysis in \autoref{theorem} (proof provided in \autoref{sec:theorem_proof}):
    
    \begin{theorem}%[Confidence-Separated Hypothesis Support Bound]
    	\label{theorem}
    	Let $\boldsymbol{p} \in \Delta^{C-1}$ be a categorical distribution over label space $\mathcal{Y} = \{1, \dots, C\}$, with entropy bounded by $\mathcal{H}(\boldsymbol{p}) \le H_0$. Then for any $\alpha \in (0,1)$, there exist thresholds $\tau_\alpha > \tau_\beta$ such that:
    	\begin{equation}
            \begin{aligned}
    		\mathcal{Y}_+, \mathcal{Y}_- &= \{c \mid p_c \ge \tau_\alpha\}, \{c \mid p_c \le \tau_\beta\},  \\
    		\mathcal{Y}_0 &= \mathcal{Y} \setminus (\mathcal{Y}_+ \cup \mathcal{Y}_-),
            \end{aligned}
    	\end{equation}
    	satisfying:
    	\begin{equation}
            \begin{aligned}
    		\mathbb{E}_{c \sim \boldsymbol{p}}[\log p_c \mid c \in \mathcal{Y}_+] &- \mathbb{E}_{c \sim \boldsymbol{p}}[\log(1 - p_c) \mid c \in \mathcal{Y}_-] \\
            &\ge \kappa(H_0, \tau_\alpha),
            \end{aligned}
    	\end{equation}
    	\begin{equation}
    		\sum_{c \in \mathcal{Y}_-} p_c \le \epsilon(H_0), \text{where } \epsilon(H_0) \to 0 \text{ as } H_0 \to 0.
    	\end{equation}
    \end{theorem}

%Building on this, the theorem demonstrates that low-entropy predictions allow confident partitioning of class probabilities into positive and negative subsets, providing a  foundation for selective learning from high-confidence predictions.  Beyond fixed thresholds, we propose an adaptive strategy that leverages the relative structure of $\boldsymbol{p}$.  Specially, high-confidence samples with a sharp decay in class probabilities are straightforward to classify, while low-confidence samples with flatter distributions introduce more ambiguity.  By exploiting these differences, the confidence-aware strategy leverages these differences to assign more negative labels to high-confidence samples and fewer to low-confidence ones, enabling robust learning from both confident and uncertain predictions.

    Building on \autoref{theorem}, the theorem demonstrates that low-entropy predictions allow confident partitioning of class probabilities into positive and negative subsets, providing a foundation for selective learning from high-confidence predictions. 
    
    Beyond fixed thresholds, we propose an adaptive strategy leveraging the relative structure of $\boldsymbol{p}$. Specially, high-confidence samples with sharp decay in class probabilities are straightforward to classify, while low-confidence samples with flatter distributions introduce ambiguity. By exploiting these differences, the confidence-aware strategy assigns more negative labels to high-confidence samples and fewer to low-confidence ones, enabling robust learning from both confident and uncertain predictions.

    To capture the relative confidence among class predictions, we first sort the predicted probabilities in descending order:
    \begin{equation}
    	\label{eq_psort}
    	\boldsymbol{p}_{\text{sorted}} = [p_{(1)}, p_{(2)}, \dots, p_{(C)}].
        % , p_{(1)} \ge p_{(2)} \ge \cdots \ge p_{(C)}.
    \end{equation}
    Based on the identified drop, we construct a ternary mask $\mathbf{m} \in \{-1, 0, 1\}^C$ to categorize each class $j$, and use it to define the \textit{complementary learning loss}:
    \begin{equation}
    	\label{eq_mask}
    	m_j = 
    	\begin{cases}
    		1, & \text{if } p_j \ge \tau_{\text{pos}}, \\
    		-1, & \text{if } j > i^*, \\
    		0, & \text{otherwise},
    	\end{cases}
    \end{equation}
    \begin{equation}
    	\label{eq_lcl}
        \begin{aligned}
        	\mathcal{L}_{CACL} = -\frac{1}{n_t} \sum_{i=1}^{n_t} \sum_{j=1}^{|\boldsymbol{x}_{t,i}|} \sum_{c=1}^{C} 
        	[
        	\boldsymbol{1}_\mathrm{(m_j = 1)} \log p_{\boldsymbol{x}_t}^{(i,j,c)} + \\
        	\boldsymbol{1}_\mathrm{(m_j = -1)} \log (1 - p_{\boldsymbol{x}_t}^{(i,j,c)})
        	],
        \end{aligned}
    \end{equation}
    where $i^*=\text{min}\{i|r_i\ge \tau_{\text{neg}}\}$, $\displaystyle \boldsymbol{1}_\mathrm{(\cdot)}$ denotes the indicator function, and $|\boldsymbol{x}_{t,i}|$ represents the number of prediction units (e.g., instances or positions) in the $i$-th target sample.
    
    By leveraging relative confidence gaps to distinguish confident predictions from rejections, this strategy improves semantic discrimination and enhances learning under domain shift. \textbf{Notably}, rather than treating unselected classes as positives, we apply an absolute threshold $\tau_{\text{pos}}$ to select confident positives. This approach avoids misclassifying ambiguous, hard-to-define classes as positives, ensuring reliable learning from clear, high-confidence labels.

\section{Experiments}
\label{experiments}

\begin{table*}[t]
	\caption{Semantic segmentation main results of SSA on different tasks. \textbf{SF} denotes the \textit{Source Free} method, which does not use source domain data during the alignment process. GTA5, Cityscapes, and ACDC share the same 19 categories, while SYNTHIA has only 16 of them, and its mIoU results are scaled accordingly.}
	\label{seg_results}
	\centering
	\small
	\resizebox{0.95\textwidth}{!}{%
		\begin{tabular}{l>{\centering\arraybackslash}p{1.6em}*{20}{>{\centering\arraybackslash}p{1.5em}}}
			\toprule
			Method & SF 
			& \rotatebox{90}{road} & \rotatebox{90}{side.} & \rotatebox{90}{build.}
			& \rotatebox{90}{wall} & \rotatebox{90}{fence} & \rotatebox{90}{pole}
			& \rotatebox{90}{light} & \rotatebox{90}{sign} & \rotatebox{90}{vege.}
			& \rotatebox{90}{terr.} & \rotatebox{90}{sky} & \rotatebox{90}{person}
			& \rotatebox{90}{rider} & \rotatebox{90}{car} & \rotatebox{90}{truck}
			& \rotatebox{90}{bus} & \rotatebox{90}{train} & \rotatebox{90}{motor.}
			& \rotatebox{90}{bike} & \rotatebox{90}{\textbf{mIoU}}\\
			\midrule
			\multicolumn{22}{c}{\textit{\textbf{GTA5 $\rightarrow$ Cityscapes (CS)}}}\\
			\midrule
			TransDA-B & \ding{55}  & 94.7 & 64.2  & 89.2   & 48.1 & 45.8  & 50.1 & 60.2  & 40.8 & 90.4  & 50.2  & 93.7 & 76.7   & 47.6  & 92.5 & 56.8  & 60.1 & 47.6  & 49.6  & 55.4 & 63.9 \\
			DAFormer  & \ding{55}  & 95.7 & 70.2  & 89.4   & 53.5 & 48.1  & 49.6 & 55.8  & 59.4 & 89.9  & 47.9  & 92.5 & 72.2   & 44.7  & 92.3 & 74.5  & 78.2 & 65.1  & 55.9  & 61.8 & 68.3 \\
			HRDA      & \ding{55}  & 96.4 & 74.4  & 91.0   & 61.6 & 51.5  & 57.1 & 63.9  & 69.3 & 91.3  & 48.4  & 94.2 & 79.0   & 52.9  & 93.9 & 84.1  & 85.7 & 75.9  & 63.9  & 67.5 & 73.8 \\
			IDM       & \ding{55}  & 97.2 & 77.1  & 89.8   & 51.7 & 51.7  & 54.5 & 59.7  & 64.7 & 89.2  & 45.3  & 90.5 & 74.2   & 46.6  & 92.3 & 76.9  & 59.6 & 81.2  & 57.3  & 62.4 & 69.5 \\
			\midrule
			DAFormer  & \ding{51}  & 87.7 & 33.4  & 83.9   & 28.1 & 27.5  & 35.9 & 42.9  & 28.7 & 82.4  & 28.6  & 83.1 & 65.0   & 37.0  & 85.8 & 53.9  & 46.3 & 31.8  & 23.6  & 36.8 & 49.6 \\
			HRDA     & \ding{51}  & 83.3 & 28.2  & 83.3   & 43.3 & 22.2  & 42.9 & 47.7  & 38.2 & 87.2  & 40.0  & 81.6 & 69.5   & 35.9  & 84.8 & 42.7  & 50.4 & 41.2  & 33.7  & 29.6 & 51.9\\
			IDM       & \ding{51}  & 93.9 & 59.1  & 86.6   & 35.3 & 30.4  & 42.2 & 45.1  & 57.8 & 88.4  & 35.1  & 89.4 & 69.7   & 39.8  & 89.1 & 66.8  & 46.0 & 13.5  & 41.1  & 61.2 & 57.4 \\
			ATP      & \ding{51}  & \textcolor{red}{\textbf{96.6}} & \textcolor{red}{\textbf{75.3}}  & \textcolor{red}{\textbf{89.4}}  & 50.2 & \textcolor{red}{\textbf{41.5}}  & 47.5 & 48.6  & 61.1 & 89.8  & \textcolor{red}{\textbf{48.3}}  & \textcolor{red}{\textbf{93.4}} & 70.4   & 40.1  & 89.8 & 66.7  & 58.2 & 30.3  & 53.4  & \textcolor{red}{\textbf{65.6}} & 64.0 \\
			%\midrule
			\rowcolor{gray!30}
			\textbf{SSA}     & \ding{51}  & 93.0 & 61.0  & 88.8   & \textcolor{red}{\textbf{51.8}} & 33.9  & \textcolor{red}{\textbf{54.3}} & \textcolor{red}{\textbf{62.5}}  & \textcolor{red}{\textbf{66.5}} & \textcolor{red}{\textbf{89.8}}  & 43.3  & 92.5 & \textcolor{red}{\textbf{78.1}}   & \textcolor{red}{\textbf{47.2}}  & \textcolor{red}{\textbf{93.0}} & \textcolor{red}{\textbf{78.6}}  & \textcolor{red}{\textbf{79.6}} & \textcolor{red}{\textbf{74.6}}  & \textcolor{red}{\textbf{63.8}}  & 63.4  & \textcolor{red}{\textbf{69.2}} \\
			
			\midrule
			\multicolumn{22}{c}{\textit{\textbf{SYNTHIA $\rightarrow$ Cityscapes (CS)}}}\\
			\midrule
			TransDA-B & \ding{55}  & 90.4 & 54.8  & 86.4   & 31.1 & 1.7  & 53.8 & 61.1  & 37.1 & 90.3  & - & 93.0  & 71.2 & 25.3   & 92.3   & - & 66.0 & -  & 44.4  & 49.8 & 59.3  \\
			DAFormer  & \ding{55}  & 84.5 & 40.7  & 88.4   & 41.5 & 6.5  & 50.0 & 55.0  & 54.6 & 86.0 & - & 89.8  & 73.2 & 48.2   & 87.2  & - & 53.2 & - & 53.9  & 61.7 & 60.9  \\
			HRDA      & \ding{55}  & 85.2 & 47.7  & 88.8   & 49.5 & 4.8  & 57.2 & 65.7  & 60.9 & 85.3  & - & 92.9  & 79.4 & 52.8   & 89.0   & - & 64.7 & - & 63.9  & 64.9 & 65.8  \\
			\midrule
			DAFormer  & \ding{51}  & 64.3 & 25.1  & 78.5   & 23.8 &\textcolor{red}{ \textbf{1.9}}  & 37.3 & 29.7  & 22.8 & 80.4  & - & 83.0  & 65.1 & 26.6   & 69.8  & - & 38.3 & - & 22.7   & 32.8 & 43.8  \\
			HRDA     & \ding{51}  & 72.2 & 26.6  & 80.8   & 23.0 & 0.5  & 42.5 & 41.0  & 31.5 & 84.3  & - & 86.2  & 64.3 & 29.3   & 73.5  & -  & 28.8 & - & 12.4  & 41.6 & 46.1  \\
			IDM       & \ding{51}  & 82.2 & 37.9  & 83.5   & 20.3 & 1.5  & 47.3 & 41.7  & 25.6 & 84.4  & - & 86.8  & 61.6 & 25.0   & 87.6  & - & 43.7 & - & 30.2  & 36.4 & 49.7  \\
			MISFIT    & \ding{51}  & 80.2 & 38.5  & 85.9   & 30.3 & 1.2  & 52.3 & 56.8  & 29.0 & \textcolor{red}{\textbf{89.9}}  & - & 88.3  & 68.1 & 10.8   & \textcolor{red}{\textbf{92.1}}  & - & \textcolor{red}{\textbf{69.0}} & - & 26.3  & 52.6 & 54.5  \\
			ATP       & \ding{51}  & \textcolor{red}{\textbf{90.6}} & \textcolor{red}{\textbf{54.4}}  & 86.7   & 28.5 & 0.5  & 50.3 & 52.4  & 50.5 & 87.4  & - & \textcolor{red}{\textbf{93.4}}  & 70.2 & 35.8   & 89.6 & -  & 53.5 & - & 50.6  & 51.1 & 59.1  \\
			%\midrule
			\rowcolor{gray!30}
			\textbf{SSA}     & \ding{51}  & 90.0 & 51.9  & \textcolor{red}{\textbf{87.0}}   & \textcolor{red}{\textbf{35.4}} & 1.0  & \textcolor{red}{\textbf{56.7}} & \textcolor{red}{\textbf{64.6}} & \textcolor{red}{\textbf{53.6}} & 88.5 & - & 92.3 & \textcolor{red}{\textbf{78.3}} & \textcolor{red}{\textbf{44.4}}   & 89.8  & - & 63.4 & - & \textcolor{red}{\textbf{62.8}}  & \textcolor{red}{\textbf{65.1}} & \textcolor{red}{\textbf{64.1}}  \\
			
			\midrule
			\multicolumn{22}{c}{\textit{\textbf{Cityscapes (CS) $\rightarrow$ ACDC}}}\\
			\midrule
			TENT      & \ding{51}  & 85.3 & 50.2  & 85.4   & 45.4 & 32.7  & 50.4 & 59.4  & 66.1 & 86.4  & 45.7  & 97.5 & 57.9   & 53.8  & 84.7 & 51.0  & 66.9 & 72.4  & 40.2  & 50.1 & 62.2 \\
			CoTTA    & \ding{51}  & 85.7 & 50.9  & \textcolor{red}{\textbf{85.9}}   & 45.9 & 33.6  & 54.8 & 62.3  & \textcolor{red}{\textbf{69.9}} & 87.1  & 45.7  & \textcolor{red}{\textbf{97.7}} & 63.3   & \textcolor{red}{\textbf{59.4}}  & 85.1 & 52.8  & 68.0 & 74.1  & 44.9  & 55.1 & 64.4 \\
			DePT    & \ding{51}  & 85.0 & 50.6  & 85.5   & 45.7 & 33.2  & 53.9 & 61.6  & 69.4 & 86.7  & 45.5  & 97.4 & 62.6   & 59.2  & 85.1 & 52.5  & 68.0 & 73.7  & 44.3  & 54.3 & 63.9 \\
			VDP      & \ding{51}  & 85.7 & 50.9  & \textcolor{red}{\textbf{85.9}}   & 45.9 & 33.6  & 54.8 & 62.2  & \textcolor{red}{\textbf{69.9}} & 87.0  & 45.7  & 97.6 & 63.3   & 59.2  & 85.1 & 52.8  & 68.1 & 74.1  & 44.8  & \textcolor{red}{\textbf{54.9}} & 64.3 \\
			IDM      & \ding{51}  & \textcolor{red}{\textbf{88.8}} & \textcolor{red}{\textbf{63.2}}  & 85.8   & 45.5 & 30.3  & 42.1 & 69.7  & 62.7 & \textcolor{red}{\textbf{87.4}}  & \textcolor{red}{\textbf{51.7}}  & 96.5 & 61.8   & 29.3  & 86.5 & 80.5  & 68.9 & 70.1  & \textcolor{red}{\textbf{57.0}}  & 52.8 & 64.8 \\
			%\midrule
			\rowcolor{gray!30}
			\textbf{SSA}     & \ding{51}  & 84.1 & 54.7  & 84.6   & \textcolor{red}{\textbf{51.7}} & \textcolor{red}{\textbf{43.7}}  & \textcolor{red}{\textbf{61.5}} & \textcolor{red}{\textbf{72.5}}  & 59.5 & 72.2  & 38.5  & 79.1 & \textcolor{red}{\textbf{63.8}}   & 44.6  & \textcolor{red}{\textbf{88.0}} & \textcolor{red}{\textbf{84.4}}  & \textcolor{red}{\textbf{76.6}} & \textcolor{red}{\textbf{78.3}}  & 48.2  & 53.4 & \textcolor{red}{\textbf{65.2}} \\
			\bottomrule
	\end{tabular}}
	\vspace{-5pt}
\end{table*}

\subsection{Datasets and evaluation mtrics}
    We evaluate SSA on semantic segmentation and image classification (single/multi-label) benchmarks. For segmentation, with its dense pixel-level supervision and higher cross-domain difficulty, provides a rigorous testbed. We use GTA5$\rightarrow$Cityscapes, SYNTHIA$\rightarrow$Cityscapes, and Cityscapes$\rightarrow$ACDC: GTA5~\cite{GTA5} and SYNTHIA~\cite{SYNTHIA} are synthetic datasets with pixel-level labels, while Cityscapes~\cite{Cityscapes} and ACDC~\cite{ACDC} contain real urban scenes, the latter emphasizing challenging conditions (e.g., night, fog). Performance is measured by mean IoU. For single-label classification, we use Office-31~\cite{Office-31}, Office-Home, VisDA-C, and DomainNet-126 covering datasets from small- to large-scale with diverse domain shifts, and report Top-1 accuracy. For multi-label classification, we use COCO2014/2017~\cite{COCO} and VOC2007/2012~\cite{VOC}. Note that the main experiment focuses on single-label classification, while multi-label classification is used to analyze the semantic-intensive scaling effects of SSA provided in \autoref{sec:scaling_ana}. More comparison methods and dataset information are provided in \autoref{sec:competitors}.
    
\subsection{Implementation details}
\label{imp_details}
    For semantic segmentation, we mainly use MiT-B5~\cite{segformer} as the backbone, and ResNet-50/101 for classification. Segmentation is implemented in MMSeg~\cite{mmseg2020}, while classification follows the SHOT~\cite{SHOT} codebase. Optimization employs SGD with learning rate $2.5 \times 10^{-4}$, momentum $0.9$, and weight decay $5 \times 10^{-4}$, using a \textit{poly} schedule $(1 - \frac{\textit{iter}}{\textit{max\_iter}})^{0.9}$. During entropy minimization the classifier is frozen; otherwise it is updated with learning rate $2.5 \times 10^{-3}$. \textbf{Note that} the Test-Time Domain Adaptation (TTDA), also known as Source-Free Domain Adaptation (SFDA)~\cite{liang2025comprehensive}. Detailed implementation details for methods in \autoref{method} are provided in \autoref{sec:exp_details}, parameter analysis provided in \autoref{sec:parameter_ana}, computational analysis provided in \autoref{sec:appendix_cost_ana}.
    %with online TTA results referenced in the \autoref{otta} (more description in \autoref{sec:exp_details}).
    
\subsection{Main Results of Semantic Segmentation}
    As shown in \autoref{seg_results}, we evaluate SSA on three UDA benchmarks: GTA5$\to$Cityscapes, SYNTHIA$\to$Cityscapes, and Cityscapes$\to$ACDC. Specially, on GTA5$\to$ task, SSA attains 69.2 mIoU, $+$5.2 over the prior best, with gains in frequent classes (\textit{car}, \textit{road}) and fine-grained or ambiguous ones (\textit{person}, \textit{bus}, \textit{train}). On SYNTHIA$\to$Cityscapes tsak, despite larger content/layout gaps, SSA delivers 64.1 mIoU (16-class), surpassing all source-free methods. Better scores on \textit{pole} (56.7), \textit{rider} (44.4), and \textit{motorcycle} (62.8) indicate that semantic alignment alleviates boundary instability typical of shallow features. On Cityscapes$\to$ACDC task, under fog, night, and rain, SSA reaches 65.2 mIoU, preserving semantic separability even when low-level cues degrade.
    
    %Overall, SHLSA’s semantic-aware stepwise alignment reduces \textbf{class confusion} and \textbf{unstable boundaries} beyond SDE, enabling reliable performance across diverse urban scenes (visualization and qualitative analyses in \autoref{seg_analysis_effect}). 

\subsection{Main Results of Single-label Image Classification}
    \begin{table*}[t]
	\centering
	\caption{Single-Label image classification performance of SSA on different tasks. \textbf{SF} denotes the \textit{Source Free} method.}
	\label{cls_results}
	\small
	\resizebox{0.9\textwidth}{!}{
		\begin{tabular}{
				l
				>{\centering\arraybackslash}p{2em}
				*{13}{>{\centering\arraybackslash}p{2.5em}}
			}
			\toprule
			% ==== 第三块 ====
			\multicolumn{15}{c}{\textit{\textbf{Office-31 (ResNet-50)}}}\\
			\cmidrule(lr){1-15}
			Method & SF &  \multicolumn{2}{c}{A$\to$D} & \multicolumn{2}{c}{A$\to$W} & \multicolumn{2}{c}{D$\to$A} & \multicolumn{2}{c}{D$\to$W} & \multicolumn{2}{c}{W$\to$A} & \multicolumn{2}{c}{W$\to$D} & Avg\\
			\midrule
			SHOT & \ding{51} & \multicolumn{2}{c}{93.7} & \multicolumn{2}{c}{91.1} & \multicolumn{2}{c}{74.2} & \multicolumn{2}{c}{98.2} & \multicolumn{2}{c}{74.6} & \multicolumn{2}{c}{\textcolor{red}{\textbf{100.}}} & 88.6 \\
			%			ELR & \ding{51} & \multicolumn{2}{c}{93.8} & \multicolumn{2}{c}{93.3} & \multicolumn{2}{c}{76.2} & \multicolumn{2}{c}{98.0} & \multicolumn{2}{c}{76.9} & \multicolumn{2}{c}{\textcolor{red}{\textbf{100.}}} & 89.6 \\
			DIFO & \ding{51} & \multicolumn{2}{c}{\textcolor{red}{\textbf{97.2}}} & \multicolumn{2}{c}{95.5} & \multicolumn{2}{c}{83.0} & \multicolumn{2}{c}{97.2} & \multicolumn{2}{c}{83.2} & \multicolumn{2}{c}{98.8} & 92.5 \\
			ProDe & \ding{51} & \multicolumn{2}{c}{96.8} & \multicolumn{2}{c}{96.4} & \multicolumn{2}{c}{\textcolor{red}{\textbf{83.1}}} & \multicolumn{2}{c}{97.0} & \multicolumn{2}{c}{\textcolor{red}{\textbf{82.5}}} & \multicolumn{2}{c}{99.8} & 92.6 \\
			%\midrule
			\rowcolor{gray!30}
			\textbf{SSA}      & \ding{51} & \multicolumn{2}{c}{97.0} & \multicolumn{2}{c}{\textcolor{red}{\textbf{96.9}}} & \multicolumn{2}{c}{82.6} & \multicolumn{2}{c}{\textcolor{red}{\textbf{98.2}}} & \multicolumn{2}{c}{82.1} & \multicolumn{2}{c}{99.8} & \textcolor{red}{\textbf{92.8}} \\
			
			% ==== 第二块 ====
			\midrule
			\multicolumn{15}{c}{\textit{\textbf{Office-Home (ResNet-50)}}}\\
			\cmidrule(lr){1-15}
			Method & SF & Ar$\to$Cl & Ar$\to$Pr & Ar$\to$Rw & Cl$\to$Ar & Cl$\to$Pr & Cl$\to$Rw &	Pr$\to$Ar & Pr$\to$Cl & Pr$\to$Rw & Rw$\to$Ar & Rw$\to$Cl & Rw$\to$Pr & Avg.\\
			\midrule
			PDA & \ding{55} & 55.4 & 85.1 & 85.8 & 75.2 & 85.2 &  85.2 & 74.2 & 55.2 & 85.8 & 74.7 & 55.8 & 86.3 & 75.3 \\
			DAMP & \ding{55} & 59.7 & 88.5 & 86.8 & 76.6 & 88.9 &  87.0 & 76.3 & 59.6 & 87.1 & 77.0 & 61.0 & 89.9 & 78.2 \\
			\midrule
			SHOT & \ding{51} & 57.1 & 78.1 & 81.5 & 68.0 & 78.2 & 78.1 & 67.4 & 54.9 & 82.2 & 73.3 & 58.8 & 84.3 & 71.8 \\
			%			ATP & \ding{51} & 59.3 & 78.5 & 82.6 & 69.4 & 82.4 & 77.7 & 68.0 & 56.3 & 82.3 & 74.3 & 62.3 & 84.5 & 73.1 \\
			DIFO & \ding{51} & 70.6 & 90.6 & 88.8 & 82.5 & 90.6 & 88.8 & 80.9 & 70.1 & 88.9 & 83.4 & 70.5 & 91.2 & 83.1 \\
			ProDe & \ding{51} & 72.7 & \textcolor{red}{\textbf{92.3}} & 90.5 & 82.5 & 91.5 & 90.7 & 82.5 & 72.5 & \textcolor{red}{\textbf{90.8}} & 83.0 & 72.6 & 92.2 & 84.5 \\
			%\midrule
			\rowcolor{gray!30}
			\textbf{SSA}      & \ding{51} & \textcolor{red}{\textbf{73.1}} & 91.9 & \textcolor{red}{\textbf{91.2}} & \textcolor{red}{\textbf{84.0}} & \textcolor{red}{\textbf{91.6}} & \textcolor{red}{\textbf{90.8}} & \textcolor{red}{\textbf{82.8}} & \textcolor{red}{\textbf{73.7}} & 90.7 & \textcolor{red}{\textbf{83.4}} & \textcolor{red}{\textbf{74.7}} & \textcolor{red}{\textbf{92.4}} & \textcolor{red}{\textbf{85.0}} \\
			
			% ==== 第一块 ====
			\midrule 
			\multicolumn{15}{c}{\textit{\textbf{VisDA-C (ResNet-101)}}}\\
			\cmidrule(lr){1-15}
			Method & SF & plane & bike & bus & car & horse & knife & mcycle & person & plant & sktbrd & train & truck & Avg.\\
			\midrule
			STAR & \ding{55}  & 95.0 & 84.0 & 84.6 & 73.0 & 91.6 & 91.8 & 85.9 & 78.4 & 94.4 & 84.7 & 87.0 & 42.2 & 82.7 \\
			RWOT & \ding{55}  & 95.1 & 80.3 & 83.7 & 90.0 & 92.4 & 68.0 & 92.5 & 82.2 & 87.9 & 78.4 & 90.4 & 68.2 & 84.0 \\
			\midrule
			SHOT & \ding{51}  & 94.3 & 88.5 & 80.1 & 57.3 & 93.1 & 94.9 & 80.7 & 80.3 & 91.5 & 89.1 & 86.3 & 58.2 & 82.9 \\
			%PLUE & \ding{51}  & 97.3 & 96.2 & 90.5 & \textcolor{red}{\textbf{91.8}} & 90.0 & 94.2 & 87.4 & 87.7 & 97.0 & 84.3 & 93.0 & 81.0 & 90.0 \\
			ATP & \ding{51}  & 97.6 & 91.8 & \textcolor{red}{\textbf{88.7}} & 73.1 & 97.6 & 92.9 & 92.0 & \textbf{95.7} & 93.4 & 89.0 & 87.9 & 71.3 & 89.3 \\
			ProDe & \ding{51}  & 98.3 & \textcolor{red}{\textbf{92.4}} & 86.6 & 80.5 & 98.1 & \textcolor{red}{\textbf{98.0}} & 92.3 & 84.3 & 94.7 & \textcolor{red}{\textbf{97.0}} & 94.1 & 75.6 & 91.0 \\
			%\midrule
			\rowcolor{gray!30}
			\textbf{SSA}   & \ding{51}  & \textcolor{red}{\textbf{98.3}} & 90.0 & 88.0 & \textcolor{red}{\textbf{87.2}} & \textcolor{red}{\textbf{98.3}} & 97.3 & \textcolor{red}{\textbf{93.5}} & 86.3 & \textcolor{red}{\textbf{97.7}} & 96.6 & \textcolor{red}{\textbf{95.7}} & \textcolor{red}{\textbf{76.8}} & \textcolor{red}{\textbf{92.1}} \\
    		\midrule
            \multicolumn{15}{c}{\textit{\textbf{DomainNet-126 (ResNet-101)}}}\\
            \midrule
    		SHOT  & \ding{51} & 63.5 & 78.2 & 59.5 & 67.9 & 81.3 & 61.7 & 67.7 & 67.6 & 57.8 & 70.2 & 64.0 & 78.0 & 68.1 \\ 
    		ProDe & \ding{51} & 79.3 & 91.0 & 75.3 & 80.0 & \textbf{\textcolor{red}{90.9}} & 75.6 & 80.4 & 78.9 & 75.4 & \textbf{\textcolor{red}{80.4}} & 79.2 & 91.0 & 81.5 \\ 
    		%\midrule
            \rowcolor{gray!30}
    		\textbf{SSA} & \ding{51} & \textbf{\textcolor{red}{81.2}} & \textbf{\textcolor{red}{92.3}} & \textbf{\textcolor{red}{77.1}} & \textbf{\textcolor{red}{83.2}} & \textbf{\textcolor{red}{90.9}} & \textbf{\textcolor{red}{77.4}} & \textbf{\textcolor{red}{82.3}} & \textbf{\textcolor{red}{79.2}} & \textbf{\textcolor{red}{77.3}} & 80.1 & \textbf{\textcolor{red}{82.3}} & \textbf{\textcolor{red}{93.4}} & \textbf{\textcolor{red}{83.1}} \\ 
			\bottomrule
		\end{tabular}
	}
\end{table*}

    As shown in \autoref{cls_results}, we evaluate SSA on Office-31, Office-Home, VisDA-C and DomainNet-126 under source-free settings to assess its ability. Specially, on Office-31 SSA reaches 92.8\% across 6 shifts. On Office-Home, SSA attains 85.0\% average accuracy over 12 domain shifts, while on VisDA-C and DomainNet-126, SSA achieves 92.1\% and 83.1\% average accuracy respectively, outperforming prior methods through multi-level semantic alignment. Strong results on \textit{plane}, \textit{horse}, and \textit{train} show how hierarchical feature aggregation improves semantic separability.

\section{Analysis}
\label{sec:analysis}

\subsection{Ablation Study}
    %As shown in \autoref{albation}, introducing \textbf{HFA} alone boosts segmentation performance from 44.5 to 57.5 mIoU by aggregating local and global features, which enhances high-level semantic representations; however, it shows minimal impact on classification (84.6\%), suggesting spatial context is more crucial in dense prediction tasks (more analysis is in the \autoref{scaling}). Adding \textbf{CACL} further improves both segmentation (65.6 mIoU) and classification (84.7\%) by leveraging high-confidence predictions and uncertain regions to provide complementary supervision. Finally, \textbf{SDA} brings consistent gains across tasks, with segmentation mIoU reaching 69.2 and classification accuracy improving to 85.0\%, highlighting the effectiveness of entropy-aware domain partitioning and progressive alignment. These results collectively demonstrate the complementary strengths of the three modules and the versatility of SHLSA across heterogeneous tasks. More detailed results are in the \autoref{appendix_albation}.

    As shown in \autoref{albation}, we evaluate the interactions among modules in SSA using GTA5$\to$Cityscapes. Specially, introducing \textbf{HFA} alone boosts segmentation performance from 44.5 to 57.5 mIoU by aggregating local and global features, enhancing high-level semantic representations. Adding \textbf{CACL} further improves segmentation (65.6 mIoU) by leveraging high-confidence predictions and uncertain regions for complementary supervision. Finally, \textbf{SDA} brings gains  with segmentation mIoU reaching 69.2, highlighting the effectiveness of stepwise alignment. Furthermore, these findings demonstrate that the three components exhibit strong inter-dependencies, with HFA serving as the foundation for enhancing the effectiveness of both CACL and SSA, while CACL provide complementary improvements that collectively contribute to the superior performance of the complete SSA framework. More detailed analysis and implementation details provided in \autoref{sec:appendix_albation}.
    \begin{table}[h]
	\centering
    \vskip -0.1in
	\caption{Ablation Studies on GTA5$\to$Cityscapes.}
	\label{albation}
	%\small
    \resizebox{0.35\textwidth}{!}{
	\begin{tabular}{ccccc}
		\toprule
		Method & HFA & CACL & SSA & mIoU \\ 
		\midrule
		Baseline &  &  &  & 44.5 \\ 
		(a) & \ding{51} &  &  & 57.5 \\ 
		(b) &  & \ding{51} &  & 50.2 \\ 
		(c) &  &  & \ding{51} & 53.4 \\ 
		(d) & \ding{51} & \ding{51} &  & 65.6 \\ 
		(e) & \ding{51} &  & \ding{51} & 66.7 \\ 
		(f) &  & \ding{51} & \ding{51} & 60.3 \\ 
		(g) & \ding{51} & \ding{51} & \ding{51} & 69.2 \\ 
		\bottomrule
	\end{tabular}
	}
    \vskip -0.1in
\end{table}

\subsection{SSA Effectiveness Visualization}
    To more intuitively demonstrate the effectiveness of SSA on downstream tasks, we present the visualization results of semantic segmentation. As shown in \autoref{fig:seg_effect}, compared to the baseline, SSA significantly improves semantic structure and boundary accuracy. It recovers fine details in challenging regions such as traffic participants (\textit{person}, \textit{rider}) and urban infrastructure (\textit{pole}, \textit{traffic sign}), which are often under-segmented or misclassified by source models.

    \begin{figure}[H]
    	\centering
        \vskip -0.1in
    	\includegraphics[width=0.88\columnwidth]{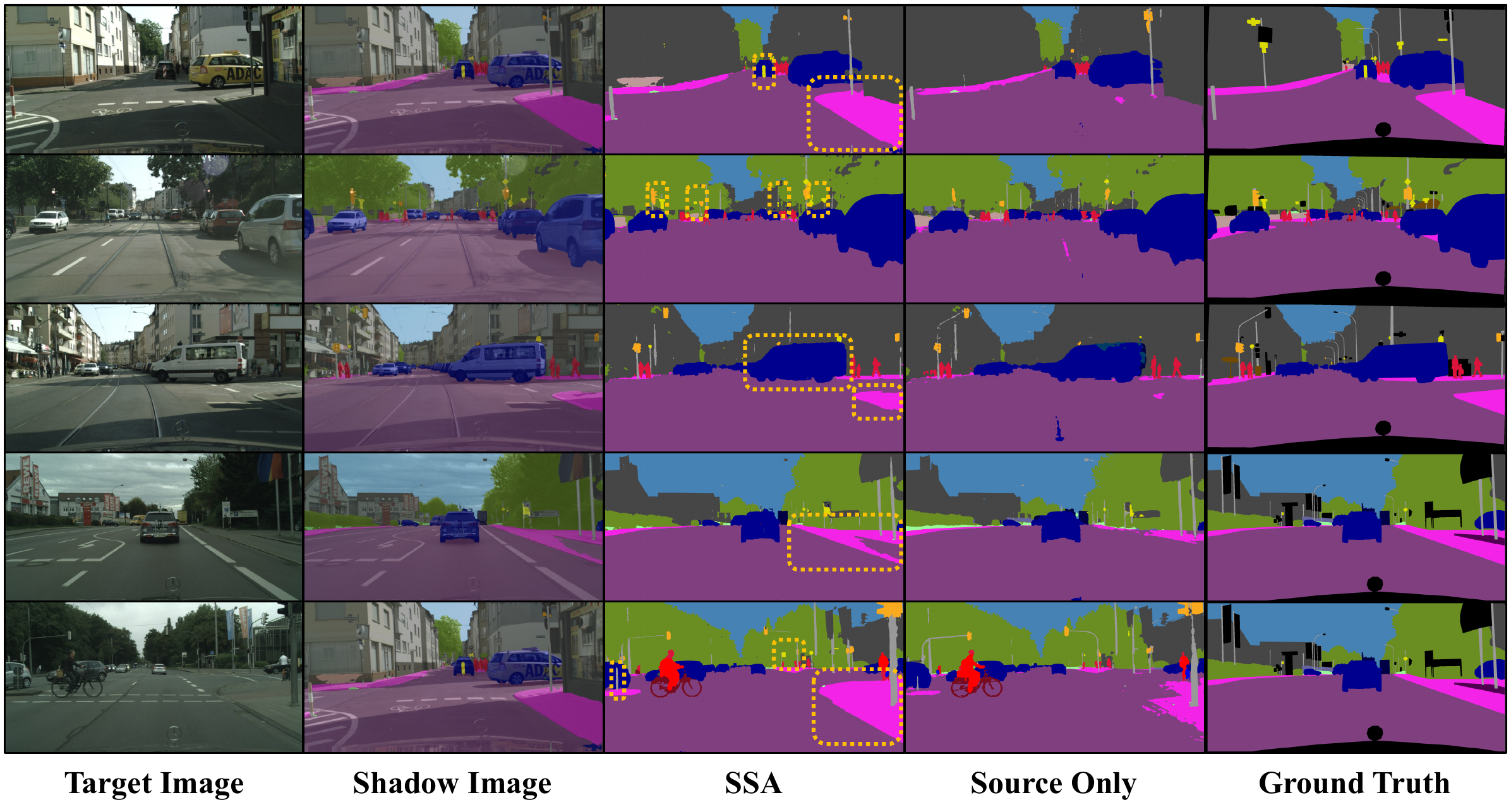}{}
    	\caption{SSA segmentation visualization on GTA5$\to$Cityscapes.}
        \vskip -0.1in
    	\label{fig:seg_effect}
        \vskip -0.1in
    \end{figure}

    \begin{figure}[ht]
    	\centering
        \vskip -0.1in
        \begin{subfigure}[b]{0.40\columnwidth}
            \centering
            \includegraphics[width=\textwidth]{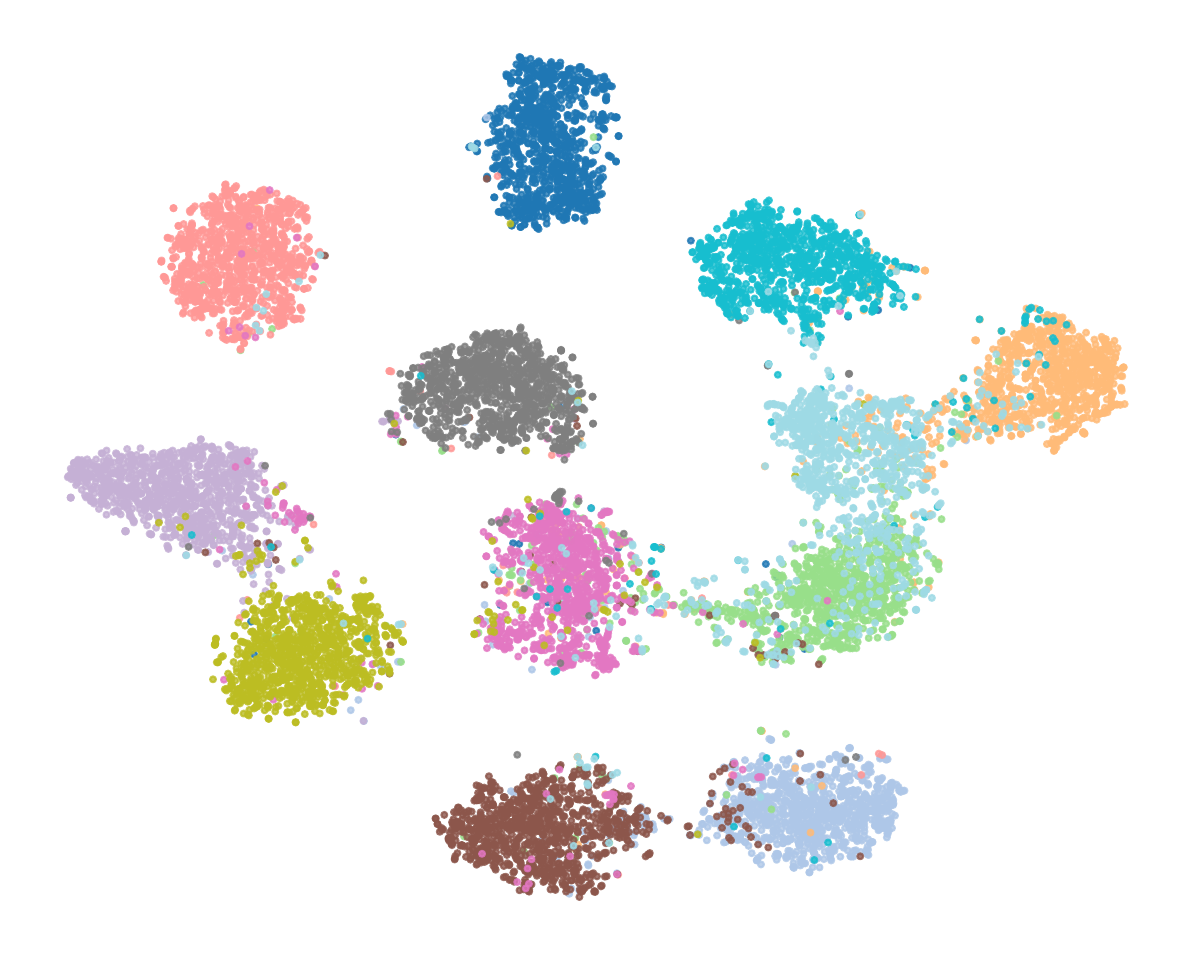}
            \caption{SHOT}
        \end{subfigure}
        \begin{subfigure}[b]{0.40\columnwidth}
            \centering
            \includegraphics[width=\textwidth]{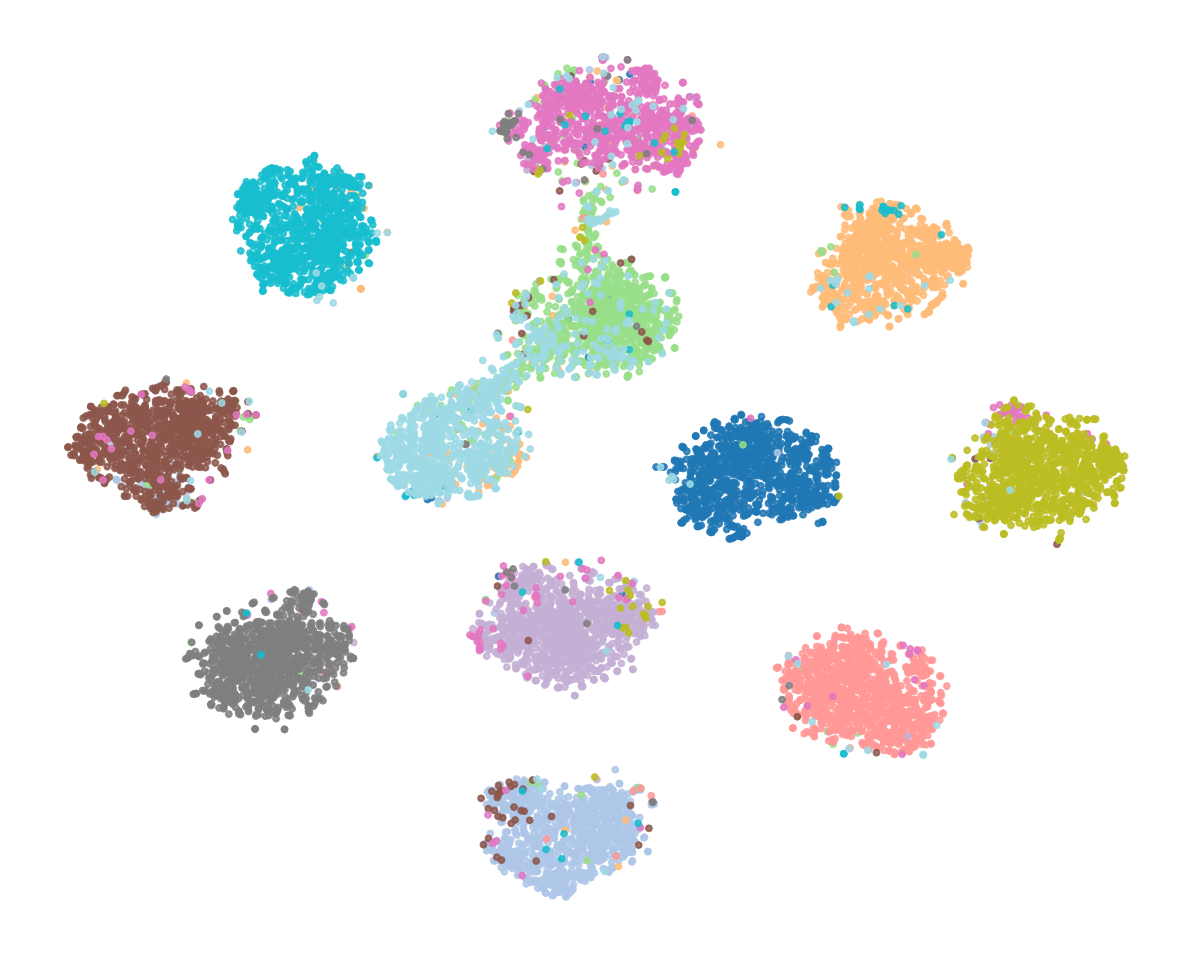}
            \caption{SSA}
        \end{subfigure}
        %\vskip -0.1in
        \vspace{-2pt}
        \caption{SSA t-SNE visualization on VisDA-C.}
        \label{fig:tsne_ssa}
        \vskip -0.2in
    \end{figure}

    Additionally, we also conducted t-SNE visualization on image classification tasks. As shown in \autoref{fig:tsne_ssa}, compared to SHOT, SSA generates the more distinct and well-separated clusters, providing strong empirical support for SSA's ability to handle challenging domain shifts better than conventional one-step alignment methods. More results provided in \autoref{sec:appendix_visual_res_ana}.

\subsection{Effectiveness of SSA in More Realistic Scenario}
    To test SSA's adaptive performance in realistic scenarios, we tested SSA in online scenarios. As shown in \autoref{tab:online_adaptation}, the method achieves an average accuracy of 76.62\% for Office-Home and 82.58\% for VisDA-C across domain transfer tasks. These results are noteworthy given the challenge of online adaptation, where the model adapts incrementally without the complete target dataset.
    
    \begin{table}[H]
	\centering
    \vskip -0.1in
	\caption{Online scenarios results of SSA in on Image Classification.}
    % \vskip -0.1in
    \vspace{-2pt}
	\label{tab:online_adaptation}
	%\small
    \resizebox{0.35\textwidth}{!}{
	\begin{tabular}{lccc}
		\toprule
		Method & Office-Home & VisDA-C & Avg.  \\ 
		\midrule
		TENT & 67.27 & 75.43 & 71.35 \\ 
		SHOT & 67.89 & 76.69 & 72.29 \\
        \rowcolor{gray!30}
		\textbf{SSA} & \textbf{\textcolor{red}{76.62}} & \textbf{\textcolor{red}{82.58}} & \textbf{\textcolor{red}{79.60}} \\
		\bottomrule
	\end{tabular}
    }
    \vskip -0.1in
\end{table}
    \begin{table}[H]
	\centering
    %\vskip -0.1in
    % \vspace{-5pt}
    \caption{Multi-source/target results of SSA on Office-Home.}
    \label{msda_mtda}
    \scriptsize
    \centering
    \begin{tabular}{c c c c c c c}
        \toprule
        Task & Method & Ar  & Cl & Pr & Rw  & Avg. \\
        \midrule
        \multirow{3}{*}{\textbf{MTDA}} 
        & CoNMix & 75.6 & 81.4 & 71.4 & 73.4 & 75.4\\
        & ProDe & 83.3 & 89.2 & 80.9 & 81.2 & 83.6\\
        & \cellcolor{gray!30}\textbf{SSA} &
        \cellcolor{gray!30}\color{red}\textbf{84.1} &
        \cellcolor{gray!30}\color{red}\textbf{89.4} &
        \cellcolor{gray!30}\color{red}\textbf{81.6} &
        \cellcolor{gray!30}\color{red}\textbf{81.6} &
        \cellcolor{gray!30}\color{red}\textbf{84.2} \\
        \midrule
        \multirow{4}{*}{\textbf{MSDA}} 
        & SHOT-Ens & 82.9 & 82.8 & 59.3 & 72.2 & 74.3\\
        & DECISION & 83.6 & 84.4 & 59.4 & 74.5 & 75.5\\
        & ProDe & 91.1 & 92.5 & 73.4 & 83.0 & 85.0\\
        & \cellcolor{gray!30}\textbf{SSA} &
        \cellcolor{gray!30}\color{red}\textbf{91.7} &
        \cellcolor{gray!30}\color{red}\textbf{93.1} &
        \cellcolor{gray!30}\color{red}\textbf{75.2} &
        \cellcolor{gray!30}\color{red}\textbf{83.6} &
        \cellcolor{gray!30}\color{red}\textbf{85.9} \\
        \bottomrule
    \end{tabular}
    % \vskip -0.1in
\end{table}

    Furthermore, we conducted tests on TTA for multi-source domains (MSDA) and multi-target domains (MTDA). As shown in \autoref{msda_mtda}, SSA demonstrates excellent performance in both multi-source TTA scenarios, which are susceptible to source domain interference, and multi-target domain scenarios, which are prone to multiple distribution shifts.

\subsection{Rationality of Using Pre-Trained Models}
    To explore the core role of the source model's pre-trained model in the SSA process, we designed additional comparative experiments to determine whether it introduces extra domain prior information. Specifically, we compared the performance of a pre-trained model that was not trained on the source domain with a trained source model on the same domain adaptation tasks. Note that since the classifier of the pre-trained model does not match the category dimensions of the target domain, we replaced it with the classifier of the source model while retaining its feature extractor.

    \begin{table}[h]
	\centering
    % \vskip -0.1in
	\caption{Source and pre-trained models results on Office-Home.}
	\label{pretrain_model_ana}
	%\scriptsize
    \resizebox{0.42\textwidth}{!}{
	\begin{tabular}{l c c c c c c}
		\toprule
        \multirow{2}{*}{Model} & \multirow{2}{*}{$\mathcal{D}_s$}
		& \multicolumn{4}{c}{$\mathcal{D}_t$}
		& \multirow{2}{*}{Avg.} \\
		\cmidrule(lr){3-6}
		& & Ar & Cl & Pr & Rw & \\
        \midrule
        \multirow{4}{*}{Pre-Trained}
		& Ar   & 90.32 & 33.17 & 63.62 & 71.72 & 64.71 \\
		& Cl  & 44.26	& 87.39	& 59.04	& 54.52	& 61.30 \\
		& Pr  & 44.68	& 32.46	& 91.28	& 71.38	& 59.95 \\
		& Rw & 54.28	& 35.60	& 71.57	& 91.05	& 63.13 \\
        \midrule
        \multirow{4}{*}{Source}
        & Ar   & 97.50  & 44.80 & 65.88 & 72.11 & \color{red}\textbf{70.07} \\
		& Cl  & 46.86 & 97.06& 59.69 & 61.07 & \color{red}\textbf{66.17} \\
		& Pr  & 49.91 & 40.20	& 99.27	& 71.55	& \color{red}\textbf{65.23} \\
		& Rw & 62.43 & 46.44 & 76.80 & 98.02 & \color{red}\textbf{70.92}\\
		\bottomrule
	\end{tabular}
    }
\end{table}

    As shown in \autoref{pretrain_model_ana}, in all 12 adaptation tasks on the Office-Home dataset, the source model significantly outperformed the pre-trained model. This indicates that \textit{\textbf{the pre-trained model does not provide extra domain priors beyond the source model}}, its role in SSA is more about providing general semantics to prevent the reference model from severely deviating from the original semantics during updates (the semantic categories of the source and target domains are the same, only features such as style differ).

\subsection{SSA's Scaling Effect with Semantic-Intensive}
\label{sec:scaling_ana}
    %During the experiments, we found that the performance improvement of SSA varies across different tasks. This difference is specifically manifested as becoming more significant with the increase in the semantic density of the task (classifier dimensions and the number of samples). It is intuitively shown that the improvement in semantic segmentation tasks is greater than in image classification tasks. To further verify this observation, we additionally supplemented the experiments with a multi-label image classification task, which has a semantic density between semantic segmentation and single-label image classification. As shown in \autoref{tab:multilabel_performance} and \autoref{fig:ssa_scaling}, the performance improvement is higher than that of single-label image classification but lower than that of semantic segmentation.

    During the experiments, we found that SSA's performance improvement varies across tasks, becoming more significant with increased semantic density (\textit{classifier} dimensions and sample numbers). Intuitively, the improvement in semantic segmentation tasks is greater than in image classification tasks. To verify this, we added a multi-label image classification task, which has a semantic density between semantic segmentation and single-label image classification. As shown in \autoref{tab:multilabel_performance} and \autoref{fig:ssa_scaling}, the performance improvement is higher than single-label image classification but lower than semantic segmentation.

    \begin{table}[h]
	\centering
	\caption{Multi-label classification of SSA.}
	\label{tab:multilabel_performance}
	%\small
    \resizebox{0.48\textwidth}{!}{
	\begin{tabular}{lccccc}
		\toprule
		Method & COCO2014 & COCO2017 & VOC2007 & VOC2012 & mAcc \\ 
		\midrule
		CLIP   & 47.53 & 47.32 & 75.91 & 74.25 & 61.25 \\ 
		TPT    & 48.52 & 48.51 & 75.54 & 73.92 & 61.62 \\ 
		ML-TTA & 51.58 & 51.39 & 78.62 & 76.63 & 64.56 \\
        \midrule
		\rowcolor{gray!30} 
		\textbf{SSA}  & \textbf{\textcolor{red}{53.27}} & \textbf{\textcolor{red}{52.96}} & \textbf{\textcolor{red}{79.37}} & \textbf{\textcolor{red}{77.29}} & \textbf{\textcolor{red}{65.72}} \\ 
		\bottomrule
	\end{tabular}
    }
\end{table}

    \begin{figure}[ht]
    	\centering
    	\includegraphics[width=0.95\columnwidth]{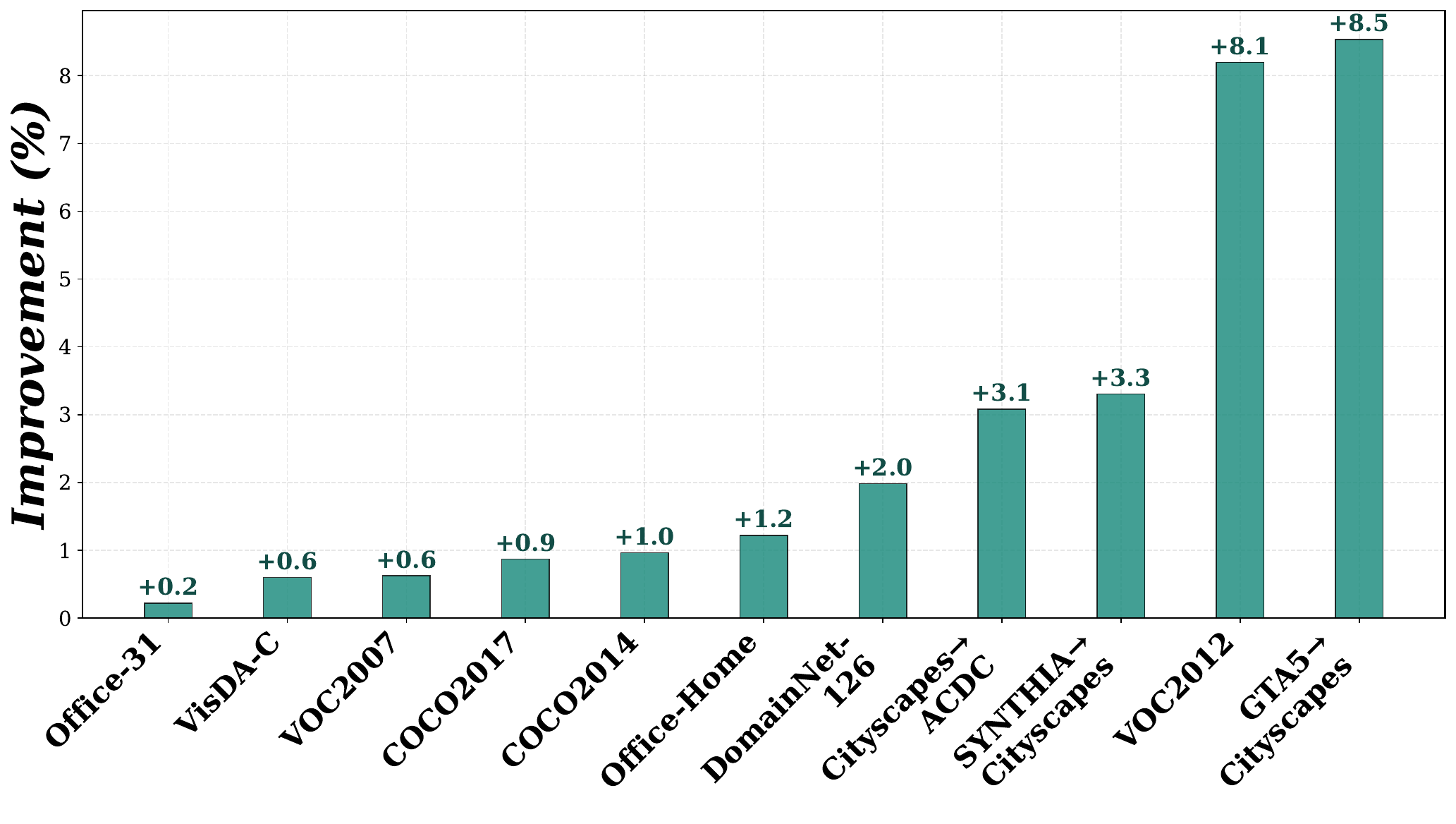}{}
    	\caption{Scaling effect of SSA's gains in different tasks.}
        % \vskip -0.1in
    	\label{fig:ssa_scaling}
    \end{figure}

\section{Conclusion}

In this paper, we propose SSA, a \textit{stepwise semantic alignment} framework for test-time adaptation that focuses on high-level semantic alignment. By combining \textit{hierarchical semantic aggregation} (HFA) with \textit{confidence-aware complementary learning} (CACL), SSA enhances semantic consistency and robustness under domain shifts. The stepwise process uses universal semantic priors to progressively bridge source, pseudo-source, and target domains, effectively narrowing domain gaps. The approach’s effectiveness is validated through extensive experiments on synthetic and real-world TTA datasets.

\paragraph{Limitations and Future Work.}  %Despite strong results, SSA has limitations. The fixed local-to-global context ratio may restrict adaptability across tasks with varying semantic granularity and is less effective for single-label classification where spatial context is less informative. Future work could explore adaptive aggregation or extend SSA to more complex multi-label scenarios. Additionally, the fixed self-entropy threshold $\tau_{\text{par}}$ for pseudo-source domain construction may not generalize well across domains, highlighting the need for more flexible, data-driven thresholding. 
Although SSA significantly improves performance in semantic segmentation and single/multi-label image classification, its impact is less pronounced in tasks with sparse semantic information, such as those with fewer categories or examples. Despite reduced annotation needs for similar gains in semantic-intensive tasks, semantic-sparse scenarios remain common. Future work should explore extending stepwise approaches in these tasks to meet SSA's test-time adaptation needs.

\section*{Impact Statement}
    This is propose a stepwise semantic alignment method that enhances the model's adaptive capability in test-time, whose goal is to advance the field of Machine Learning. There are many potential societal consequences of our work, none which we feel must be specifically highlighted here.

\bibliography{shlsa}
\bibliographystyle{icml2026}

%%%%%%%%%%%%%%%%%%%%%%%%%%%%%%%%%%%%%%%%%%%%%%%%%%%%%%%%%%%%%%%%%%%%%%%%%%%%%%%
%%%%%%%%%%%%%%%%%%%%%%%%%%%%%%%%%%%%%%%%%%%%%%%%%%%%%%%%%%%%%%%%%%%%%%%%%%%%%%%
% APPENDIX
%%%%%%%%%%%%%%%%%%%%%%%%%%%%%%%%%%%%%%%%%%%%%%%%%%%%%%%%%%%%%%%%%%%%%%%%%%%%%%%
%%%%%%%%%%%%%%%%%%%%%%%%%%%%%%%%%%%%%%%%%%%%%%%%%%%%%%%%%%%%%%%%%%%%%%%%%%%%%%%

\newpage
\appendix
\onecolumn

\section{Pseudo-code of the proposed method}
\label{sec:appendix_pseudo_code}

\begin{algorithm}[htb]
\caption{Stepwise Semantic Alignment (SSA)}
\label{alg:shlsa}
\begin{algorithmic}
\STATE {\bfseries Input:} Target dataset $D_t$, source model $\mathcal{M}_s = (\mathbf{f_s}, \mathbf{g_s})$
\STATE {\bfseries Output:} Adapted model $\mathcal{M}_t$
\STATE
\STATE {\color{gray} \textit{// Stage1: Update HFA module}}
\WHILE{not converged}
    \FORALL{$\boldsymbol{x} \in D_t$}
        \STATE {\color{gray} \textit{// Hierarchical Feature Aggregation \& Entropy Estimation}}
        \STATE Extract hierarchical semantic features $\{\mathbf{f}^{(l)}(\boldsymbol{x}),\mathbf{f}^{(g)}(\boldsymbol{x})\}$ from $\mathbf{f_s}$
        \STATE Fuse local-global features via attention: $f_{\text{hfa}}(x) \gets \text{AttnFusion}(\{\mathbf{f}^{(l)}(\boldsymbol{x}),\mathbf{f}^{(g)}(\boldsymbol{x})\})$
        \STATE Generate pseudo-probabilities $p(\boldsymbol{x}) \gets \text{softmax}(\mathbf{g_s}(f_{\text{hfa}}(\boldsymbol{x})))$
        \STATE Update entropy bank: $\mathcal{H}(\boldsymbol{x}) \gets \alpha \cdot \mathcal{H}(\boldsymbol{x}) + (1-\alpha) \cdot H(p(\boldsymbol{x}))$
        \STATE {\color{gray} \textit{// Confidence-Aware Complementary Learning}}
        \STATE Build confidence mask $m$ from ranked logits (\autoref{eq_mask})
        \STATE Compute $\mathcal{L}_{\text{cacl}}$ from confident and uncertain regions (\autoref{eq_lcl})
        \STATE Update model using $\mathcal{L}_{\text{cacl}}$
    \ENDFOR
\ENDWHILE
\STATE
\STATE {\color{gray} \textit{// Stage2: Stepwise Semantic Alignment}}
\STATE Partition $D_t$ into $D_{ps}$ and $D_{rt}$ using entropy threshold $\tau_{\text{par}}$
\WHILE{not converged}
    \FORALL{$(\boldsymbol{x}_{ps}, \boldsymbol{x}_{rt}) \in \mathcal{D}_{ps} \times \mathcal{D}_{rt}$}
        \STATE {\color{gray} \textit{// STEP1: Pseudo-Source Semantic Correction}}
        \STATE Extract features $f_{\text{ps}}(\boldsymbol{x}^{ps})$ via HFA; extract $f^{\text{pre}}(\boldsymbol{x}_{ps})$ via frozen pretrained net
        \STATE Compute semantic alignment loss $\mathcal{L}_{\text{dis}}$ (\autoref{eq_ldis})
        \STATE {\color{gray} \textit{// STEP2: Remaining-Target Semantic Alignment}}
        \STATE Generate class-masked mixed sample $(\tilde{\boldsymbol{x}}, \tilde{y})$ (\autoref{eq_mixup})
        \STATE Compute cross-entropy loss $\mathcal{L}_{\text{mix}}$ on mixed sample using CACL (\autoref{eq_lce})
        \STATE Update model using $\mathcal{L}_{\text{dis}} + \mathcal{L}_{\text{mix}}$
    \ENDFOR
\ENDWHILE
\STATE
\STATE \textbf{return} $\mathcal{M}_t$
\end{algorithmic}
\end{algorithm}

\section{Theorem Proof}
\label{sec:theorem_proof}

\subsection{Proof of \autoref{theorem}}
\begin{proof}
	Let $\boldsymbol{p} = (p_1, \dots, p_C) \in \Delta^{C-1}$ be a categorical distribution with entropy $\mathcal{H}(\boldsymbol{p}) = -\sum_{c=1}^C p_c \log p_c \le H_0$. Fix any $\alpha \in (0,1)$. We aim to construct thresholds $\tau_\alpha > \tau_\beta$ such that the sets
	\begin{equation}
		\mathcal{Y}_+ := \{c \mid p_c \ge \tau_\alpha\}, \quad
		\mathcal{Y}_- := \{c \mid p_c \le \tau_\beta\},
	\end{equation}
	satisfy the stated bounds.
	
	We begin by bounding the low-confidence tail. For any $\tau \in (0,1)$, define $\mathcal{Y}_{<\tau} := \{c \mid p_c \le \tau\}$. Then
	\begin{equation}
		\mathcal{H}(\boldsymbol{p}) \ge -\sum_{c \in \mathcal{Y}_{<\tau}} p_c \log p_c \ge -\log \tau \sum_{c \in \mathcal{Y}_{<\tau}} p_c,
	\end{equation}
	which implies
	\begin{equation}
		\sum_{c \in \mathcal{Y}_{<\tau}} p_c \le \frac{H_0}{-\log \tau}. \label{eq:tail-bound}
	\end{equation}
	Setting $\tau_\beta := \tau$, this yields the desired upper bound on $\sum_{c \in \mathcal{Y}_-} p_c$ with $\epsilon(H_0) := H_0 / -\log \tau_\beta \to 0$ as $H_0 \to 0$.
	
	Next, for the high-confidence region, choose $\tau_\alpha \in (0,1)$ such that $\mathcal{Y}_+$ is nonempty. Then, for all $c \in \mathcal{Y}_+$, $p_c \ge \tau_\alpha$ implies $\log p_c \ge \log \tau_\alpha$, and for $c \in \mathcal{Y}_-$, $p_c \le \tau_\beta$ implies $1 - p_c \ge 1 - \tau_\beta$, so $\log(1 - p_c) \ge \log(1 - \tau_\beta)$. Thus,
	\begin{align}
		\mathbb{E}_{c \sim \boldsymbol{p}}[\log p_c \mid c \in \mathcal{Y}_+] 
		&\ge \log \tau_\alpha, \\
		\mathbb{E}_{c \sim \boldsymbol{p}}[\log(1 - p_c) \mid c \in \mathcal{Y}_-] 
		&\ge \log(1 - \tau_\beta),
	\end{align}
	and therefore
	\begin{equation}
		\mathbb{E}_{c \sim \boldsymbol{p}}[\log p_c \mid c \in \mathcal{Y}_+] - 
		\mathbb{E}_{c \sim \boldsymbol{p}}[\log(1 - p_c) \mid c \in \mathcal{Y}_-] 
		\ge \log \tau_\alpha - \log(1 - \tau_\beta) := \kappa(H_0, \tau_\alpha). \label{eq:kappa-bound}
	\end{equation}
	This concludes the proof.
\end{proof}

\subsection{Proof of independence for parameter tuning}
\label{paramstuning_proof}
% \subsection{Proof of independent tunability}
% \label{paramstuning_proof}
\begin{proof}
	Let $\boldsymbol{p} = (p_1, \dots, p_C) \in \Delta^{C-1}$ be the predicted softmax vector for a sample with entropy $\mathcal{H}(\boldsymbol{p}) = -\sum_{c=1}^C p_c \log p_c$. Let $\boldsymbol{p}_{\text{sorted}} = [p_{(1)}, p_{(2)}, \dots, p_{(C)}]$ denote the sorted version of $\boldsymbol{p}$ as same as \autoref{eq_psort}.
	
	The threshold $\tau_{\text{pos}}$ is used to identify confident predictions:
	\begin{equation}
		\mathcal{Y}_+ := \{c \mid p_c \ge \tau_{\text{pos}}\}.
	\end{equation}
	The choice of $\tau_{\text{pos}}$ depends only on the absolute values of $\boldsymbol{p}$ and does not affect its ordering or entropy.
	
	The threshold $\tau_{\text{neg}}$ is based on the relative drop between adjacent sorted entries:
	\begin{equation}
		r_i := \frac{p_{(i)} - p_{(i+1)}}{p_{(i)}}, \quad
		i^* := \min \{ i \mid r_i \ge \tau_{\text{neg}} \},
	\end{equation}
	from which we define the low-confidence region:
	\begin{equation}
		\mathcal{Y}_- := \{c \mid c > i^*\}.
	\end{equation}
	Note that $\tau_{\text{neg}}$ depends only on the relative differences in the ordered vector $\boldsymbol{p}_{\text{sorted}}$ and is independent of any fixed threshold such as $\tau_{\text{pos}}$.
	
	Now consider the partition threshold $\tau_{\text{par}}$, which operates at the sample level. Let $\mathcal{X}_{\text{ps}} := \{x \mid \mathcal{H}(\boldsymbol{p}_x) \le \tau_{\text{par}} \}$ denote the subset of target samples with low entropy. Since $\mathcal{H}(\boldsymbol{p})$ is a symmetric function of the distribution and depends only on the global shape of $\boldsymbol{p}$ rather than pointwise values or ordering, the partitioning induced by $\tau_{\text{par}}$ is independent of the selection mechanisms for class-level masks.
	
	To formalize this independence, observe that the three thresholding operations in our framework correspond to distinct mappings:
	\begin{equation}
		\boldsymbol{p} \mapsto \mathcal{Y}_+, \quad
		\boldsymbol{p}_{\text{sorted}} \mapsto \mathcal{Y}_-, \quad
		\boldsymbol{p} \mapsto \mathcal{H}(\boldsymbol{p}).
	\end{equation}
	Here, $\mathcal{Y}_+$ is determined by a fixed threshold on the unnormalized softmax scores ($\tau_{\text{pos}}$), $\mathcal{Y}_-$ is based on relative gaps in the sorted probability vector using $\tau_{\text{neg}}$, and $\mathcal{H}(\boldsymbol{p})$ aggregates the entire distribution through a permutation-invariant functional. Each of these quantities responds only to changes in its respective threshold.
	
	Therefore, varying $\tau_{\text{pos}}$ affects only the set $\mathcal{Y}_+$, changing $\tau_{\text{neg}}$ influences only the structure of $\mathcal{Y}_-$, and adjusting $\tau_{\text{par}}$ modifies only the partitioning of samples into $\mathcal{X}_{\text{ps}}$ and $\mathcal{X}_{\text{rt}}$. These quantities are used in distinct components of the framework. Specifically, $\mathcal{Y}_+$ and $\mathcal{Y}_-$ are involved in confidence-aware complementary learning, while $\mathcal{X}_{\text{ps}}$ is used in the progressive domain alignment module. As a result, the thresholds $\tau_{\text{pos}}$, $\tau_{\text{neg}}$, and $\tau_{\text{par}}$ are functionally independent and can be tuned separately without mutual interference. 
\end{proof}

\newpage
\section{Experiment Description}
\label{sec:appendix_exp_describ}

\subsection{Baseline and Competitors}
\label{sec:competitors}
    We adopt SHOT~\cite{SHOT} as the baseline for semantic segmentation, which utilizes source hypothesis transfer for source-free domain adaptation with CrossEntropyLoss, implementing model consistency regularization to prevent excessive deviation between target and source models. For single-label classification, we select ProDe~\cite{ProDe} as the baseline. We conducted comprehensive experiments across different tasks, datasets, and architectures, encompassing both source-available and source-free domain adaptation scenarios.
    
    \textbf{Semantic Segmentation Methods.} For source-available adaptation, we compared against TransDA-B~\cite{TransDA-B}, DAFormer~\cite{DAFormer}, HRDA~\cite{HRDA}, and IDM~\cite{IDM}. For source-free adaptation, we evaluated against DAFormer~\cite{DAFormer}, HRDA~\cite{HRDA}, IDM~\cite{IDM}, ATP~\cite{ATP}, MISFIT~\cite{MISFIT}, TENT~\cite{TENT}, CoTTA~\cite{CoTTA}, DePT~\cite{DePT}, VDP~\cite{VDP}, SFKT~\cite{SFKT}, and SFDA-Seg~\cite{SFDA-Seg}.
    
    \textbf{Classification Methods.} In single-label image classification, we compared with source-available methods including PDA~\cite{PDA}, DAMP~\cite{DAMP}, STAR~\cite{STAR}, and RWOT~\cite{RWOT}, and source-free methods including SHOT~\cite{SHOT}, DIFO~\cite{DIFO}, ProDe~\cite{ProDe}, ATP~\cite{ATP}, PLUE~\cite{PLUE}, and SFDA+~\cite{SFDA+}.
    
    \textbf{Network Architectures.} Our experiments span multiple architectures of varying scales, including Segformer-B0/B1/B2/B3/B4/B5, P2T-Base, ResNet-50/101, and VGG16, demonstrating the generalizability of our approach across different network designs.
    
    \textbf{Semantic Segmentation Datasets.} We evaluated on three challenging adaptation scenarios: GTA5$\to$Cityscapes, SYNTHIA$\to$Cityscapes, and Cityscapes$\to$ACDC. The synthetic dataset GTA5~\cite{GTA5} contains 24,966 annotated images with a resolution of 1914×1052, taken from the famous game Grand Theft Auto with ground truth generated by game rendering. SYNTHIA~\cite{SYNTHIA} is another synthetic dataset containing 9,400 fully annotated images with a resolution of 1280×760. Cityscapes~\cite{Cityscapes} consists of 2,975 annotated training images and 500 validation images with a resolution of 2048×1024. The Adverse Conditions Dataset (ACDC)~\cite{ACDC} contains four different adverse visual conditions: Fog, Night, Rain, and Snow, sharing the same semantic classes with Cityscapes. These datasets enable evaluation of source-free adaptation from synthetic to real domains and from normal to adverse conditions.
    
    \textbf{Classification Datasets.} For image classification, we selected four datasets of varying scales: Office-31~\cite{Office-31} is a small-scale dataset including three domains (Amazon, Webcam, and Dslr) with 4,652 images of 31 categories taken in office environments. Office-Home~\cite{Office-Home} is a medium-scale dataset containing 15k images belonging to 65 categories from four domains: Artistic images, Clip Art, Product images, and Real-world images. VisDA-C~\cite{VisDA-C} is a large-scale dataset with synthetic to real transfer tasks, containing 152k synthetic source images and 55k real target images from COCO. DomainNet~\cite{DomainNet} is a challenging large-scale dataset created by removing noisy labels from the original version, containing 600k images of 345 classes from 6 domains with varying image styles.

\subsection{Implementation Details}
\label{sec:exp_details}
    All experiments were conducted on a single NVIDIA GeForce RTX 3090. Some necessary experimental parameters have already been provided in \autoref{imp_details}, and here we mainly supplement the implementation details of the Method in \autoref{method}.
    
    \paragraph{SSA Implementation Details.} For the $\text{Grid}(\boldsymbol{x})$ operation in \autoref{eq_flocal}, we divide the input along the $x$ and $y$ directions with a fixed step to obtain the grid $\{y_{i1}, y_{i2}, x_{i1}, x_{i2}\}$, where each $r_i$ corresponds to a patch in $\boldsymbol{x}$ defined by four coordinates. For the \text{Pad}($P_i$) operation in \autoref{eq_local}, we apply zero padding. For the \text{Mask}$_i$ operation in the same equation, we compute it by counting the number of pixels covered by the $i$-th local patch. For the $\boldsymbol{A}$ operation in \autoref{eq_fused}, we obtain it through an attention module. In semantic segmentation tasks, this module outputs a pixel-level attention map to represent the importance of local semantics, while in image classification tasks, it directly averages local and global semantics. The $\text{Align}(\cdot)$ operation in the same equation is implemented using a simple bilinear interpolation-based \textit{resize} to align the sizes of local and global features. The resulting $P_{\text{fused}}$ is directly used as the input to CACL for post-processing. For $\boldsymbol{M}$ in \autoref{eq_mixup}, we adopt the MixMatch operation to generate semantic-level mixed data, following the procedure described in~\cite{mixmatch}, which we keep consistent with.

    \paragraph{Online TTA Setting.} The online test-time adaptation setting differs significantly from traditional offline scenarios. Specifically, the batch size was adjusted to appropriate mini-batches (e.g., from 64 to 128), and the testing process was modified to ensure streaming data processing. In a single epoch, the model outputs results immediately after training on each mini-batch, ensuring that data flows in a streaming manner where each sample participates in gradient updates only once. This contrasts with offline scenarios where each sample participates in multiple gradient updates and testing occurs only after complete model training.

\section{Parameter Sensitivity Analysis and Mitigation}
\label{sec:parameter_ana}

\subsection{Sensitivity of Key Parameters}
We conduct comprehensive sensitivity analysis on several key parameters in SSA, including the positive threshold $\tau_{\text{pos}}$, the negative threshold $\tau_{\text{neg}}$, the split ratio $\tau_{\text{par}}$, and the local-global hyperparameters $\tau_{\text{lg}}$ (ratio) and $\tau_{\text{lw}}$ (local window size).

\begin{figure}[H]
	%\vspace{-5pt}
	\centering
	\small
	\begin{subfigure}[b]{0.32\textwidth}
		\vspace{-5pt}
		\centering
		\includegraphics[width=\textwidth]{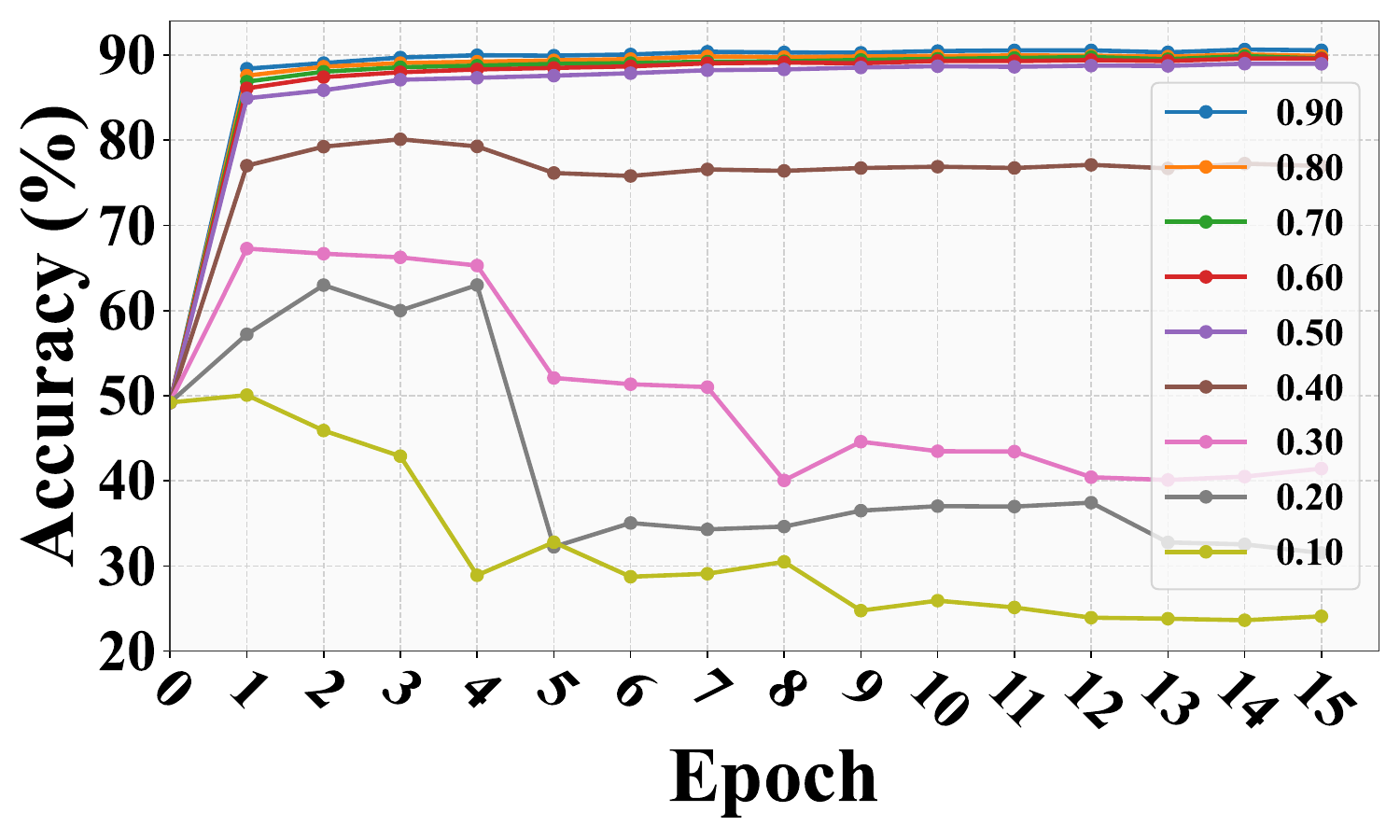}
		\caption{$\tau_{\text{neg}}$ sensitivity}
		\label{neg_threshold_plot}
	\end{subfigure}
	\begin{subfigure}[b]{0.32\textwidth}
		\centering
		\includegraphics[width=\textwidth]{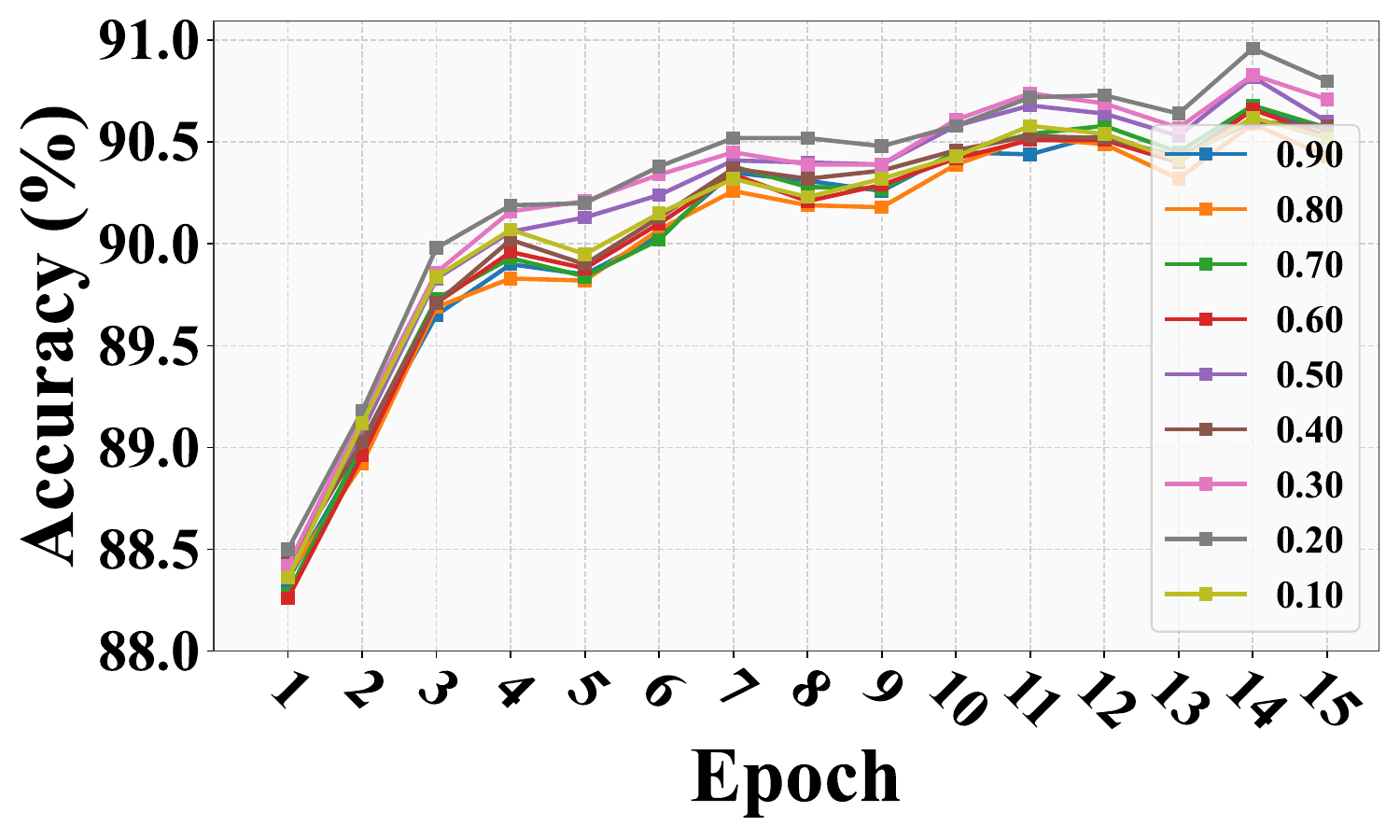}
		\caption{$\tau_{\text{pos}}$ sensitivity}
		\label{pos_threshold_plot}
	\end{subfigure}
	\begin{subfigure}[b]{0.32\textwidth}
		\centering
		\includegraphics[width=\textwidth]{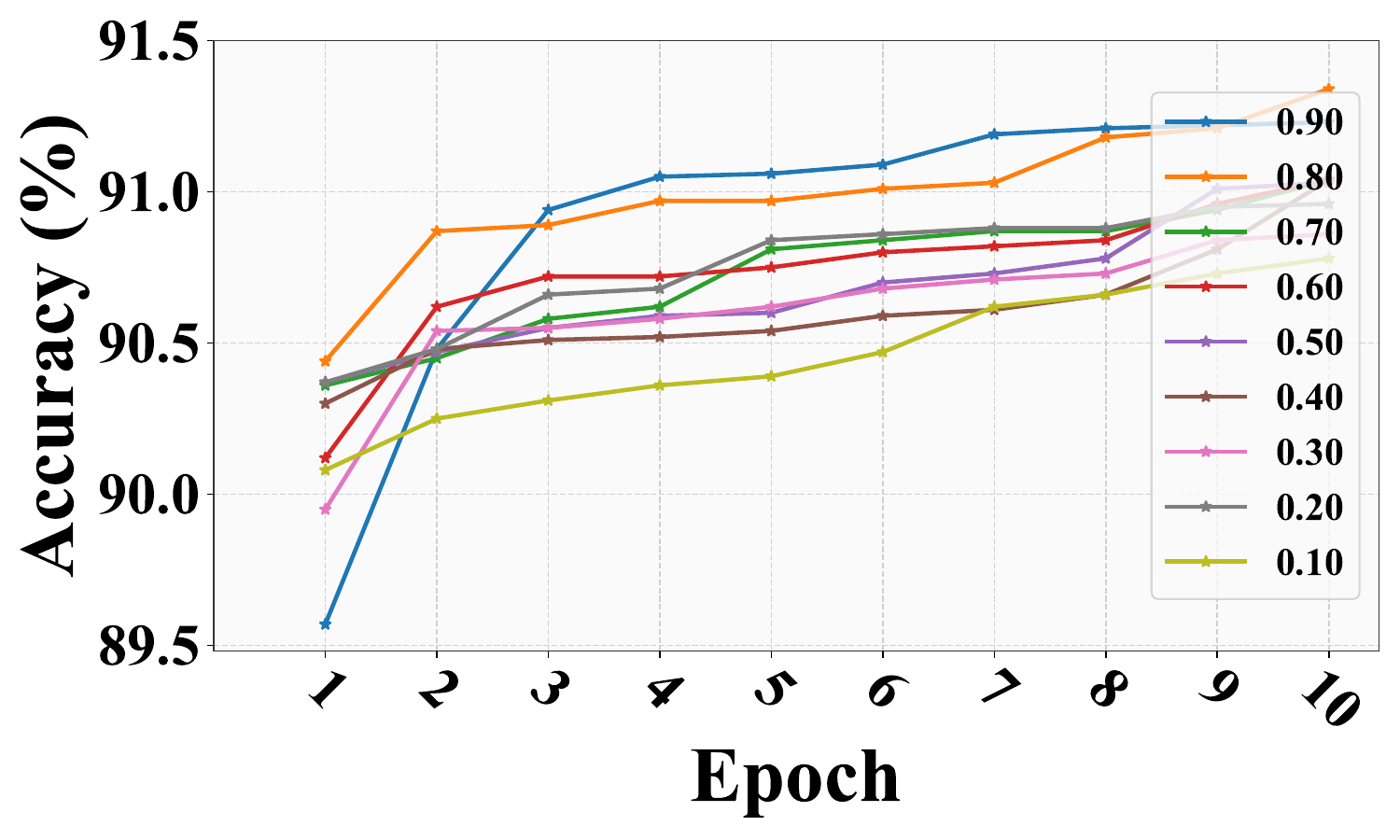}
		\caption{$\tau_{\text{par}}$ sensitivity}
		\label{par_threshold_plot}
	\end{subfigure}
	
	\caption{Sensitivity analysis for the thresholds.}
	\label{sensitivity_plot}
	\vspace{-12pt}
\end{figure}

\textbf{Positive threshold $\tau_{\text{pos}}$} serves as a filter for probability distributions, treating labels with probabilities above $\tau_{\text{pos}}$ as positive while ignoring others. As demonstrated in \autoref{sensitivity_plot}, $\tau_{\text{pos}}$ exhibits relatively low sensitivity, primarily because it only influences the inclusion or exclusion of a limited number of positive labels without substantially affecting the overall learning dynamics.

\textbf{Negative threshold $\tau_{\text{neg}}$} emerges as the most sensitive parameter, critically controlling the distinction between \textit{likely correct} and \textit{unlikely correct} predictions. Within the CACL framework, predictions undergo descending probability sorting, with entries exceeding the gradient threshold $\tau_{\text{neg}}$ being designated as negative pseudo-labels. This heightened sensitivity stems from several interconnected factors: (1) $\tau_{\text{neg}}$ directly governs negative pseudo-label quality within CACL; (2) During early training phases characterized by high model uncertainty and relatively flat probability distributions, minor parameter adjustments can dramatically alter negative label quantities; (3) Excessively low $\tau_{\text{neg}}$ values introduce erroneous negative labels, causing noise accumulation and performance degradation, which is the primary source of error propagation compared to $\tau_{\text{pos}}$.

\textbf{Split ratio $\tau_{\text{par}}$} regulates the proportion of target domain data allocated to the pseudo-source domain within SDA. For tasks exhibiting substantial domain shifts (e.g., VisDA-C), employing fewer low-entropy samples as pseudo-source domain through reduced $\tau_{\text{par}}$ values enhances sample quality. Conversely, for tasks with moderate shifts (e.g., Office-31), elevated $\tau_{\text{par}}$ values enable better exploitation of source-like samples within the target domain. \autoref{tab:tau_par_analysis} validates that confidence-aware dynamic $\tau_{\text{par}}$ configuration consistently outperforms fixed threshold approaches across diverse tasks. Specifically, we followed the idea of CACL and selected the points with large variations in self-entropy of the samples in the target domain as the partition points.

\begin{table}[h]
	\centering
	\caption{Impact of Different $\tau_{\text{par}}$ Parameters on Single-label Classification}
	\label{tab:tau_par_analysis}
	\scriptsize
	\begin{tabular}{lccccc}
		\toprule
		Method & $\bar{\tau_{\text{par}}}$ & Office-31 (0.1335) & Office-Home (0.2491) & VisDA-C (0.1677) & DomainNet (0.3478) \\
		\midrule
		\multirow{4}{*}{Conf-Aware} & 0.92 & 92.8 & - & - & - \\
		 & 0.87 & - & 85.0 & - & - \\
		 & 0.79 & - & - & 91.3 & - \\
		 & 0.44 & - & - & - & 83.7 \\
		 \cmidrule{1-6}
		\multirow{9}{*}{Fixed} & 0.10 & 92.5 & 84.7 & 90.9 & 83.4 \\
		 & 0.20 & 92.5 & 84.6 & 91.2 & 83.5 \\
		 & 0.30 & 92.6 & 84.7 & 91.2 & 83.5 \\
		 & 0.40 & 92.5 & 84.7 & 91.1 & 83.7 \\
		 & 0.50 & 92.7 & 84.8 & 91.3 & 83.7 \\
		 & 0.60 & 92.6 & 84.9 & 91.1 & 83.5 \\
		 & 0.70 & 92.8 & 85.0 & 91.3 & 83.3 \\
		 & 0.80 & 92.8 & 85.0 & 91.3 & 83.1 \\
		 & 0.90 & 92.8 & 84.9 & 91.2 & 82.9 \\
		\bottomrule
	\end{tabular}
\end{table}

\textbf{Local--global ratio $\tau_{\text{lg}}$ and local window size $\tau_{\text{lw}}$} demonstrate distinct sensitivity patterns across different tasks and window configurations. As shown in \autoref{tab:local_global_sensitivity}, segmentation tasks (GTA5$\to$Cityscapes) exhibit significantly higher parameter sensitivity compared to classification tasks (VisDA-C). For segmentation tasks, sensitivity varies considerably with window size: large windows ($1024\times1024$) show minimal sensitivity to $\tau_{\text{lg}}$ with performance variations of only 0.4\% (68.8-69.2 mIoU), while smaller windows demonstrate progressively increased sensitivity, where512×512 windows show 0.7\% variation (68.4-69.1 mIoU) and 256×256 windows exhibit the highest sensitivity with 1.3\% variation (67.9-68.8 mIoU). This pattern indicates that smaller windows lack sufficient semantic context and become more dependent on the precise balance between local and global information. In contrast, classification tasks (VisDA-C) demonstrate remarkably stable performance across all parameter combinations, with minimal variations of 0.2\% for 256×256 windows (91.1-91.3 mAcc) and 0.8\% for 128×128 windows (90.6-91.4 mAcc). This low sensitivity reflects that classification tasks with sparse semantic structures are less dependent on fine-grained local-global feature alignment and primarily rely on global semantic representations.

\begin{table}[h]
	\centering
	\caption{Sensitivity analysis of $\tau_{\text{lg}}$ with varying $\tau_{\text{lw}}$ on different datasets.}
	\label{tab:local_global_sensitivity}
	\small
	\begin{tabular}{cccccc}
		\toprule
		\multirow{3}{*}{$\tau_{\text{lg}} /\ \tau_{\text{lw}}$} & \multicolumn{3}{c}{GTA5 $\rightarrow$ Cityscapes} & \multicolumn{2}{c}{VisDA-C} \\
		\cmidrule(lr){2-4} \cmidrule(lr){5-6}
		& \textbf{1024×1024} & \textbf{512×512} & \textbf{256×256} & \textbf{256×256} & \textbf{128×128} \\
		\midrule
		0.1 & 68.8 & 68.9 & 68.1 & 91.1 & 91.3 \\
		0.2 & 68.9 & 68.8 & 68.1 & 91.1 & 91.4 \\
		0.3 & 69.1 & 68.9 & 68.2 & 91.2 & 90.9 \\
		0.4 & 69.0 & 68.7 & 68.5 & 91.1 & 91.0 \\
		0.5 & 69.2 & 69.1 & 68.7 & 91.3 & 91.1 \\
		0.6 & 68.8 & 68.9 & 68.8 & 91.2 & 90.9 \\
		0.7 & 68.8 & 68.6 & 68.4 & 91.3 & 90.7 \\
		0.8 & 68.8 & 68.5 & 68.3 & 91.2 & 90.8 \\
		0.9 & 68.8 & 68.4 & 67.9 & 91.2 & 90.6 \\
		\bottomrule
	\end{tabular}
\end{table}

\subsection{Mitigation for Highly Sensitive Parameters}
For parameters with low sensitivity ($\tau_{\text{pos}}$, $\tau_{\text{par}}$, $\tau_{\text{lg}}$, $\tau_{\text{lw}}$), fixed values can be set based on empirical results. For the highly sensitive parameter $\tau_{\text{neg}}$, we propose adaptive strategies to mitigate error accumulation: \textbf{Conservative initialization} sets $\tau_{\text{neg}}$ to a high value initially to prevent large fluctuations; \textbf{Early-stage restriction} constrains $\tau_{\text{neg}}$ to a narrow range in early epochs; \textbf{Dynamic adjustment} adapts $\tau_{\text{neg}}$ based on gradient stability; \textbf{Uncertainty-aware tuning} uses entropy metrics to balance precision and recall.

We evaluate three implementations: \textbf{Strategy 1} fixes $\tau_{\text{neg}}=0.9$ for the first 5 epochs; \textbf{Strategy 2} keeps $\tau_{\text{neg}}\in[0.6,0.9]$ initially then expands to $[0.4,0.9]$; \textbf{Strategy 3} dynamically adjusts $\tau_{\text{neg}}$ every two epochs based on gradient variance. \autoref{tab:parameter_tau_neg} shows that these dynamic strategies effectively mitigate sensitivity and improve SSA's stability and performance.

\begin{table}[h]
	\centering
	\caption{Performance comparison of mitigation strategies for $\tau_{\text{neg}}$ across epochs.}
	\label{tab:parameter_tau_neg}
	\small
	\begin{tabular}{ccccc}
		\toprule
		\textbf{Epoch} & \textbf{Baseline} & \textbf{Strategy 1} & \textbf{Strategy 2} & \textbf{Strategy 3} \\
		\midrule
		1  & 88.74 & 88.90 & 89.00 & 89.20 \\
		2  & 89.73 & 89.91 & 90.00 & 90.10 \\
		3  & 90.06 & 90.15 & 90.20 & 90.30 \\
		4  & 90.20 & 90.31 & 90.35 & 90.50 \\
		5  & 90.34 & 90.47 & 90.50 & 90.70 \\
		6  & 90.51 & 90.63 & 90.65 & 90.90 \\
		7  & 90.75 & 90.81 & 90.83 & 91.05 \\
		8  & 90.54 & 90.70 & 90.95 & 91.20 \\
		9  & 90.63 & 90.85 & 91.00 & 91.25 \\
		10 & 90.76 & 90.80 & 91.08 & 91.21 \\
		\bottomrule
	\end{tabular}
\end{table}

\subsection{Theoretical Guidance for Parameter Selection}

\paragraph{Bridging Theory and Practice.} \autoref{theorem} establishes the theoretical foundation for threshold selection by proving that when the entropy of output probability distributions is sufficiently low ($\mathcal{H}(p) \leq H_0$), there exist two thresholds $\tau_{\alpha} > \tau_{\beta}$ such that the high-confidence positive set $\mathcal{Y}_+$ and low-confidence negative set $\mathcal{Y}_-$ can be well-separated in expectation, with the probability mass of $\mathcal{Y}_-$ being strictly controlled to prevent excessive noise from negative pseudo-labels.

While \autoref{theorem} does not directly specify the experimental parameters $\tau_{\text{pos}}$ and $\tau_{\text{neg}}$, it provides theoretical guidance for their selection. The correspondence between theoretical and practical parameters is established as follows: $\tau_{\alpha} = \tau_{\text{pos}}$ represents the absolute probability threshold for positive pseudo-labels, while $\tau_{\beta}$ corresponds to the probability of the first category satisfying the gradient condition, defined as:
\begin{equation}
	\tau_{\beta} = p_{\arg\max_i \{|p^*_i - p^*_{i-1}|\} \geq \tau_{\text{neg}}},
\end{equation}

where $p^*_i$ denotes the probability of the $i$-th category after descending sorting.

\paragraph{Derivation of Threshold Bounds.} Given a target distribution $\mathbf{p} = (p_1, \ldots, p_C)$ with entropy $\mathcal{H}(\mathbf{p}) \leq H_0$, we define $G(t) = \sum_{c: p_c \leq t} p_c$ as the cumulative probability mass below threshold $t$. To ensure the total probability mass of the negative set does not exceed $\epsilon > 0$, we require $G(\tau_{\beta}) \leq \epsilon$.

From Theorem 1's inequality, we derive the upper bound:
\begin{equation}
	\tau_{\beta} \leq \exp\left(-\frac{H_0}{\epsilon}\right).
\end{equation}

This constrains the feasible range of $\tau_{\beta}$ to $(0, \exp(-H_0/\epsilon))$. Smaller tolerance $\epsilon$ results in tighter upper bounds for $\tau_{\beta}$.

For practical implementation with $H_0 = 0.5$ and tolerance $\epsilon = 0.05$ (5\%), we obtain:
\begin{equation}
	\tau_{\beta} \leq \exp\left(-\frac{0.5}{0.05}\right) = \exp(-10) \approx 4.5 \times 10^{-5}.
\end{equation}

\paragraph{Practical Parameter Derivation.} To derive $\tau_{\text{neg}}$ from the theoretical bounds, we control the quality of tail negative pseudo-labels using $\sum_{i>k} p_i \leq \epsilon$. Given the gradient definition $p_{k+1} \leq p_k - \tau_{\text{neg}}$ and assuming the worst-case scenario where $p_{k+1}$ is the maximum in the tail:
\begin{equation}
	\sum_{i>k} p_i \leq (C-k)p_{k+1} \leq (C-k)(p_k - \tau_{\text{neg}}).
\end{equation}
With $\tau_{\alpha} = \tau_{\text{pos}} = 0.9$, we have:
\begin{equation}
	\sum_{i>k} p_i \leq (C-k)(0.9 - \tau_{\text{neg}}).
\end{equation}
To ensure the right side satisfies $\leq \epsilon$:
\begin{equation}
	0.9 - \tau_{\text{neg}} \leq \frac{\epsilon}{C-k} \Rightarrow \tau_{\text{neg}} \geq 0.9 - \frac{\epsilon}{C-k}.
\end{equation}
For typical SSA scenarios with $C \approx 10$, $k > 1$, and $\epsilon = 0.05$:
\begin{equation}
	\tau_{\text{neg}} \geq 0.9 - \frac{0.05}{10-1} \approx 0.9 - 0.0056 \approx 0.8944.
\end{equation}
	
Therefore, setting $\tau_{\text{pos}} = 0.9$ and controlling negative pseudo-label quality error to $\epsilon \leq 0.05$ yields $\tau_{\text{neg}} \approx 0.9$, which aligns with our experimental findings in the parameter sensitivity analysis.

\paragraph{Handling High-Entropy Samples.} \autoref{theorem} assumes $\mathcal{H}(p) \leq H_0$. When significant distribution shifts occur in the target domain, high-entropy samples become prevalent, making $\mathcal{H}(p) \approx H_0$ and potentially violating the theoretical assumptions. However, CACL maintains pseudo-label reliability through several mechanisms:

\begin{itemize}
	\item \textbf{Positive Pseudo-label Filtering.} The absolute probability threshold $\tau_{\text{pos}}$ automatically filters high-entropy samples. Positive pseudo-labels are accepted only when $p_i \geq \tau_{\text{pos}}$ (e.g., 0.9), while high-entropy samples typically have maximum probabilities well below 0.9 and are thus rejected. This ensures positive pseudo-labels are extracted only from low-entropy, confident samples.
	
	\item \textbf{Stable Negative Pseudo-label Selection.} Negative pseudo-labels rely on probability gradient ratios $\tau_{\text{neg}}$, which are more stable than absolute probabilities. For high-entropy samples with moderate maximum probabilities, tail categories still exhibit clear decreases. High entropy typically results from multiple possible categories with several high $p_i$ values, but tail probability differences remain more stable than global entropy, enabling reliable negative pseudo-label generation.
	
	\item \textbf{Automatic Weight Reduction.} CACL automatically reduces the influence of high-entropy samples during training. Both contrastive loss and pseudo-label loss depend on prediction probabilities, meaning high-entropy samples with flatter probability distributions contribute smaller gradients (based on CrossEntropyLoss). Low-entropy confident samples provide stronger supervision through pseudo-labels. This mechanism ensures that erroneous high-entropy pseudo-labels do not destabilize training, as reliable low-entropy pseudo-labels dominate gradient updates, guaranteeing training stability.
\end{itemize}

\section{Computational cost analysis}
\label{sec:appendix_cost_ana}

To comprehensively evaluate the computational efficiency of SSA, we conduct both qualitative and quantitative analyses of the resource requirements. We first analyze the theoretical complexity of each component, then perform empirical evaluations across different architectures, and finally explore optimization strategies.

\subsection{Qualitative Analysis of Computational Complexity}

To identify the factors contributing to the increased computational cost of SSA, we analyze its framework, which consists of \textit{Stepwise Domain Alignment} (SDA), \textit{Hierarchical Feature Alignment} (HFA), and \textit{Contrastive Adaptive Confidence Learning} (CACL). The CACL component, being a post-processing method, requires only $O(n\log n)$ sorting and $O(n)$ first-order gradient computation, thus introducing minimal computational overhead.

The computational overhead primarily arises from the SDA and HFA components. In the SDA stage, \textit{pseudo-source domain} samples undergo two passes through the HFA module: the first pass aligns with common semantics before mixing, and the second pass guides the \textit{remaining target domain} post-mixing. Consequently, the main computational cost increase is attributed to the MixMatch operations in SDA and the design of the HFA module.

For the MixMatch operations, the time complexity is determined by the number of data augmentations $K$, the batch size of labeled data (\textit{pseudo-source domain}) $L$, and the batch size of unlabeled data (\textit{remaining target domain}) $U$. Thus, the overall time complexity is $O(K \cdot L \cdot U)$.

For the HFA module, the complexity arises from two parts: the feature encoder and the attention module.

\textbf{Feature Encoder.} The feature encoder processes local and global information across different abstraction levels to capture semantics at different granularities. The time complexity impact is similar to MobileViT~\cite{MobileViT}'s local attention, reducing the original $O(n^2)$ operations to $O\left(k(\frac{n}{k})^2\right) = O\left(\frac{n^2}{k}\right)$. However, HFA additionally computes global attention to obtain coarse-grained semantic information, resulting in a complexity of:
\begin{equation}
	O\left((k+1)(\frac{n}{k})^2\right) = O\left(\frac{(k+1) \cdot n^2}{k^2}\right).
\end{equation}
 Here, $k$ denotes the number of local patches, and global semantic information is obtained by resizing the original image to local size, thus having a complexity of $O\left(\left(\frac{n}{k}\right)^2\right)$. In practice, $k \in [4, 9]$, ensuring $O\left(\frac{k+1}{k^2} \cdot n^2\right) < O(n^2)$. The main computational cost increase occurs in the subsequent attention fusion operations.

\textbf{Attention Module.} This module performs attention fusion on $(\frac{n}{k})^2$ pixels across $k+1$ patches, resulting in a complexity of $O\left(\frac{(k+1)^2}{k^2} \cdot n^2\right)$. Notably, this computational overhead is significant in semantic segmentation tasks. For single-image classification tasks, pixel-level attention scores for $k+1$ patches are not required, reducing the complexity to $O((k+1)^2)$, which is manageable under the condition $k \in [4, 9]$. Combining the complexities of the feature encoder and attention module, the overall time complexity of HFA is:
\begin{equation}
	O\left(\frac{k+1}{k^2} \cdot n^2 + \frac{(k+1)^2}{k^2} \cdot n^2\right) = O\left(\frac{(k+1)(k+2)}{k^2} \cdot n^2\right) > O(n^2).
\end{equation}
The computational cost increase is primarily due to the attention fusion process across different semantic levels in semantic segmentation tasks.

\subsection{Quantitative Experiments on Computational Cost}
We conduct comprehensive experiments to quantify the computational impact of different components. All experiments are performed on a single NVIDIA GeForce RTX 3090 (24GB) GPU. For semantic segmentation, we evaluate on SYNTHIA$\to$Cityscapes with batch size 2, local window size $\tau_{\text{lw}} = 1024 \times 1024$, and 40,000 iterations. For classification, we use Office-Home with batch size 64, $\tau_{\text{lw}} = 256 \times 256$, and 10 epochs.

\begin{table}[h]
	\centering
	\caption{Performance and Computational Cost Analysis on Semantic Segmentation.}
	\label{computational_segmentation}
	\small
	\resizebox{0.95\textwidth}{!}{
	\begin{tabular}{l c c c c c c}
		\toprule
		Method & GPU Memory (GB)  & Time-Item (s/item)& Time-Train (GPU·h)  & mIoU (\%) & Backbone & Model Size (MB) \\
		\midrule
		SSA & 22.4 & 2.40 & 26.7 & 64.1 & Segformer-B5 & 313.1 \\
		SSA w/o CACL & 22.4 & 2.40 & 26.4 & 63.4 & Segformer-B5 & 313.1 \\
		SSA w/o HFA & 20.2 & 1.52 & 18.8 & 61.3 & Segformer-B5 & 313.1 \\
		SSA w/o SDA & 19.9 & 2.40 & 14.3 & 62.7 & Segformer-B5 & 313.1 \\
		SSA-B4 & 20.4 & 2.12 & 22.1 & 63.5 & Segformer-B4 & 234.4 \\
		SSA-B3 & 17.0 & 1.84 & 17.4 & 62.3 & Segformer-B3 & 170.3 \\
		SSA-B2 & 13.1 & 1.66 & 12.4 & 61.2 & Segformer-B2 & 94.4 \\
		SSA-B1 & 12.6 & 1.50 & 11.7 & 54.3 & Segformer-B1 & 52.2 \\
		SSA-B0 & 11.7 & 1.47 & 11.2 & 48.9 & Segformer-B0 & 13.7 \\
		SSA-P2T & 12.3 & 1.49 & 11.5 & 60.2 & P2T-Base & 144.5 \\
		SSA-ResNet & 13.2 & 1.55 & 11.3 & 58.2 & ResNet-101 & 178.9 \\
		\cmidrule(lr){1-7}
		ATP & 20.1 & 1.56 & 19.6 & 63.7 & Segformer-B5 & 313.1 \\
		ATP-ResNet & 12.5 & 1.51 & 11.0 & 57.6 & ResNet-101 & 178.9 \\
		ATP-P2T & 11.9 & 1.32 & 9.7 & 59.6 & P2T-Base & 144.5 \\
		SFKT & 10.8 & 1.25 & 8.3 & 45.9 & ResNet-101 & 178.9 \\
		SFDA-Seg & 11.3 & 1.27 & 8.5 & 48.9 & ResNet-101 & 178.9 \\
		\bottomrule
	\end{tabular}
	}
\end{table}

\begin{table}[h]
	\centering
	\caption{Performance and Computational Cost Analysis on Classification.}
	\label{computational_classification}
	\small
	\resizebox{0.95\textwidth}{!}{
	\begin{tabular}{l c c c c c c}
		\toprule
		Method & GPU Memory (GB) & Time-Item (s/item) & Time-Train (GPU·h) & mAcc (\%) & Backbone & Model Size (MB) \\
%		& (GB) & (s/item) & (GPU·h) & (\%) &  & (MB) \\
		\midrule
		SSA & 10.8 & 0.65 & 0.17 & 85.0 & ResNet-50 & 102.6 \\
		SSA w/o CACL & 10.8 & 0.65 & 0.16 & 84.9 & ResNet-50 & 102.6 \\
		SSA w/o HFA & 9.7 & 0.57 & 0.15 & 84.6 & ResNet-50 & 102.6 \\
		SSA w/o SDA & 8.2 & 0.65 & 0.13 & 84.7 & ResNet-50 & 102.6 \\
		SSA-RN101 & 13.6 & 0.67 & 0.18 & 89.3 & ResNet-101 & 178.9 \\
		SSA-VGG16 & 13.9 & 0.67 & 0.18 & 86.7 & VGG16 & 533.4 \\
		\cmidrule(lr){1-7}
		SHOT & 7.6 & 0.52 & 0.12 & 71.8 & ResNet-50 & 102.6 \\
		SHOT-RN101 & 10.4 & 0.55 & 0.14 & 78.9 & ResNet-101 & 178.9 \\
		SHOT-VGG16 & 10.6 & 0.56 & 0.15 & 76.3 & VGG16 & 533.4 \\
		TENT & 6.9 & 0.49 & 0.11 & 61.7 & ResNet-50 & 102.6 \\
		TENT-RN101 & 9.7 & 0.52 & 0.13 & 70.2 & ResNet-101 & 178.9 \\
		TENT-VGG16 & 10.2 & 0.54 & 0.14 & 68.3 & VGG16 & 533.4 \\
		SFDA+ & 7.9 & 0.55 & 0.14 & 73.4 & ResNet-50 & 102.6 \\
		\bottomrule
	\end{tabular}
}
\end{table}

The computational cost analysis provides insights into SSA's efficiency characteristics across different components and architectures. Here, SSA represents our proposed method, while SSA w/o CACL/HFA/SDA denote variants excluding respective components. CACL operates as post-processing during training without affecting inference, HFA involves architectural modifications impacting both training and inference, and SDA is a training methodology affecting only the training process.

\textbf{Component-wise Impact.} From \autoref{computational_segmentation}, HFA contributes most significantly to computational overhead. Comparing SSA (22.4 GB, 26.7 GPU hours) with SSA w/o HFA (20.2 GB, 18.8 GPU hours), HFA increases memory usage by 10.9\% and training time by 42.0\% due to pixel-level attention fusion requirements. SDA also substantially impacts training efficiency, with SSA w/o SDA requiring only 14.3 GPU hours compared to 26.7 hours for the full method (46.4\% reduction). CACL shows minimal computational impact, with negligible differences in memory (22.4 GB) and training time (26.4 vs. 26.7 GPU hours).

\textbf{Architecture Scalability.} \autoref{computational_segmentation} and \autoref{computational_classification} demonstrate clear scalability trends. In segmentation, Segformer variants show progressive resource requirements: B0 (11.7 GB, 11.2 GPU hours, 48.9\% mIoU) to B5 (22.4 GB, 26.7 GPU hours, 64.1\% mIoU). Alternative architectures like P2T-Base (12.3 GB, 11.5 GPU hours, 60.2\% mIoU) and ResNet-101 (13.2 GB, 11.3 GPU hours, 58.2\% mIoU) offer different performance-efficiency trade-offs. For classification, SSA with ResNet-101 achieves higher accuracy (89.3\% vs. 85.0\%) at increased cost (13.6 GB vs. 10.8 GB).

\subsection{Optimization Strategies}

Based on the computational analysis, three optimization strategies are proposed to improve the efficiency of SSA while maintaining competitive performance:

\textbf{Lightweight Attention Mechanisms.} Analysis reveals that the attention fusion operations in HFA are the primary cause of significant computational overhead increases. Therefore, employing more lightweight attention mechanisms can effectively reduce SSA's computational cost. As demonstrated in \autoref{optimization_results}, replacing the attention mechanism in HFA with more efficient MobileViT operations reduces GPU memory from 22.4GB to 21.2GB while maintaining comparable performance (64.1\% vs 63.9\% mIoU for semantic segmentation tasks).

\textbf{Simplifying SDA's Two-Step Alignment.} When domain shift is relatively small, direct alignment between the pseudo-source domain and remaining target domain can reduce HFA usage by one step while avoiding PRE-NET storage and computation, theoretically improving training efficiency. We propose using self-entropy as an indicator for domain shift assessment.

As shown in \autoref{optimization_results}, when the average entropy is below 0.5 during source model testing on the target domain, simplifying SDA's two-step alignment can improve training efficiency without significantly affecting model performance. For semantic segmentation tasks with dense semantics (SYNTHIA$\to$Cityscapes: 0.5396, GTA5$\to$Cityscapes: 0.7239), removing SDA results in substantial performance degradation (1.4\% and 3.6\% mIoU drop respectively). In contrast, for image classification tasks with sparse semantics (entropy < 0.35), removing SDA provides significant computational savings (reducing GPU memory from 10.8GB to 8.2GB) with minimal performance impact (< 1.2\% accuracy drop). This suggests that the entropy threshold of 0.5 serves as an effective criterion for determining when to simplify SDA operations.

\textbf{Simplifying HFA for Classification Tasks.} For single-label classification tasks with sparse semantics, HFA's hierarchical feature alignment provides limited benefits as these tasks primarily rely on global semantic features rather than fine-grained local patterns. As demonstrated in \autoref{optimization_results}, removing HFA from classification tasks reduces GPU memory usage by 24\% (from 10.8GB to 8.2GB) and training time by 20-25\% while maintaining comparable performance. The performance drops are minimal: 0.2\% for Office-31, 0.3\% for Office-Home, 1.0\% for VisDA-C, and 1.7\% for DomainNet. This suggests that for classification tasks, especially those with low semantic complexity, HFA can be safely omitted to achieve better computational efficiency. This suggests that the entropy threshold of 0.5 serves as an effective criterion for determining when to simplify SDA operations, while HFA simplification is recommended for all single-label classification tasks to balance performance and efficiency.

\begin{table}[h]
	\centering
	\caption{Performance Optimization Results Across Different Tasks.}
	\label{optimization_results}
	\small
	\resizebox{0.95\textwidth}{!}{
		\begin{tabular}{l c c c c c c}
			\toprule
			Method & GPU Memory (GB) & Time-Train (GPU·h)  & mIoU/mAcc  (\%) & Task & Dataset & Backbone \\
			\midrule
			SSA & 22.4 & 26.5 & 64.1 & \multirow{6}{*}{Segmentation} & {SYNTHIA$\to$Cityscapes} (0.5396)&  \multirow{6}{*}{Segformer-B5} \\
			SSA+MobileViT & 21.2 & 24.2 & 63.9 &  & SYNTHIA$\to$Cityscapes (0.5396) &  \\
			SSA w/o SDA & 19.9 & 14.5 & 62.7 &  & SYNTHIA$\to$Cityscapes (0.5396)&  \\
			SSA & 22.4 & 26.7 & 69.2 &  & GTA5$\to$Cityscapes (0.7239)&  \\
			SSA+MobileViT & 21.2 & 24.3 & 68.9 &  & GTA5$\to$Cityscapes (0.7239)&  \\
			SSA w/o SDA & 19.9 & 14.3 & 65.6 &  & GTA5$\to$Cityscapes (0.7239)&  \\
			\cmidrule(lr){1-7}
			SSA & 10.8 & 0.06 & 92.8 & \multirow{8}{*}{Classification} & Office-31 (0.1335) & \multirow{8}{*}{ResNet-50} \\
			SSA w/o HFA & 8.2 & 0.04 & 92.6 &  & Office-31 (0.1335) &  \\
			SSA w/o SDA & 8.2 & 0.04 & 92.7 &  & Office-31 (0.1335) &  \\
			SSA & 10.8 & 0.17 & 85.0 &  & Office-Home (0.2491) &  \\
			SSA w/o HFA & 8.2 & 0.13 & 84.7 &  & Office-Home (0.2491) &  \\
			SSA w/o SDA & 8.2 & 0.13 & 84.7 &  & Office-Home (0.2491) &  \\
			SSA & 10.8 & 1.72 & 91.3 &  & VisDA-C (0.1677) &  \\
			SSA w/o HFA & 8.2 & 1.36 & 90.3 &  & VisDA-C (0.1677) &  \\
			SSA w/o SDA & 8.2 & 1.36 & 90.6 &  & VisDA-C (0.1677) &  \\
			SSA & 10.8 & 1.98 & 83.1 &  & DomainNet (0.3478) &  \\
			SSA w/o HFA & 8.2 & 1.53 & 81.4 &  & DomainNet (0.3478) &  \\
			SSA w/o SDA & 8.2 & 1.53 & 81.9 &  & DomainNet (0.3478) &  \\
			\bottomrule
		\end{tabular}
	}
\end{table}

\subsection{Computational Cost vs Performance Trade-off.} The balance between computational cost and performance is critical in TTA tasks. \autoref{tab:tradeoff_analysis} presents a comprehensive trade-off analysis using various Segformer architectures (B0-B5), demonstrating an inverse relationship between computational cost and segmentation performance. Considering diminishing performance gains as computational cost increases, Segformer-B2 represents a balanced choice.

SSA achieves a favorable trade-off between efficiency and performance. For classification tasks with sparse semantics, removing SDA provides significant computational savings (from 10.8GB to 8.2GB GPU memory) with minimal performance degradation. For semantic segmentation tasks with dense semantics, lightweight attention mechanisms offer a balanced solution. While SSA requires more resources than lightweight methods like TENT, the significant performance improvements justify the additional computational cost for accuracy-prioritized applications.

\begin{table}[h!]
	\centering
	\renewcommand{\arraystretch}{1.2}
	\caption{Trade-off Analysis Between Computational Cost and Segmentation Performance}
	\label{tab:tradeoff_analysis}
	\small
	\resizebox{0.8\textwidth}{!}{
	\begin{tabular}{lccccc}
		\toprule
		Method & GPU Memory (GB) & Time-Item (s/item) & Time-Train (GPU$\cdot$h) & mIoU (\%) & Backbone \\ 
		\midrule
		SSA-B0 & 11.7 (+0.0) & 1.47 (+0.00) & 11.2 (+0.0) & 48.9 (+0.0) & Segformer-B0 \\ 
		SSA-B1 & 12.6 (+0.9) & 1.50 (+0.03) & 11.7 (+0.5) & 54.3 (+5.4) & Segformer-B1 \\ 
		SSA-RN101 & 13.2 (+1.5) & 1.55 (+0.08) & 11.3 (+0.1) & 58.2 (+9.3) & ResNet-101 \\ 
		SSA-P2T & 12.3 (+0.6) & 1.49 (+0.02) & 11.5 (+0.3) & 60.2 (+11.3) & P2T-Base \\ 
		SSA-B2 & 13.1 (+1.4) & 1.66 (+0.09) & 12.4 (+1.2) & 61.2 (+12.3) & Segformer-B2 \\ 
		SSA-B3 & 17.0 (+5.3) & 1.84 (+0.37) & 17.4 (+6.2) & 62.3 (+13.4) & Segformer-B3 \\ 
		SSA-B4 & 20.4 (+8.7) & 2.12 (+0.65) & 22.1 (+10.9) & 63.5 (+14.6) & Segformer-B4 \\ 
		SSA-B5 & 22.4 (+10.7) & 2.40 (+0.93) & 26.7 (+15.5) & 64.1 (+15.2) & Segformer-B5 \\ 
		\bottomrule
	\end{tabular}
}
\end{table}

\newpage
\section{Supplementary Experiment}
\label{sec:appendix_exp}

\subsection{Cross-Module Dependency Analysis}
\label{sec:appendix_albation}
To address concerns regarding insufficient analysis of cross-module dependencies, comprehensive ablation experiments were conducted to analyze the relationships between HFA, CACL, and SDA components. The results are presented in \autoref{tab:cross_module_ablation}.

The baseline represents source-free domain adaptation using SHOT-based source hypothesis transfer. Configuration (a) modifies the backbone encoder to HFA, (b) adds CACL post-processing to the baseline, (c) performs stepwise alignment on the baseline, (d) removes SDA operations from SSA for direct one-step alignment, (e) removes CACL post-processing from SSA, (f) replaces the SSA encoder with the backbone encoder, and (g) represents the complete SSA method.

The experimental results reveal several key insights regarding cross-module dependencies:

\begin{itemize}
	\item \textbf{Individual Component Contributions.} From experiments (a), (b), and (c), HFA contributes most significantly to SSA performance improvement, followed by SDA, with CACL showing the weakest individual contribution.
	
	\item \textbf{HFA-CACL Synergy.} Comparing baseline with (b) and (a) with (d), CACL demonstrates higher performance gains when HFA is present (57.5$\to$65.6) compared to without HFA (44.5$\to$50.2). This indicates that HFA enhances semantic feature quality, thereby improving CACL's label quality.
	
	\item \textbf{HFA-SDA Interaction.} Similarly, comparing baseline with (c) and (a) with (e), SDA shows more pronounced performance gains with HFA (57.5$\to$66.7) than without (44.5$\to$53.4).
	
	\item \textbf{CACL-SDA Relationship.} Comparing baseline with (b) and (c) with (f), CACL demonstrates more effective performance when SDA is present (44.5$\to$50.2 vs. 53.4$\to$60.3). Conversely, CACL and SDA also promote HFA's performance enhancement (from 44.5$\to$57.5 to 50.2$\to$65.6 and 53.4$\to$66.7). However, CACL's promotion effect on SDA is less pronounced (44.5$\to$53.4 vs. 50.2$\to$60.3).
	
	\item \textbf{Diminishing Returns.} Under conditions of HFA+CACL, CACL+SDA, or HFA+SDA, the performance gains of the remaining modules are relatively lower compared to implementing only one module or the baseline. This phenomenon may be attributed to diminishing marginal effects as segmentation performance improves.
\end{itemize}

These findings demonstrate that the three components exhibit strong interdependencies, with HFA serving as the foundation for enhancing the effectiveness of both CACL and SDA, while CACL and SDA provide complementary improvements that collectively contribute to the superior performance of the complete SSA framework.

\subsection{Result Visualization Analysis}
\label{sec:appendix_visual_res_ana}

\subsubsection{Confusion matrix for image classification}
\label{sec:appendix_confusion_matrix}

To better understand the behavior of SSA under different classification scenarios, we conduct a confusion matrix analysis across three datasets in decreasing order of scale and complexity: VisDA-C, Office-Home, and Office-31. This analysis complements the quantitative accuracy metrics by highlighting which classes are most frequently confused and why. It also provides insight into the method’s fine-grained discriminative ability and its robustness to inter-class semantic proximity or visual ambiguity under domain shift.

\paragraph{Confusion Analysis on Office-Home.} Office-Home is a medium-scale benchmark spanning 4 domains and 65 categories. As shown in \autoref{cm_officehome}, confusion in the \textit{Ar}$\rightarrow$\textit{Rw} task is more dispersed due to the dataset’s fine-grained label space and higher visual variability. For instance, \textit{monitor} is often misclassified as \textit{computer} (12 times), \textit{bottle} as \textit{soda} (15 times), and \textit{pen} as \textit{marker}, reflecting the challenges in distinguishing semantically related and visually similar objects. While SSA mitigates domain shift through hierarchical integration of local and global semantics, its fixed fusion strategy may struggle with class-specific ambiguity, particularly in densely populated and complex label spaces. This highlights the potential benefit of adopting more adaptive aggregation mechanisms that better capture intermediate semantic complexity and enhance discrimination.

\begin{figure}[H]
    \centering
    % -------- Row 1 --------
    \begin{subfigure}[t]{0.30\textwidth}
        \includegraphics[width=\textwidth]{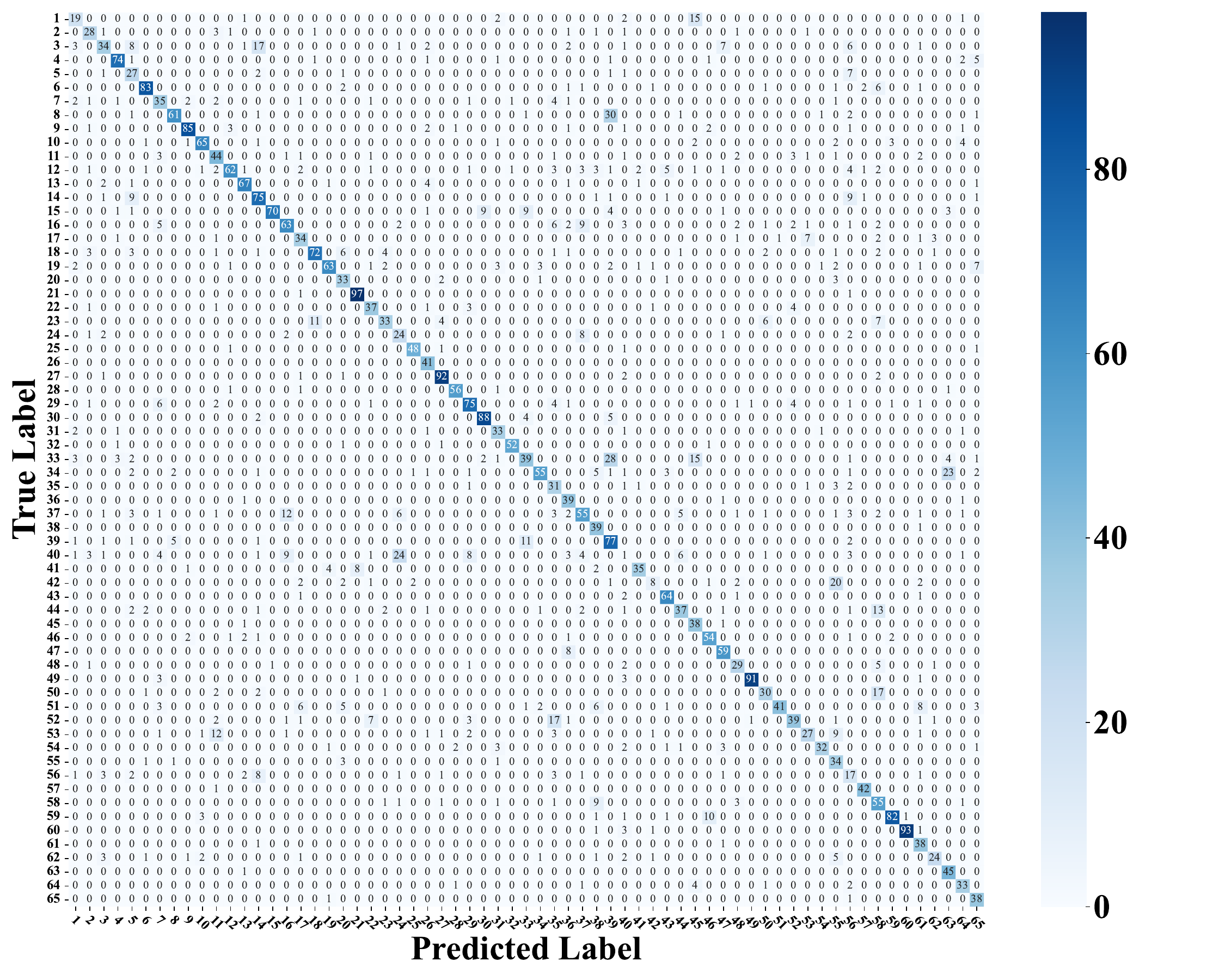}
        \caption{Ar$\to$Cl}
    \end{subfigure}
    \hfill
    \begin{subfigure}[t]{0.30\textwidth}
        \includegraphics[width=\textwidth]{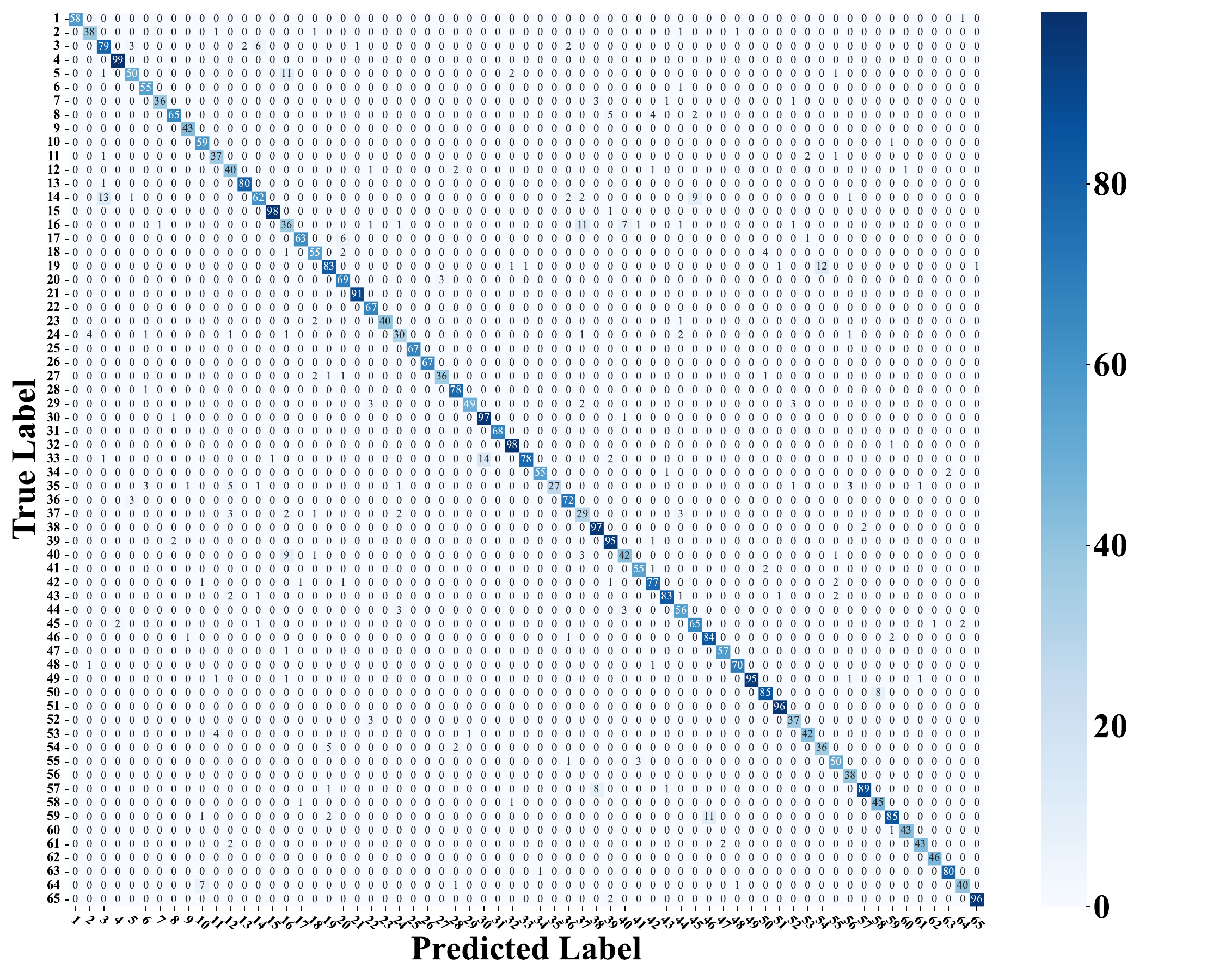}
        \caption{Ar$\to$Pr}
    \end{subfigure}
    \hfill
    \begin{subfigure}[t]{0.30\textwidth}
        \includegraphics[width=\textwidth]{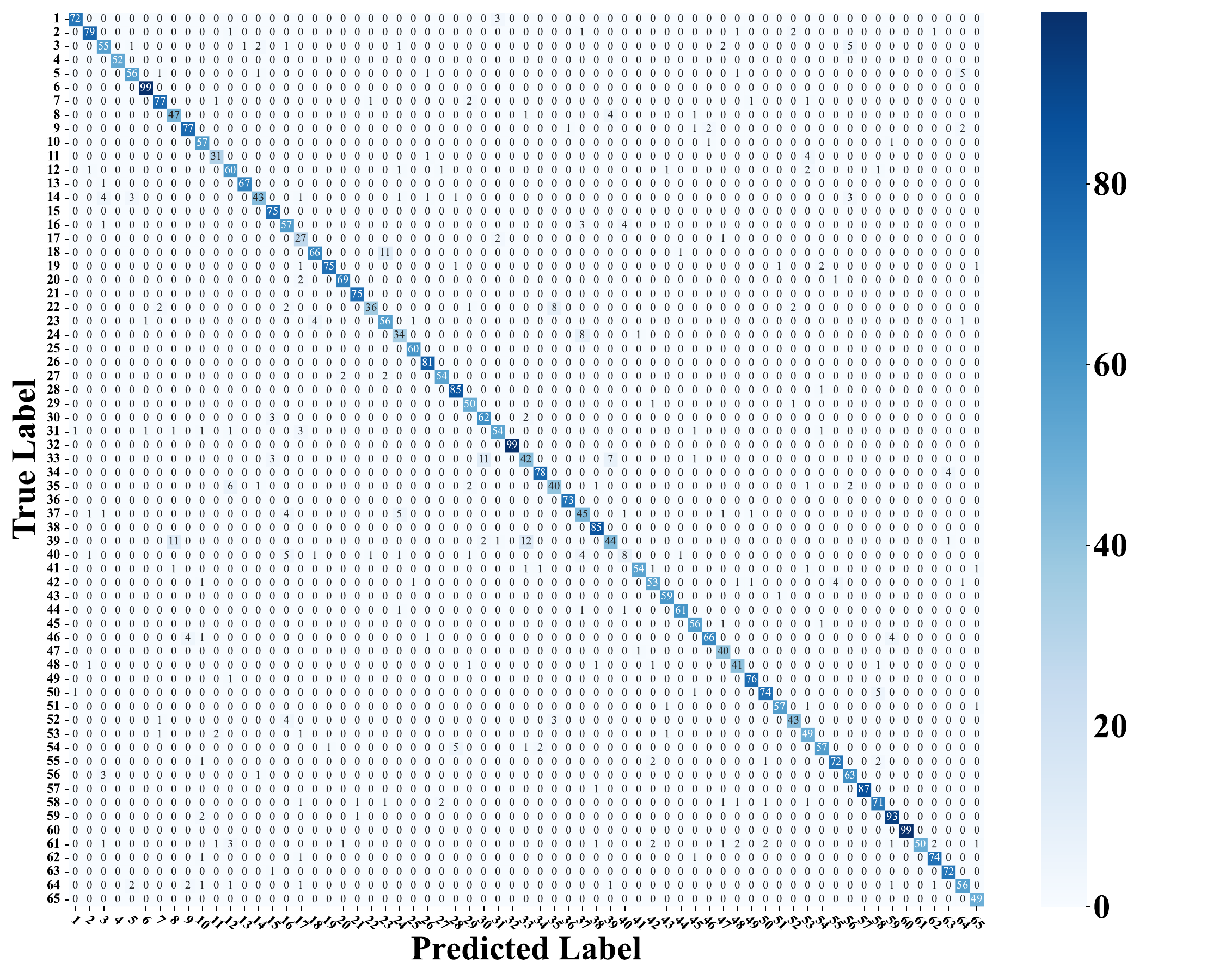}
        \caption{Ar$\to$Rw}
    \end{subfigure}
    
    % -------- Row 2 --------
    \begin{subfigure}[t]{0.30\textwidth}
        \includegraphics[width=\textwidth]{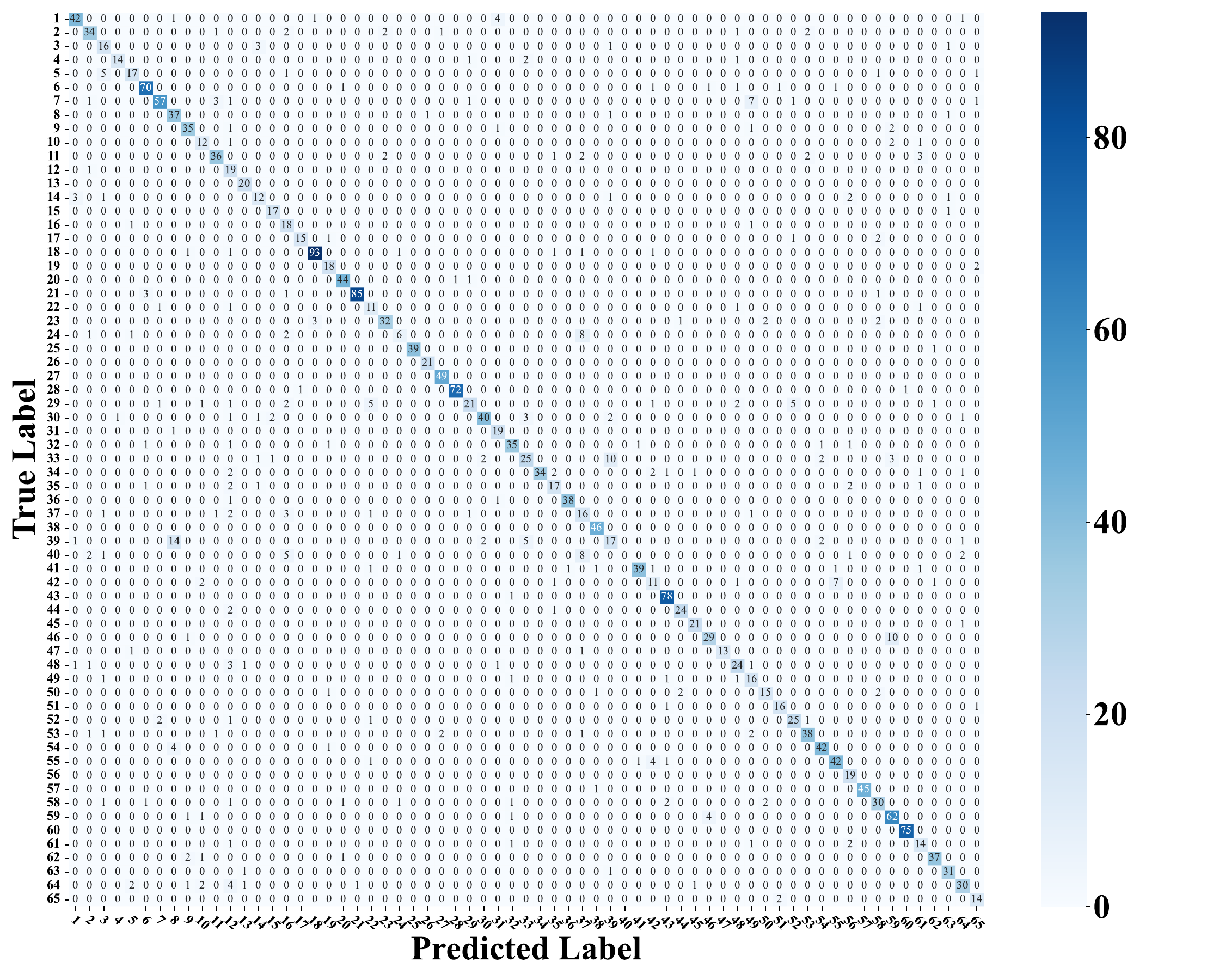}
        \caption{Cl$\to$Ar}
    \end{subfigure}
    \hfill
    \begin{subfigure}[t]{0.30\textwidth}
        \includegraphics[width=\textwidth]{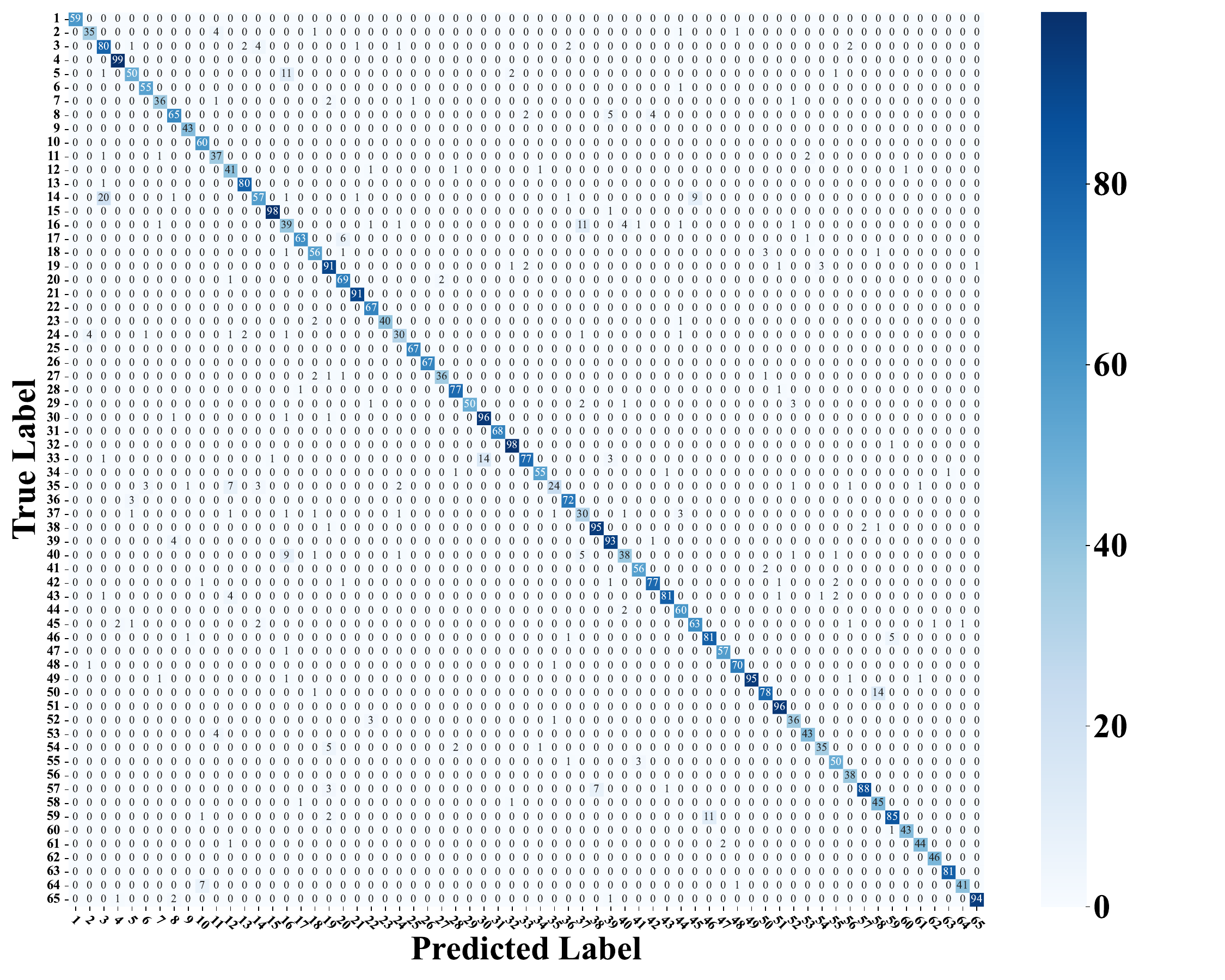}
        \caption{Cl$\to$Pr}
    \end{subfigure}
    \hfill
    \begin{subfigure}[t]{0.30\textwidth}
        \includegraphics[width=\textwidth]{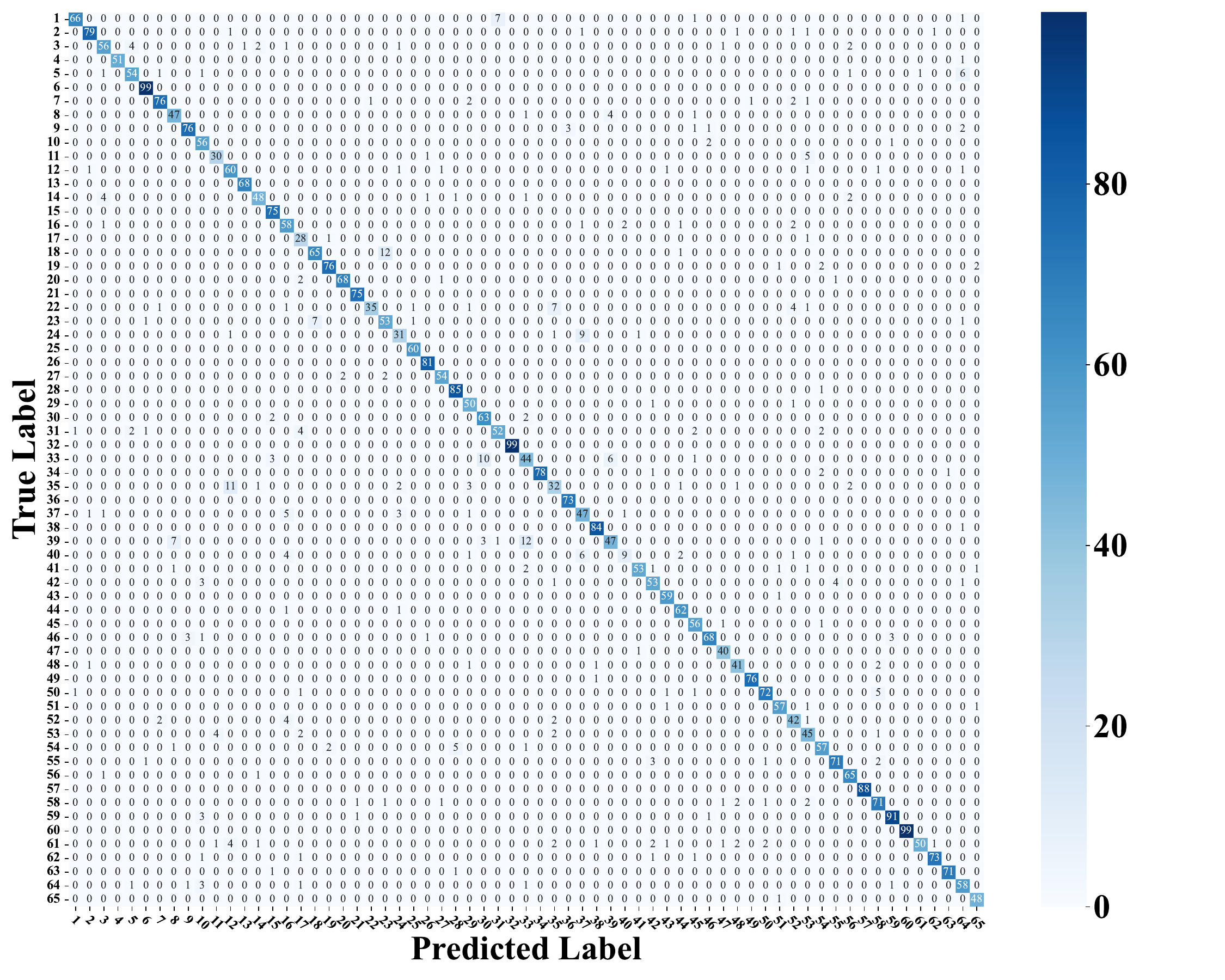}
        \caption{Cl$\to$Rw}
    \end{subfigure}
    
    % -------- Row 3 --------
    \begin{subfigure}[t]{0.30\textwidth}
        \includegraphics[width=\textwidth]{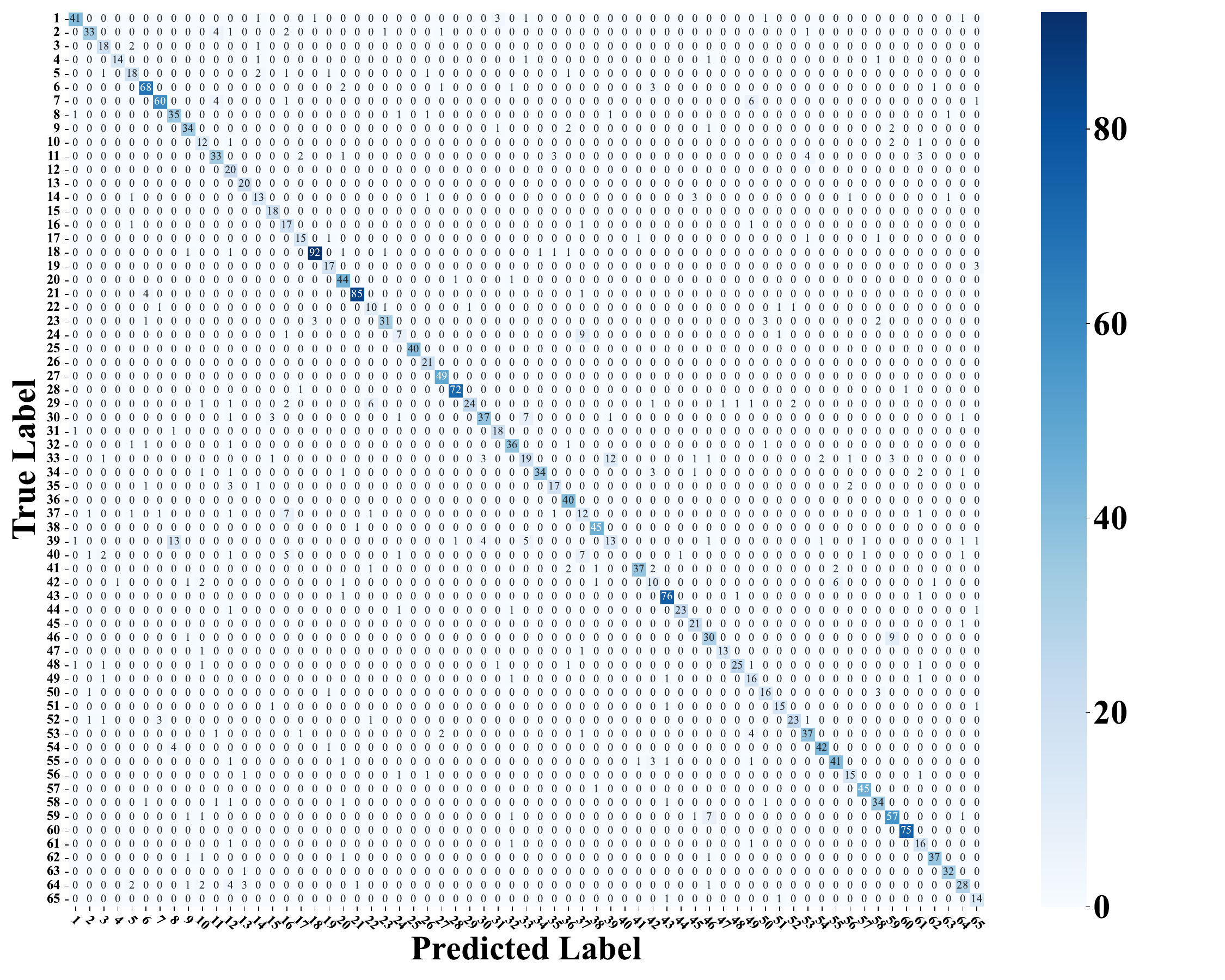}
        \caption{Pr$\to$Ar}
    \end{subfigure}
    \hfill
    \begin{subfigure}[t]{0.30\textwidth}
        \includegraphics[width=\textwidth]{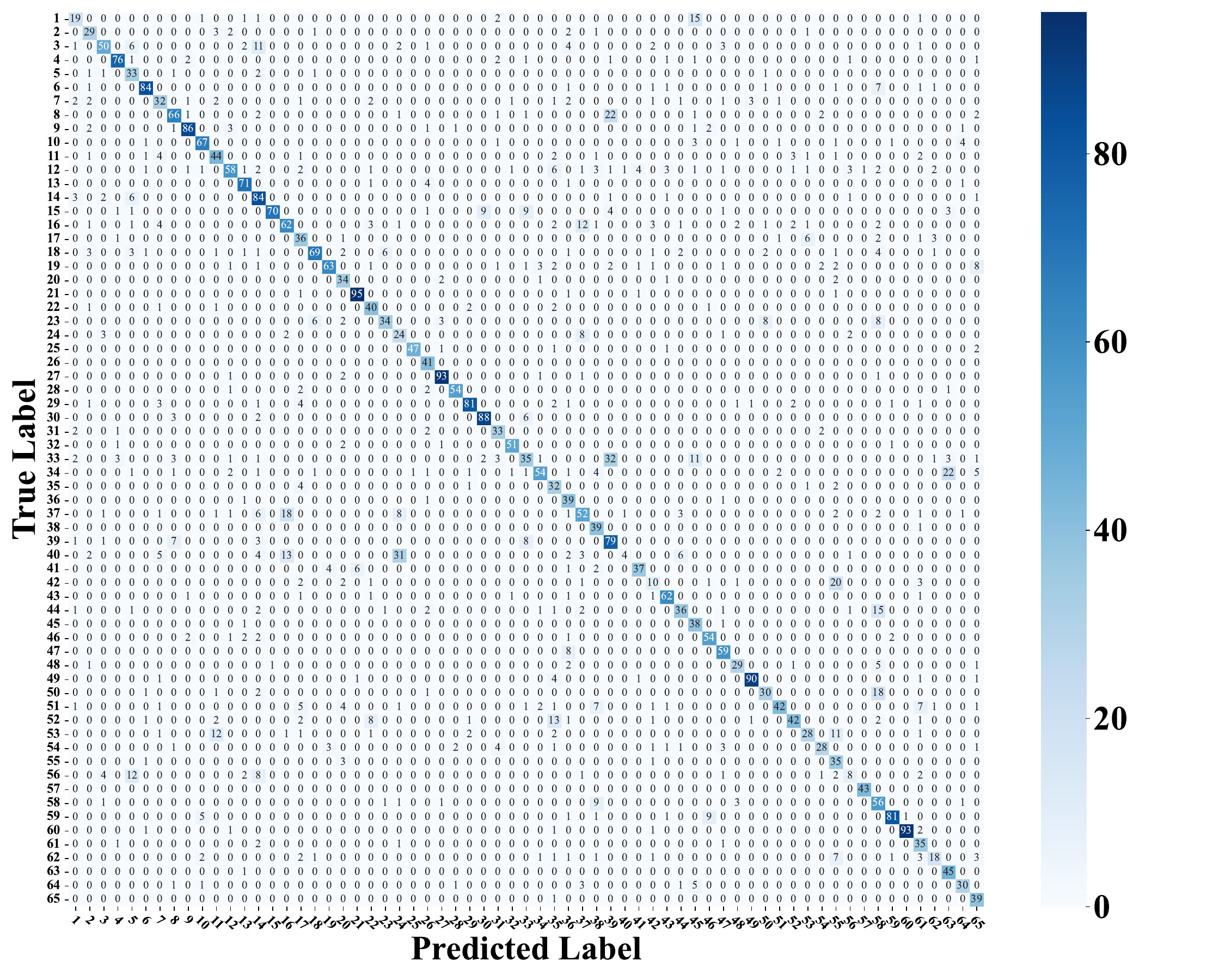}
        \caption{Pr$\to$Cl}
    \end{subfigure}
    \hfill
    \begin{subfigure}[t]{0.30\textwidth}
        \includegraphics[width=\textwidth]{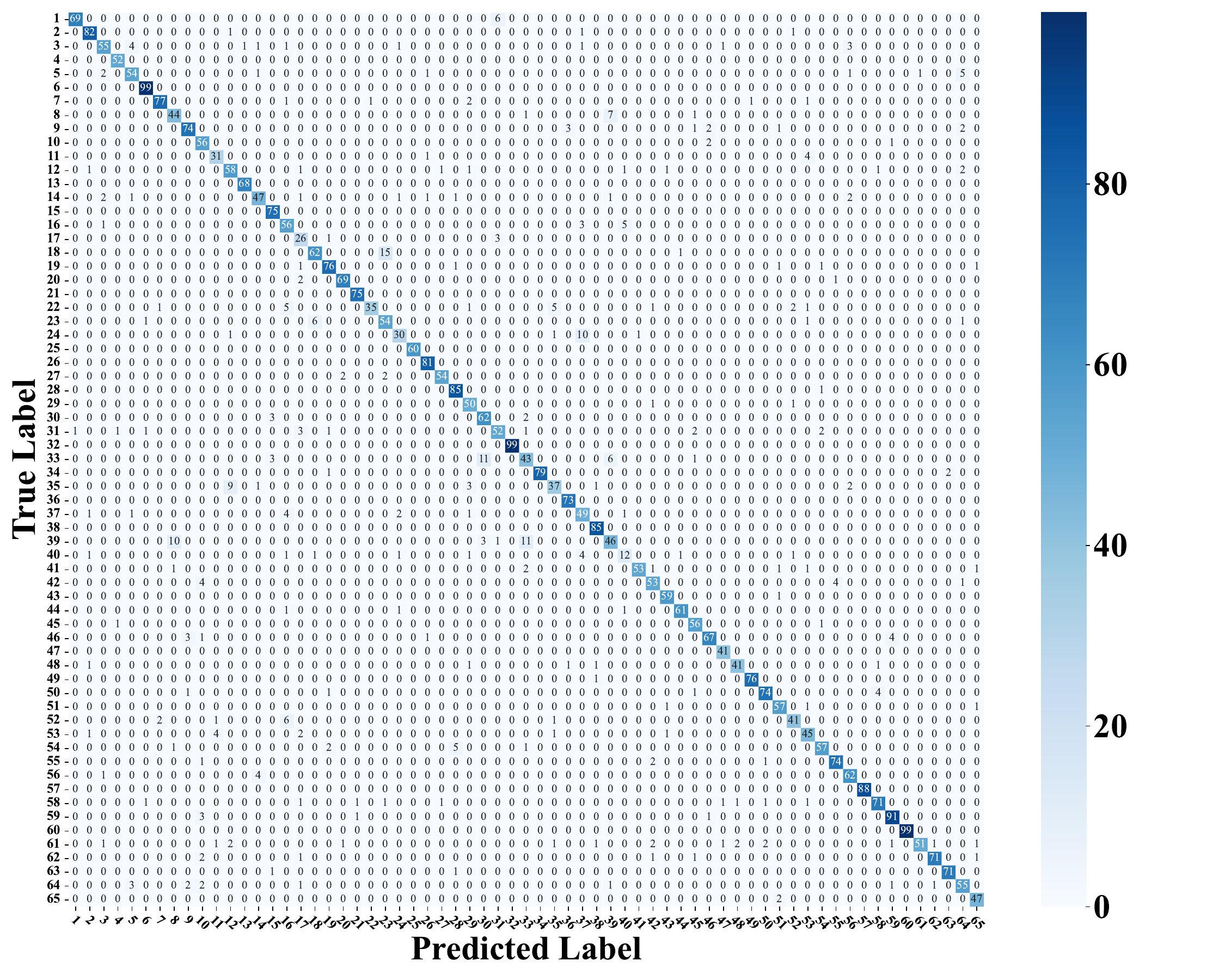}
        \caption{Pr$\to$Rw}
    \end{subfigure}
    
    % -------- Row 4 --------
    \begin{subfigure}[t]{0.30\textwidth}
        \includegraphics[width=\textwidth]{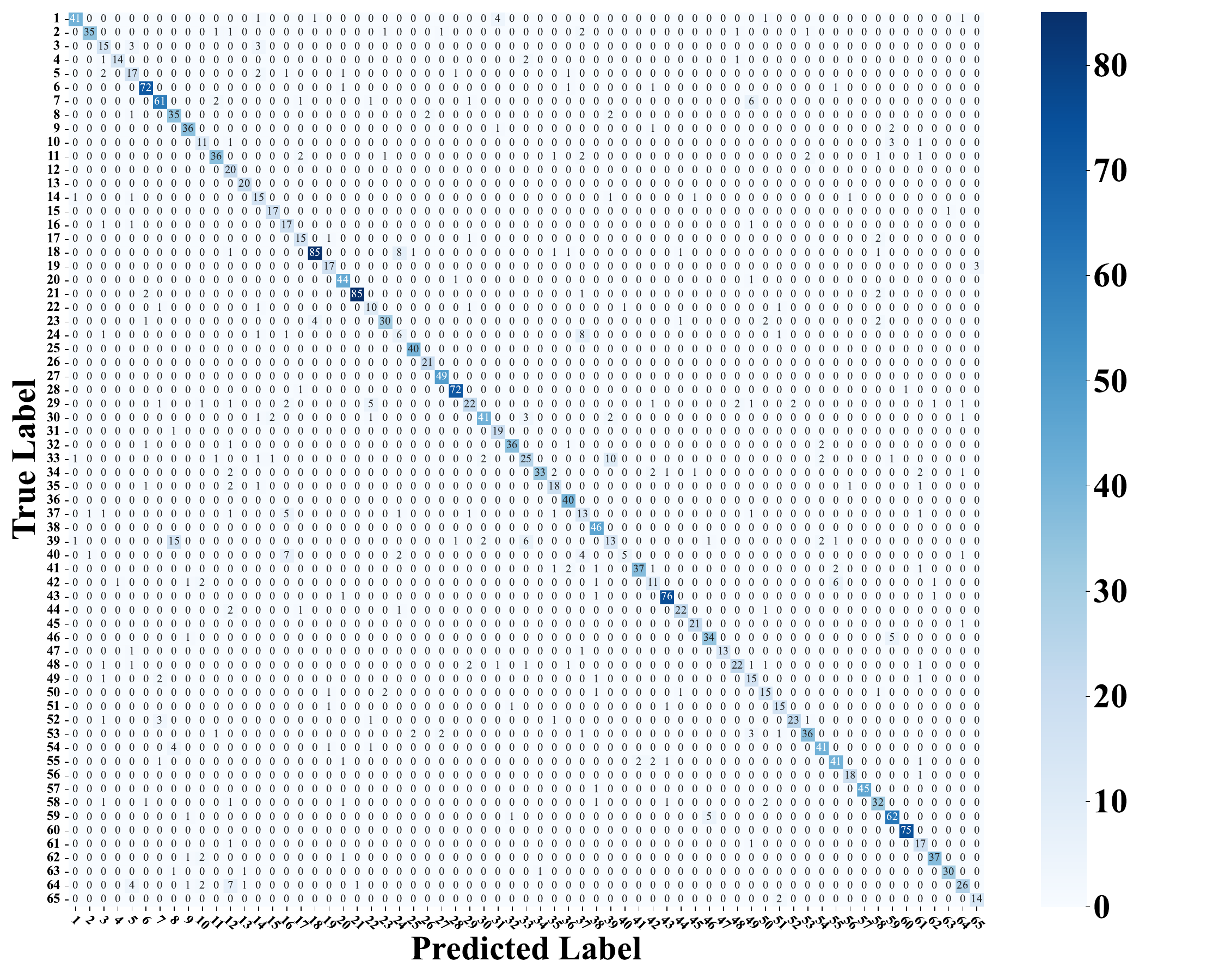}
        \caption{Rw$\to$Ar}
    \end{subfigure}
    \hfill
    \begin{subfigure}[t]{0.30\textwidth}
        \includegraphics[width=\textwidth]{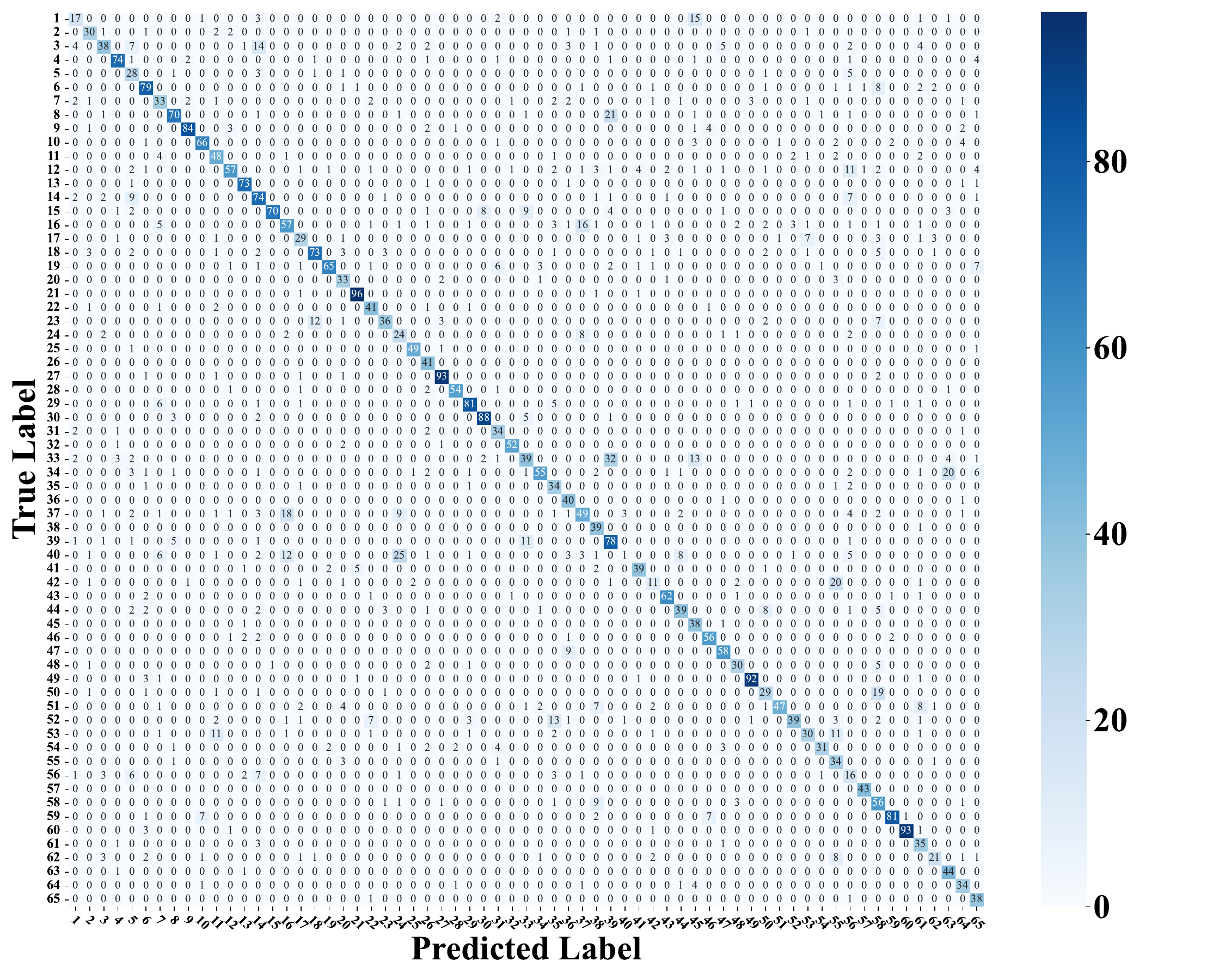}
        \caption{Rw$\to$Cl}
    \end{subfigure}
    \hfill
    \begin{subfigure}[t]{0.30\textwidth}
        \includegraphics[width=\textwidth]{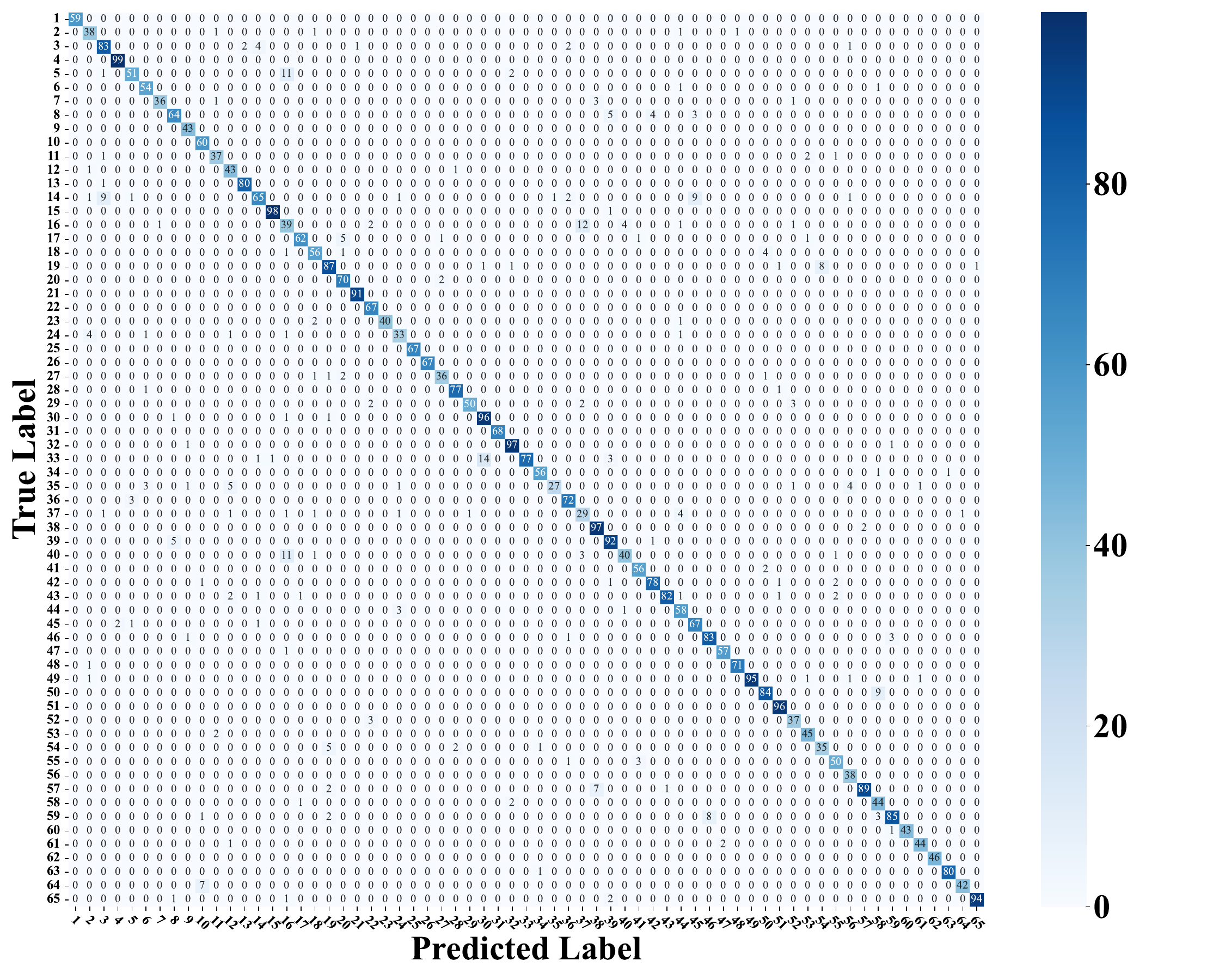}
        \caption{Rw$\to$Pr}
    \end{subfigure}
    \begin{subfigure}[t]{0.95\textwidth}
        \includegraphics[width=\textwidth]{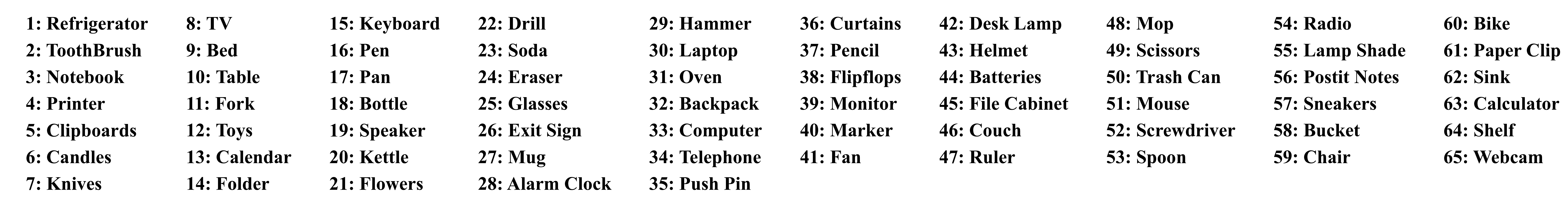}
    \end{subfigure}
    \caption{Confusion matrices for different domain shifts in the Office-Home dataset.}
    \label{cm_officehome}
\end{figure}

\paragraph{Confusion Analysis on VisDA-C.} VisDA-C is a challenging synthetic-to-real benchmark comprising over 280K images across 12 object categories. As shown in \autoref{cm_visda}, the confusion matrix for the \textit{train}$\rightarrow$\textit{val} task unveils frequent misclassifications among visually similar classes. For instance, \textit{car} is often misclassified as \textit{person} (916 instances), and there is substantial confusion between \textit{bus} and \textit{truck}, as well as between \textit{bicycle} and \textit{motorcycle}. These trends reflect the difficulty of preserving fine-grained semantics under domain shift, especially in real-world scenes characterized by complex textures, occlusions, and noisy backgrounds. While SSA successfully mitigates many such confusions through its hierarchical semantic alignment mechanism, its use of fixed-level feature aggregation imposes certain constraints. Specifically, it may struggle in scenarios with high intra-class variability or overlapping inter-class appearances. This suggests that future improvements could benefit from incorporating more flexible, instance-aware representations capable of adapting to subtle semantic distinctions and dynamic visual contexts.

\begin{figure}[t]
    \centering
    \includegraphics[width=0.8\textwidth]{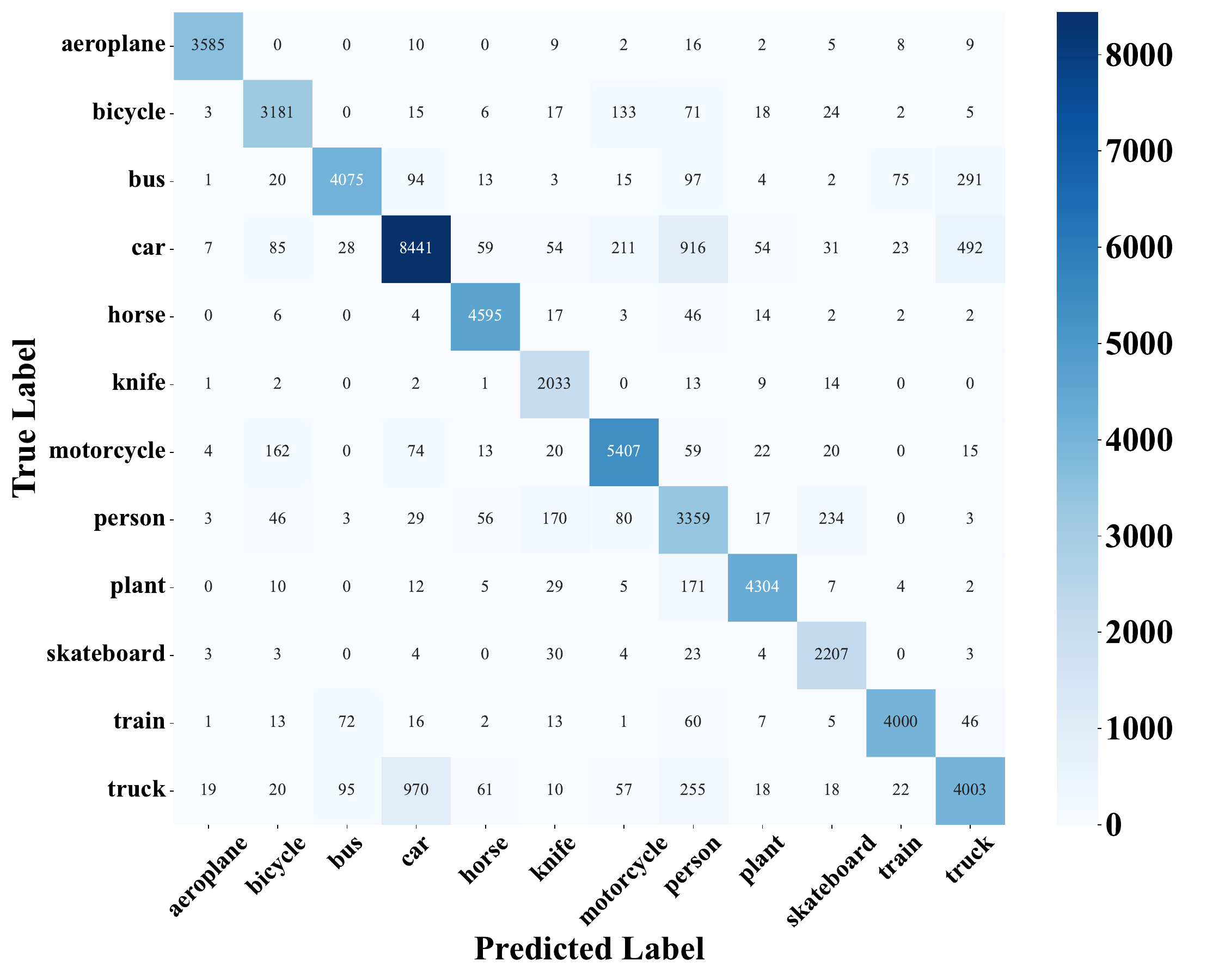}{}
    \caption{Confusion matrices for different domain shifts in the VisDA-C dataset.}
    \label{cm_visda}
    % \small
\end{figure}

\paragraph{Confusion Analysis on Office-31.} Office-31 is a small-scale dataset with 3 domains and 6 transfer tasks. In the Amazon$\rightarrow$DSLR (A$\rightarrow$D) setting (\autoref{cm_office31}), the model achieves strong overall performance but still confuses visually or contextually similar classes, such as \textit{monitor} vs. \textit{desk lamp} and \textit{laptop} vs. \textit{ring binder}. These errors likely stem from shared shapes or frequent co-occurrence. While SSA handles mild domain shifts well via stepwise alignment, its global fusion may overlook subtle distinctions, suggesting the need for more localized or instance-aware modeling.

\begin{figure}[t]
    \centering
    % -------- Row 1 --------
    \begin{subfigure}[t]{0.23\textwidth}
        \includegraphics[width=\textwidth]{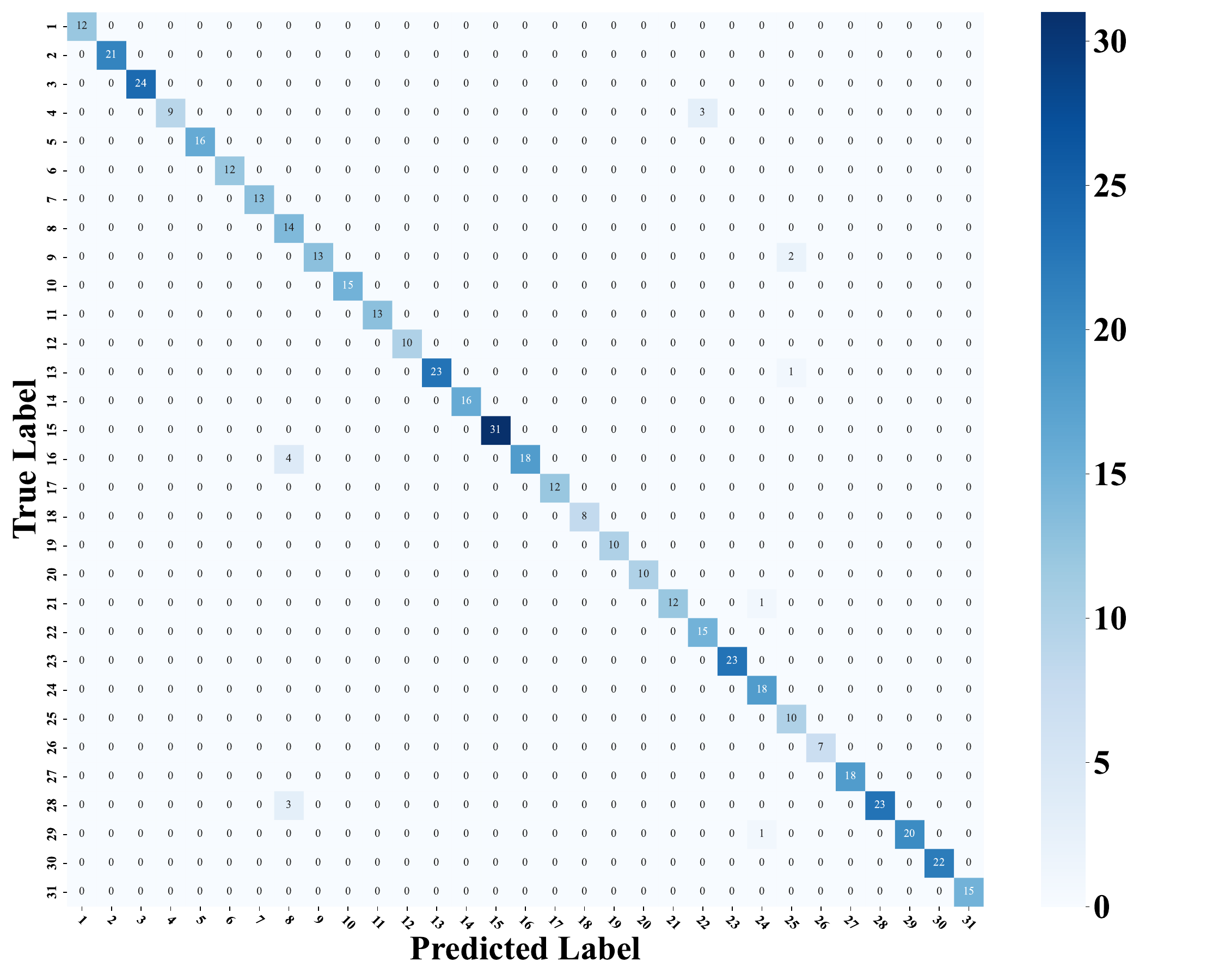}
        \caption{A$\to$D}
    \end{subfigure}
    % \hfill
    \begin{subfigure}[t]{0.23\textwidth}
        \includegraphics[width=\textwidth]{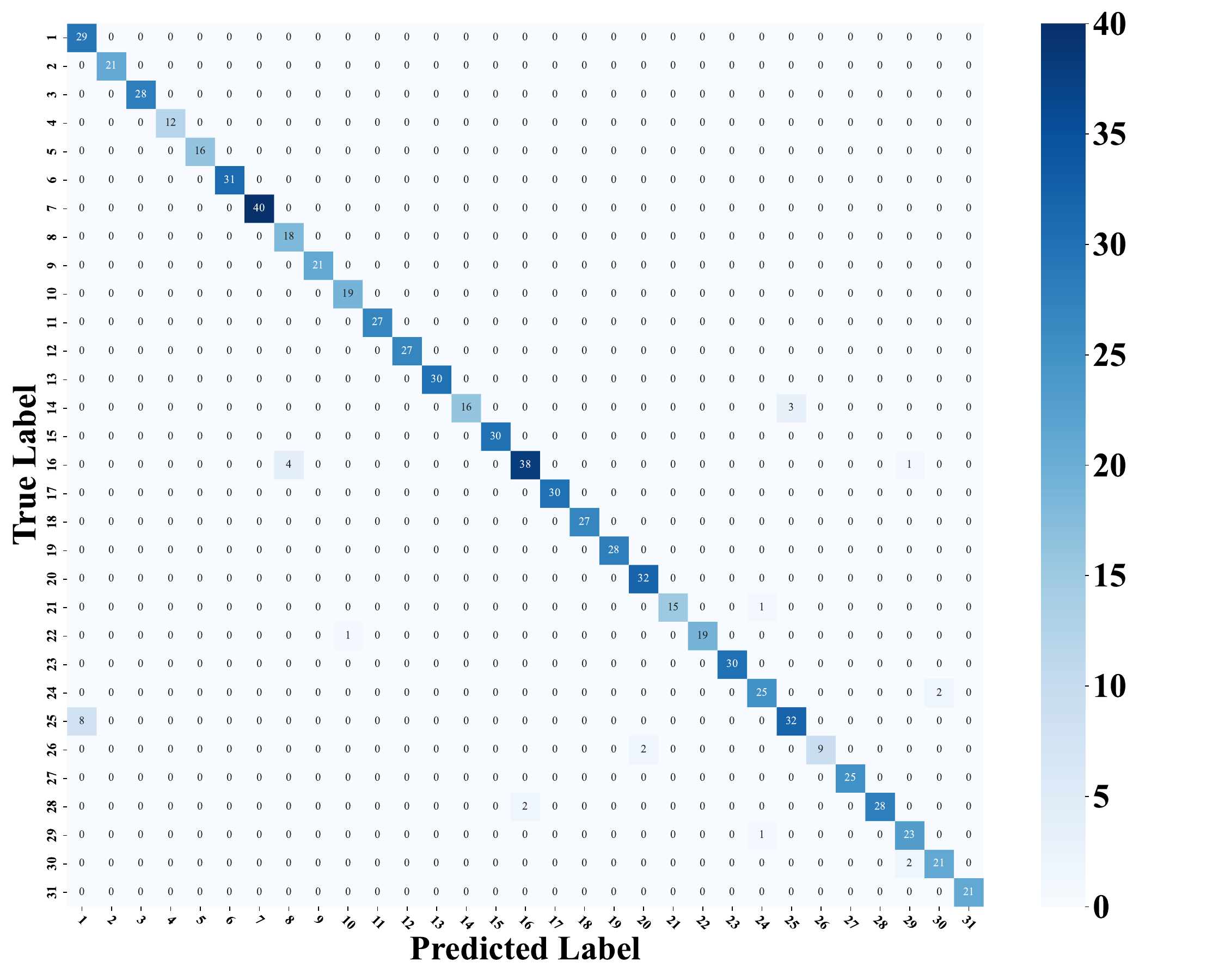}
        \caption{A$\to$W}
    \end{subfigure}
    % \hfill
    \begin{subfigure}[t]{0.23\textwidth}
        \includegraphics[width=\textwidth]{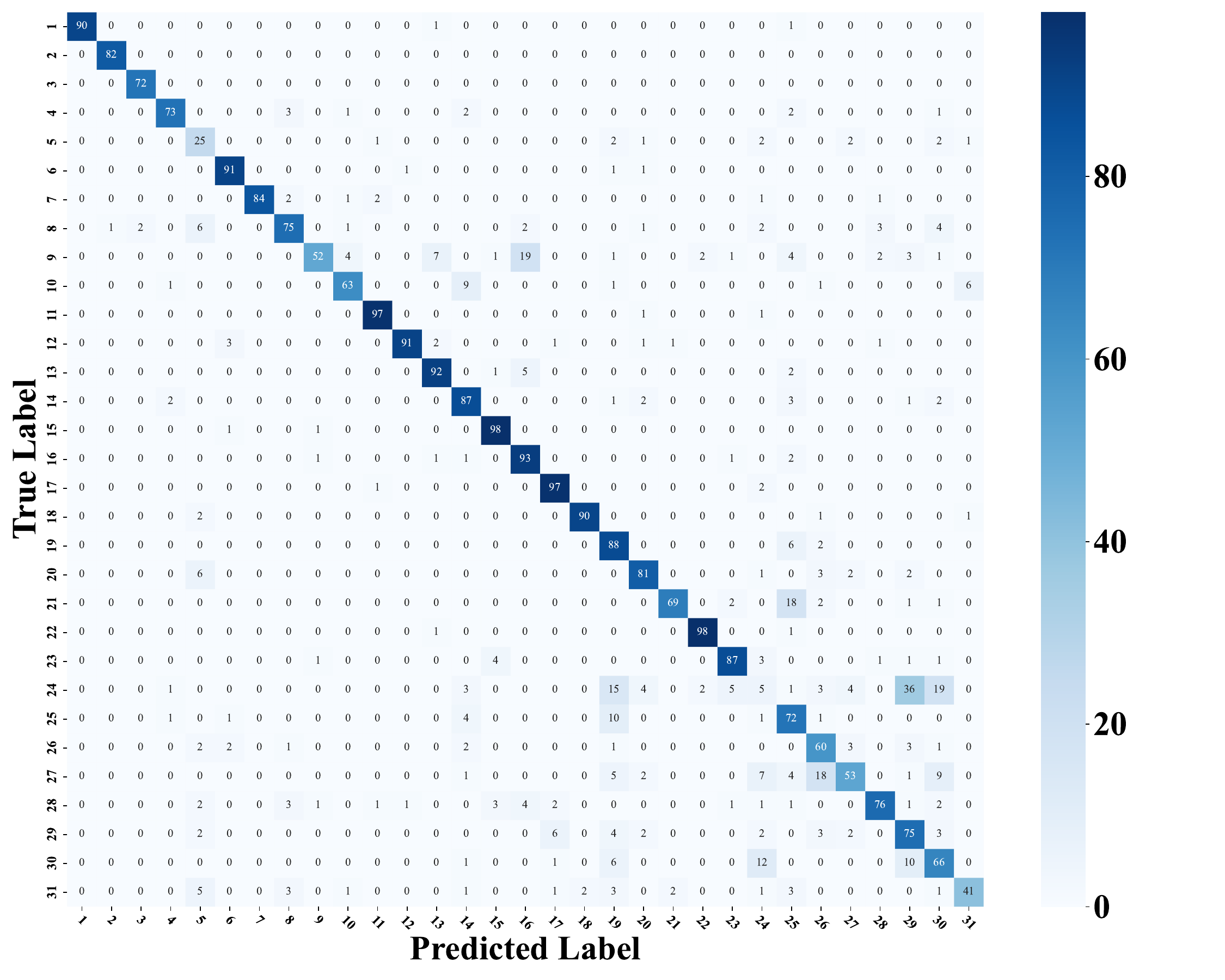}
        \caption{D$\to$A}
    \end{subfigure}
    
    % -------- Row 2 --------
    \begin{subfigure}[t]{0.23\textwidth}
        \includegraphics[width=\textwidth]{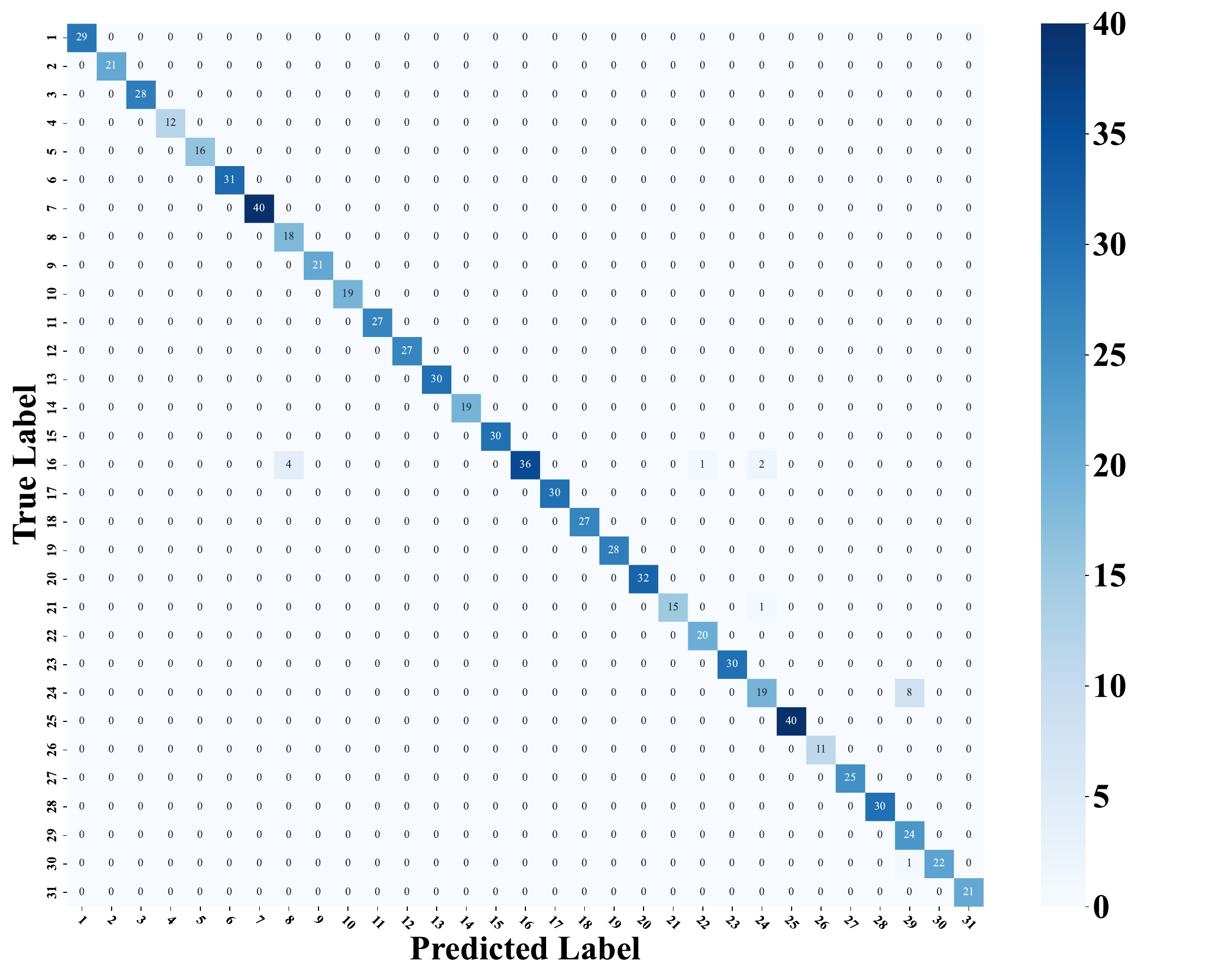}
        \caption{D$\to$W}
    \end{subfigure}
    % \hfill
    \begin{subfigure}[t]{0.23\textwidth}
        \includegraphics[width=\textwidth]{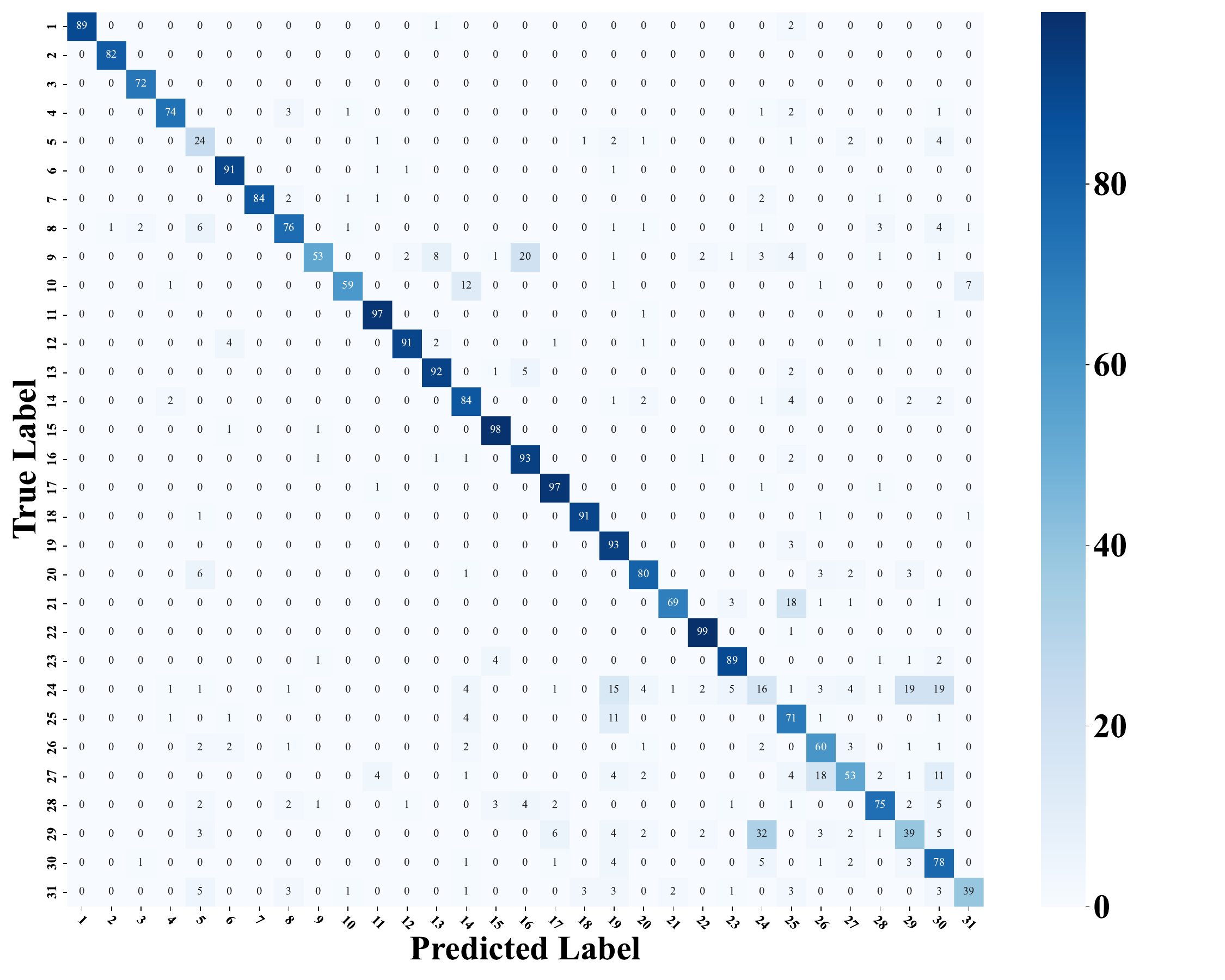}
        \caption{W$\to$A}
    \end{subfigure}
    % \hfill
    \begin{subfigure}[t]{0.23\textwidth}
        \includegraphics[width=\textwidth]{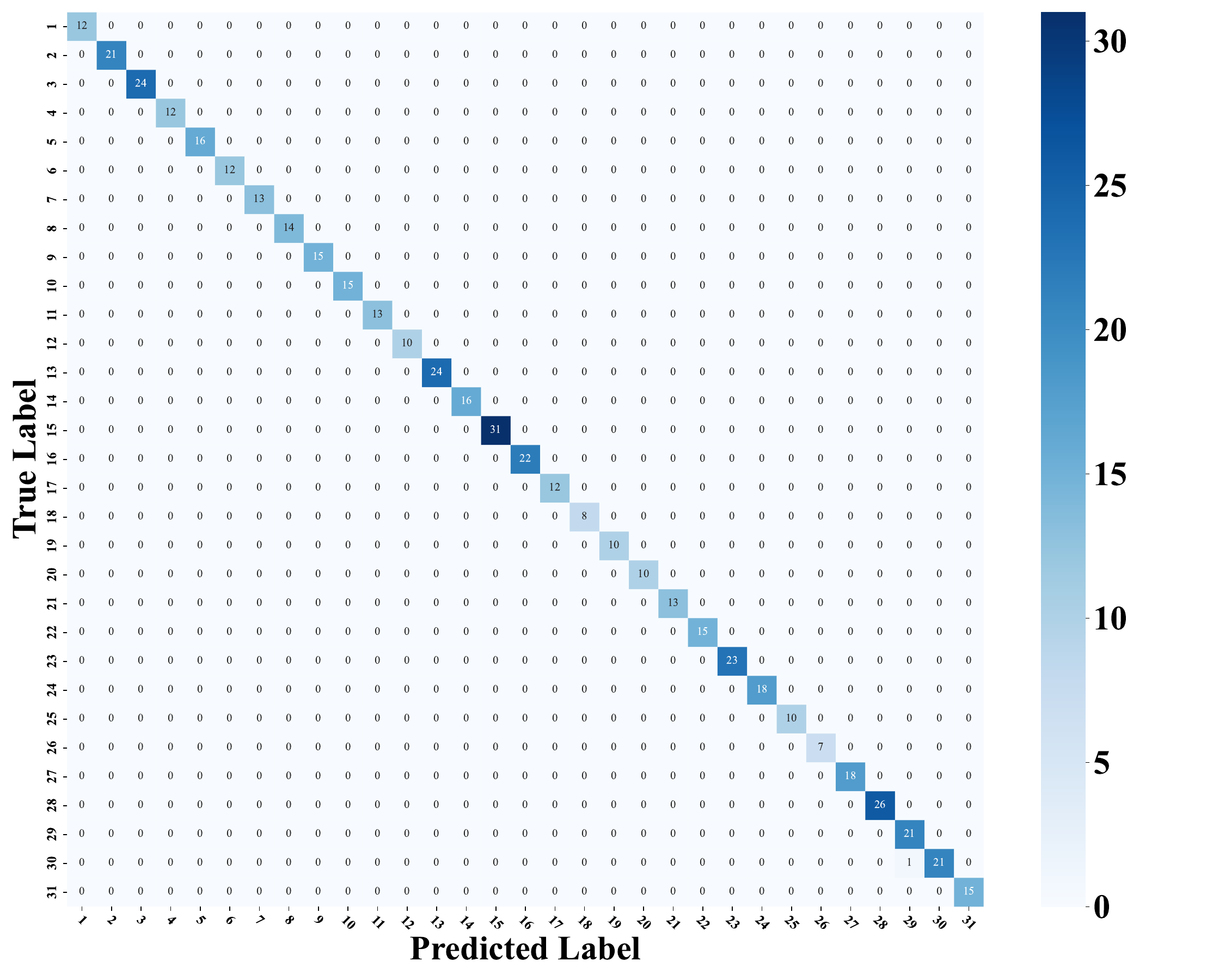}
        \caption{W$\to$D}
    \end{subfigure}
    \begin{subfigure}[t]{0.8\textwidth}
        \includegraphics[width=\textwidth]{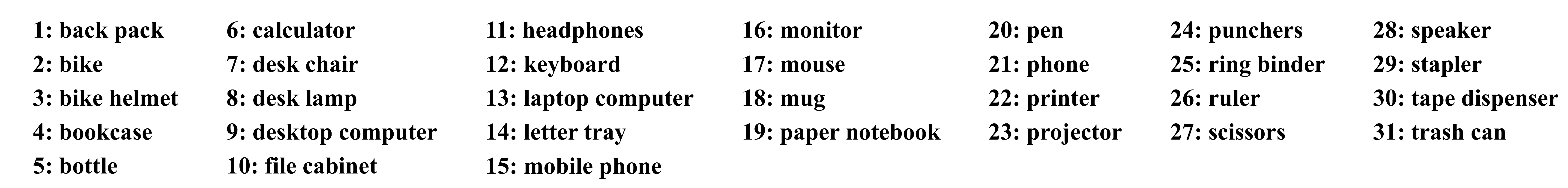}
    \end{subfigure}
    \caption{Confusion matrices for different domain shifts in the Office-31 dataset.}
    \label{cm_office31}
\end{figure}

\paragraph{Cross-Dataset Comparison.} The three datasets reflect a spectrum of domain adaptation challenges, from the large-scale synthetic-to-real gap in VisDA-C to the denser label space of Office-Home and the relatively mild shifts in Office-31. SSA maintains robust performance across these settings by progressively aligning semantic features, yet residual confusions persist in visually or functionally similar categories. This suggests that while SSA effectively addresses broad distribution gaps, its current feature modeling may lack the granularity needed for fine-grained discrimination. Incorporating part-aware cues, adaptive fusion, or relational reasoning could further improve its adaptability in complex or densely labeled domains.

\subsubsection{T-SNE visualization for image classification on VisDA-C}
\label{sec:appendix_tsne}

\autoref{tsne_show} illustrates the feature distributions of source-only, SHOT, and SSA on the VisDA-C dataset using T-SNE visualization. The \textbf{source-only features (a)} show significant overlap and scattered clusters, indicating severe \textbf{class confusion} and weak semantic separability under domain shift. This suggests that models trained only on source data struggle to form clear boundaries between classes in the target domain.

\begin{figure}[H]
    \centering
    % -------- Row 1 --------
    \begin{subfigure}[t]{0.2\textwidth}
        \includegraphics[width=\textwidth]{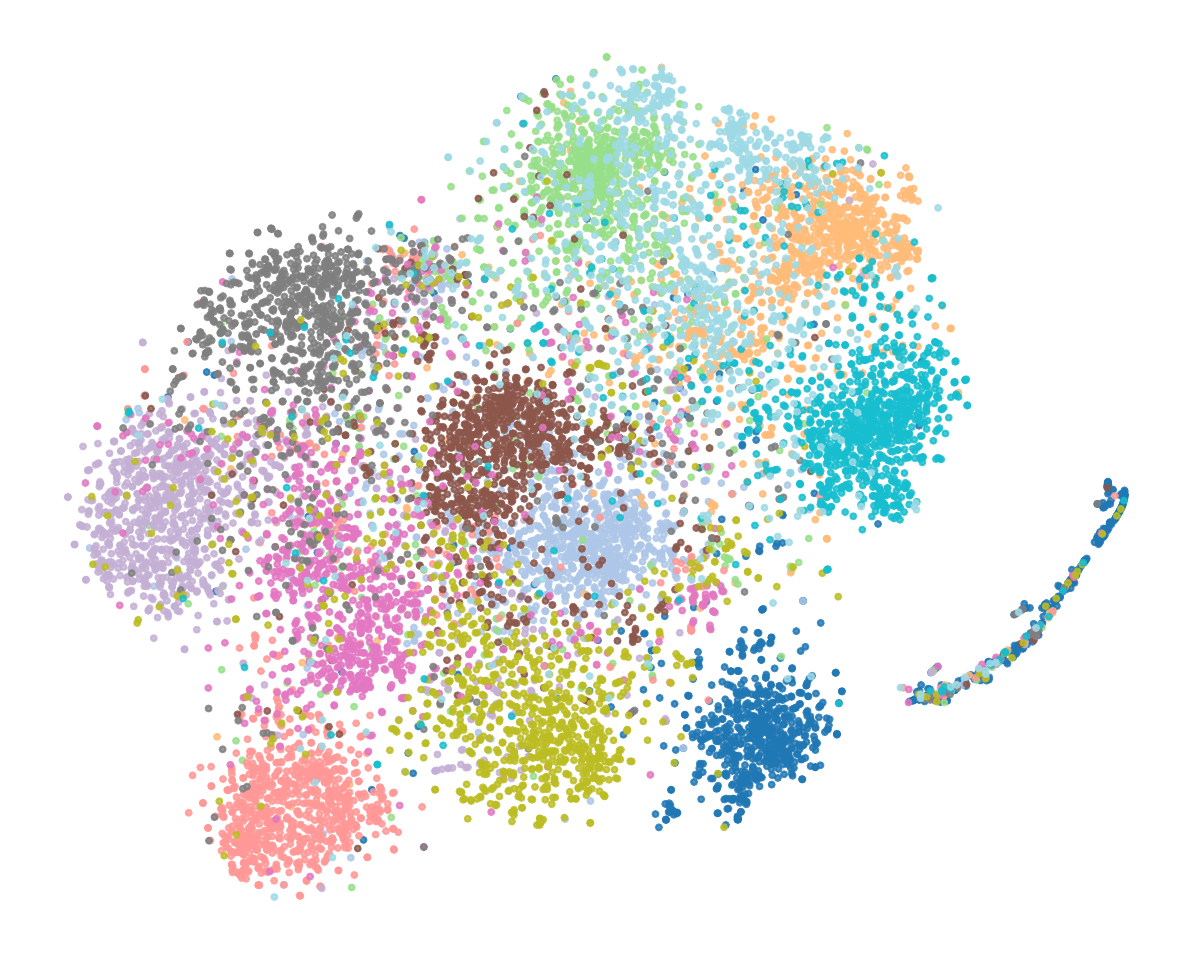}
        \caption{Source Only}
    \end{subfigure}
    %\hfill
    \begin{subfigure}[t]{0.2\textwidth}
        \includegraphics[width=\textwidth]{./img/show/tsne/tsne_shot_visda.pdf}
        \caption{SHOT}
    \end{subfigure}
    %\hfill
    \begin{subfigure}[t]{0.2\textwidth}
        \includegraphics[width=\textwidth]{./img/show/tsne/tsne_shlsa_visda.pdf}
        \caption{SSA}
    \end{subfigure}
    \begin{subfigure}[t]{0.6\textwidth}
        \includegraphics[width=\textwidth]{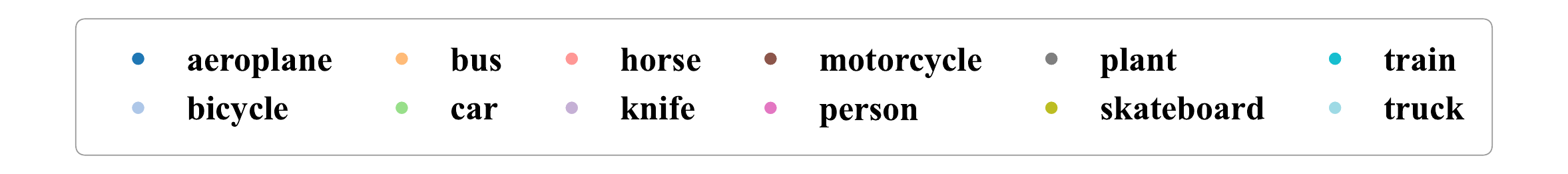}
    \end{subfigure}
    \caption{t-SNE visualization for image classification on VisDA-C dataset.}
    \label{tsne_show}
\end{figure}

The \textbf{SHOT method (b)} partially improves the compactness and separation of clusters. However, some ambiguity and overlap persist among visually or semantically related classes, revealing that SHOT’s global alignment alone cannot fully resolve fine-grained category differences.

In contrast, \textbf{SSA (c)} generates the most distinct and well-separated clusters. Its hierarchical semantic alignment enhances both intra-class cohesion and inter-class separation, effectively reducing class confusion. This results in more stable decision boundaries and improved feature discriminability. The visualization thus provides strong empirical support for SSA’s ability to handle challenging domain shifts better than conventional low-level alignment methods by capturing richer semantic structure.

%%%%%%%%%%%%%%%%%%%%%%%%%%%%%%%%%%%%%%%%%%%%%%%%%%%%%%%%%%%%%%%%%%%%%%%%%%%%%%%
%%%%%%%%%%%%%%%%%%%%%%%%%%%%%%%%%%%%%%%%%%%%%%%%%%%%%%%%%%%%%%%%%%%%%%%%%%%%%%%

\end{document}